\documentclass{article} 

\let\oldaddcontentsline\addcontentsline
\usepackage[accepted]{icml2024}
\let\addcontentsline\oldaddcontentsline

\usepackage{amsmath,amsfonts,bm}

\def\eqref#1{equation~\ref{#1}}

\def\1{\bm{1}}

\def\rvv{{\mathbf{v}}}

\DeclareMathAlphabet{\mathsfit}{\encodingdefault}{\sfdefault}{m}{sl}
\SetMathAlphabet{\mathsfit}{bold}{\encodingdefault}{\sfdefault}{bx}{n}

\def\gB{{\mathcal{B}}}

\def\gD{{\mathcal{D}}}

\def\gI{{\mathcal{I}}}

\def\gN{{\mathcal{N}}}

\def\gP{{\mathcal{P}}}
\def\gQ{{\mathcal{Q}}}

\def\gY{{\mathcal{Y}}}

\def\sI{{\mathbb{I}}}

\def\sP{{\mathbb{P}}}

\def\sR{{\mathbb{R}}}

\DeclareMathOperator*{\argmax}{arg\,max}

\usepackage{hyperref}
\usepackage{url}
\usepackage{algorithm}
\usepackage{algorithmic}
\usepackage{amsmath, amsthm, amssymb}

\newtheorem{theorem}{Theorem}

\usepackage{graphicx}
\usepackage{xcolor}
\definecolor{Highlight}{HTML}{39b54a}  

\usepackage{bbding}
\usepackage{multirow}
\usepackage{algorithm}
\usepackage{algorithmic}
\usepackage{amsmath}
\usepackage{amsthm}
\usepackage{amssymb}
\usepackage{bm}
\usepackage{threeparttable}
\usepackage{xspace}
\usepackage{enumitem}
\usepackage{bbm}
\usepackage{cleveref}
\usepackage{dsfont}
\usepackage{subfigure}
\usepackage{caption}
\usepackage{graphicx}
\usepackage[textsize=scriptsize,backgroundcolor=white]{todonotes}
\usepackage{enumitem}
\usepackage{titletoc}
\usepackage{xspace}
\Crefname{problem}{Problem}{Problems}
\usepackage{mathtools}
\usepackage{titletoc}
\usepackage{natbib}
\usepackage{wrapfig}
\usepackage{footmisc} 
\newcommand\DoToC{%
  \startcontents
  \printcontents{}{1}{\textbf{Contents}\vskip3pt\hrule\vskip5pt}
  \vskip3pt\hrule\vskip5pt
}

\usepackage{booktabs}
\usepackage{tabu}
\usepackage{tabularx}

\usepackage[nice]{nicefrac}

\newtheorem{lemma}{Lemma}[section]

\theoremstyle{definition}

\Crefname{adxproposition}{Proposition}{Propositions}
\Crefname{adxcorollary}{Corollary}{Corollaries}
\Crefname{adxdefinition}{Definition}{Definitions}
\Crefname{adxtheorem}{Theorem}{Theorems}

\theoremstyle{remark}
\newtheorem*{remark}{Remark}

\newbool{IsTwoColumn}
\setbool{IsTwoColumn}{true}

\newcommand{\name}{Rob-FCP\xspace}

\newcommand{\mcolor}[2]{\textcolor{#1}{#2}}

\newcommand{\reb}[1]{\mcolor{black}{#1}}

\begin{document}

\twocolumn[
\icmltitle{Certifiably Byzantine-Robust Federated Conformal Prediction}




\begin{icmlauthorlist}
\icmlauthor{Mintong Kang}{1}
\icmlauthor{Zhen Lin}{1}
\icmlauthor{Jimeng Sun}{1,3}
\icmlauthor{Cao Xiao}{2}
\icmlauthor{Bo Li}{1,4}
\end{icmlauthorlist}

\icmlaffiliation{1}{University of Illinois at Urbana-Champaign, USA}
\icmlaffiliation{2}{GE Healthcare, USA}
\icmlaffiliation{3}{Carle’s Illinois College of Medicine, USA}
\icmlaffiliation{4}{University of Chicago, USA}

\icmlcorrespondingauthor{Mintong Kang}{mintong2@illinois.edu}
\icmlcorrespondingauthor{Bo Li}{lbo@illinois.edu}

\icmlkeywords{Machine Learning, ICML}

\vskip 0.3in
]



\printAffiliationsAndNotice{} 

\begin{abstract}
    Conformal prediction has shown impressive capacity in constructing statistically rigorous prediction sets for machine learning models with exchangeable data samples.
    The siloed datasets, coupled with the escalating privacy concerns related to local data sharing, have inspired recent innovations extending conformal prediction into federated environments with distributed data samples. However, this framework for distributed uncertainty quantification is susceptible to Byzantine failures. A minor subset of malicious clients can significantly compromise the practicality of coverage guarantees.
    To address this vulnerability, we introduce a novel framework \name, which executes robust federated conformal prediction, effectively countering malicious clients capable of reporting arbitrary statistics in the conformal calibration process. We theoretically provide the conformal coverage bound of \name in the Byzantine setting and show that the coverage of \name is asymptotically close to the desired coverage level. 
    We also propose a malicious client number estimator to tackle a more challenging setting where the number of malicious clients is unknown to the defender. We theoretically show the precision of the malicious client number estimator.
    Empirically, we demonstrate the robustness of \name against various portions of malicious clients under multiple Byzantine attacks on five standard benchmark and real-world healthcare datasets.
\end{abstract}

\section{Introduction}

As deep neural networks (DNNs) achieved great success across multiple fields~\citep{He_2016_CVPR,vaswani2017attention,li2022competition}, quantifying the uncertainty of model predictions has become essential, especially in safety-conscious domains such as healthcare and medicine \citep{ahmad2018interpretable,erickson2017machine,kompa2021second}.
For example, in sleep medicine domain, accurately classifying sleep stages (typically on EEG recordings) is crucial for understanding sleep disorders.
Analogous to a human expert who may offer multiple possible interpretations of a single recording, it is desirable for a DNN to provide not just a singular prediction but a set of possible outcomes.
In constructing such prediction sets, we often consider the following coverage guarantee: the prediction set should contain the true outcome with a pre-specified probability (e.g. 90\%).
\textit{Conformal prediction} \citep{shafer2008tutorial,balasubramanian2014conformal,romano2020classification} demonstrates the capacity to provide such statistical guarantees for any black-box DNN with exchangeable data.

Meanwhile, the demand for training machine learning models on large-scale and diverse datasets necessitates model training across multiple sites and institutions. 
\reb{
Federated learning~\citep{konevcny2016federated,smith2017federated,mcmahan2017communication,bonawitz2019towards, yang2019federated, kairouz2021advances} offers an effective approach to collaboratively train a global model while preserving data privacy, as it enables training with distributed data samples without the requirement of sharing the raw data.
}
For example, multiple hospitals (``clients'') could jointly train a global clinical risk prediction model without sharing raw patient data.
However, this introduces a unique challenge: { 
the existence of malicious or negligent clients can negatively affect the training/testing of the global model.}


\reb{Recently, federated conformal prediction (FCP) methods \citep{lu2021distribution,lu2023federated,plassier2023conformal,humbert2023one} provide rigorous bounds on the coverage rate with distributed data samples.}
However, FCP demonstrates vulnerability to \textit{Byzantine failures} \citep{lamport2019byzantine}, which are caused by uncontrollable behaviors of malicious clients. 
For example, a hospital's data could be corrupted with incorrect or even fabricated medical information due to human negligence or deliberate manipulation of data statistics (such as age, gender, or disease prevalence).
In the Byzantine federated setting, the prediction coverage guarantees of FCP are broken, and the empirical marginal coverage is downgraded severely, even with a small portion of malicious clients as \Cref{fig:fig_res_1}.

\begin{figure}[t]
    \centering
\includegraphics[width=0.6\linewidth]{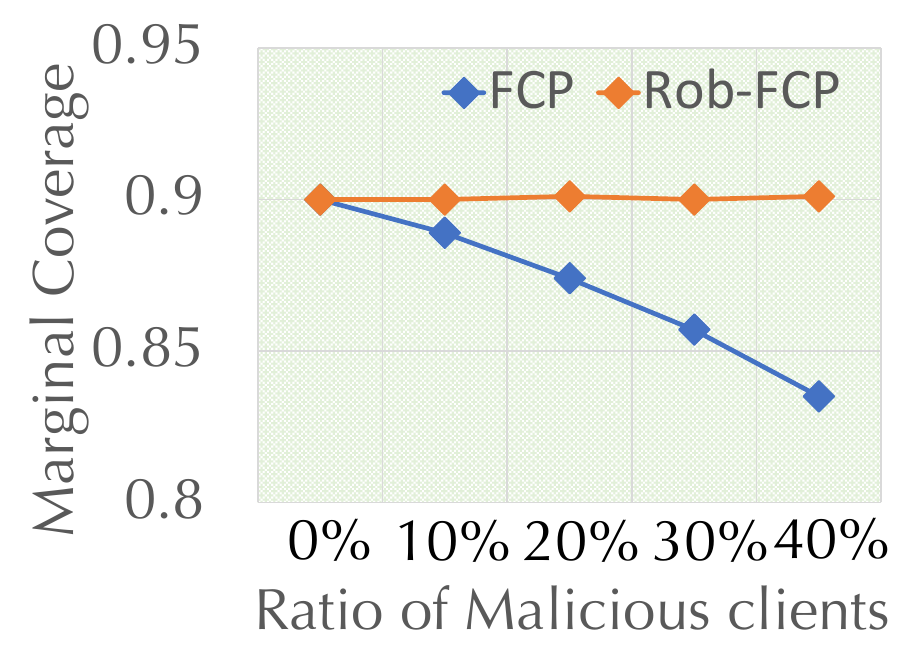}
\vspace{-0.4em}
    \caption{Coverage rate with different ratios of malicious clients on SHHS dataset. The desired coverage is $0.9$.}
    \label{fig:fig_res_1}
    \vspace{-0.8em}
\end{figure}

In this paper, we propose a robust federated conformal prediction algorithm, \name, aimed at mitigating the impact of malicious clients on the coverage rate in Byzantine federated learning environments.
The \name algorithm computes local conformity scores, sketches them with characterization vectors, and detects malicious clients based on averaged vector distance. Clients deemed highly malicious are subsequently excluded from the calibration process. 
Furthermore, we provide a technique for estimating the number of malicious clients, when their exact count is unknown, by optimizing the likelihood of the characterization vectors.   Our theoretical analysis of the coverage bounds shows that the coverage of \name is asymptotically close to the desired coverage level as long as the number of malicious clients is less than that of benign clients and the sample sizes of benign clients are sufficiently large. 

We empirically evaluate \name against multiple Byzantine attacks.
\name outperforms FCP by a large margin and even achieves comparable prediction coverage and efficiency as the benign settings on \textit{five} realistic datasets covering multiple fields. 
We also demonstrate the validity and tightness of the bounds of prediction coverage with different ratios of malicious clients.
We further conduct a set of ablation studies on the methods of conformity scores characterization and different distance measurements to highlight the critical components in \name.

\noindent\textbf{\underline{Technical Contributions:}} Our contributions span both theoretical and empirical aspects.
\vspace{-1em}
\begin{itemize}[noitemsep,leftmargin=*]
    \item We provide the \textit{first} certifiably robust federated conformal prediction framework (\name) in the Byzantine setting where malicious clients can report arbitrary conformity score statistics.
    \item We propose a maliciousness score to effectively detect Byzantine clients and a malicious client number estimator to predict the number of Byzantine clients.
    \item We theoretically certify the coverage guarantees of \name. We also theoretically analyze the precision of the malicious client number estimator.
    \item We empirically demonstrate the robustness of \name in federated Byzantine settings across multiple datasets. We also empirically validate the soundness and tightness of the coverage guarantees.
\end{itemize}

\section{Preliminaries}
\subsection{Conformal prediction}
Suppose that we have $n$ data samples $\{(X_i,Y_i)\}_{i=1}^n$ with features $X_i \in \mathbb{R}^d$ and labels $Y_i \in \mathcal{Y}:=\{1,2,...,C\}$. Assume that the data samples are drawn exchangeably from some unknown joint distribution of feature $X$ and label $Y$, denoted by $\gP_{XY}$.
\reb{
Given a desired coverage $1-\alpha \in (0,1)$, conformal prediction methods construct a prediction set $\hat{C}_{n,\alpha} \subseteq \mathcal{Y}$ for a new data sample $(X_{n+1}, Y_{n+1})\sim \gP_{XY}$ with the guarantee of \textit{marginal prediction coverage}:
$
    \mathbb{P}[Y_{n+1} \in \hat{C}_{n,\alpha}(X_{n+1})] \ge 1-\alpha
$.}

\reb{
In this work, we focus on the split conformal prediction setting~\citep{papadopoulos2002inductive}, where the data samples are randomly partitioned into two disjoint sets: a training set $\mathcal{I}_{\text{tr}}$ and a calibration (hold-out) set $\mathcal{I}_{\text{cal}}=[n]${$\backslash \mathcal{I}_{\text{tr}}$}.}
\footnote{In here and what follows, {$[n] := \{1, \cdots, n\}$}.}
We fit a classifier to the training set {$\mathcal{I}_{\text{tr}}$} to estimate the conditional class probability $\pi: \sR^d \mapsto \Delta^C$, with the $y$-th element denoted as {$\pi_y(x) = \sP[Y=y|X=x]$}.
Using the estimated probabilities that we denote by {$\hat{\pi}(x)$}, we then compute a non-conformity score {$S_{\hat{\pi}}(X_i,Y_i)$} for each sample in the calibration set $\gI_{\text{cal}}$.
The non-conformity score measures how much non-conformity each sample has with respect to its ground truth label.
A small non-conformity score $S_{\hat{\pi}}(X_i,Y_i)$ indicates that the estimated class probability $\hat{\pi}(X_i)$ aligns well with the ground truth label $Y_i$ for the data sample $(X_i,Y_i)$.
A simple and standard non-conformity score \citep{sadinle2019least} is $S_{\hat{\pi}}(x,y) = 1 - \hat{\pi}_y(x)$.

Given a desired coverage {$1-\alpha$}, the prediction set of the new test data point {$X_{n+1}$} is formulated as:
\begin{equation}
\small
\label{eq:pre_set}
\begin{aligned}
    \hat{C}_{n,\alpha}(X_{n+1}) =& \left\{ y \in \gY: S_{\hat{\pi}}(X_{n+1},y) \le \right.\\ & \left.  Q_{1-\alpha}\left(\{S_{\hat{\pi}}(X_i,Y_i)\}_{i \in \gI_{\text{cal}}} \right) \right\},
\end{aligned}
\end{equation}
where {$Q_{1-\alpha}(\{S_{\hat{\pi}}(X_i,Y_i)\}_{i \in \gI_{\text{cal}}})$} is the {$\lceil (1-\alpha)(1+|\gI_{\text{cal}}|) \rceil$}-th largest value of the set {$\{S_{\hat{\pi}}(X_i,Y_i)\}_{i \in \gI_{\text{cal}}}$}.
The prediction set {$\hat{C}_{n,\alpha}(X_{n+1})$} includes all the labels with a smaller non-conformity score than the {$(1-\alpha)$}-quantile of scores in the calibration set. 
Since we assume the data samples are exchangeable, the marginal coverage of the prediction set {$\hat{C}_{n,\alpha}(X_{n+1})$} is no less than {$1-\alpha$}. 
We refer to
\citep{vovk2005algorithmic} for a more rigorous analysis of the prediction coverage.

\subsection{Federated conformal prediction}
\label{sec:fcp_pre}
In federated learning, multiple clients own their private data locally and collaboratively develop a global model. Let $K$ be the number of clients. 
We denote the local data distribution of the $k$-th client ($k \in [K]$) by $\gP^{(k)}$. Let $\{(X_i^{(k)},Y_i^{(k)})\}_{i \in [n_k]} \sim \gP^{(k)}$ be $n_k$ calibration samples owned by the $k$-th client. 
\reb{
We denote $(X_{\text{test}},Y_{\text{test}})$ as the future test point sampled from the global distribution $\gQ_{\text{test},\lambda}$ for some probability vector $\mathbf{\lambda} \in \Delta^K$: $(X_{\text{test}},Y_{\text{test}}) \sim \gQ_{\text{test},\lambda}:= \sum_{k=1}^K \lambda_k \gP^{(k)}$.
}
Let $N=\sum_{k=1}^K n_k$ be the total sample size of $K$ clients and $\hat{q}_\alpha$ be the $\lceil (1-\alpha)(N+K) \rceil$-th largest value in $\{S_{\hat{\pi}}(X_i^{(k)},Y_i^{(k)})\}_{i \in [n_k], k \in [K]}$, where $\hat{\pi}$ is the collaboratively trained conditional class probability estimator ($\alpha \ge {1}/{(N/K+1)}$). 
FCP \citep{lu2023federated} proves that under the assumption of partial exchangeability \citep{carnap1980studies} and $\lambda_k \propto (n_k+1)$, the prediction set $\hat{C}_\alpha(X_{\text{test}})=\{y \in \gY: S_{\hat{\pi}}(X_{\text{test}},y) \le \hat{q}_{\alpha}\}$ is a valid conformal prediction set with the guarantee:
\begin{equation}
\small
\label{eq:guarantee_marginal}
    1-\alpha \le \sP\left[ Y_{\text{test}} \in \hat{C}_\alpha(X_{\text{test}})\right] \le 1-\alpha+\dfrac{K}{N+K}.
\end{equation}
Considering communication cost and privacy concerns, having all agents upload their local non-conformity scores to the server for quantile computation of $\hat{q}_\alpha$ is impractical. Consequently, FCP \citep{lu2023federated} utilizes data sketching algorithms like T-digest \citep{Dunning2021100049} for efficient and privacy-preserving distributed quantile estimation.
They prove that if the rank of quantile estimate $\hat{q}_\alpha$ is between $(1-\alpha-\epsilon)(N+K)$ and $(1-\alpha+\epsilon)(N+K)$ where \textit{$\epsilon$ denotes the quantile estimation error induced by data sketching}, then the guarantee in \Cref{eq:guarantee_marginal} can be corrected as the following:
\begin{equation}
\small
\label{eq:marginal_eps}
    1-\alpha -\dfrac{\epsilon N + 1}{N+K} \le \sP\left[ Y_{\text{test}} \in \hat{C}_\alpha(X_{\text{test}})\right] \le 1-\alpha + \epsilon +\dfrac{K}{N+K},
\end{equation}
where $K$ is the number of clients and $N$ is the total sample sizes of clients.

\section{\name and coverage guarantees}

\begin{figure}[t]
    \centering
    \includegraphics[width=\linewidth]{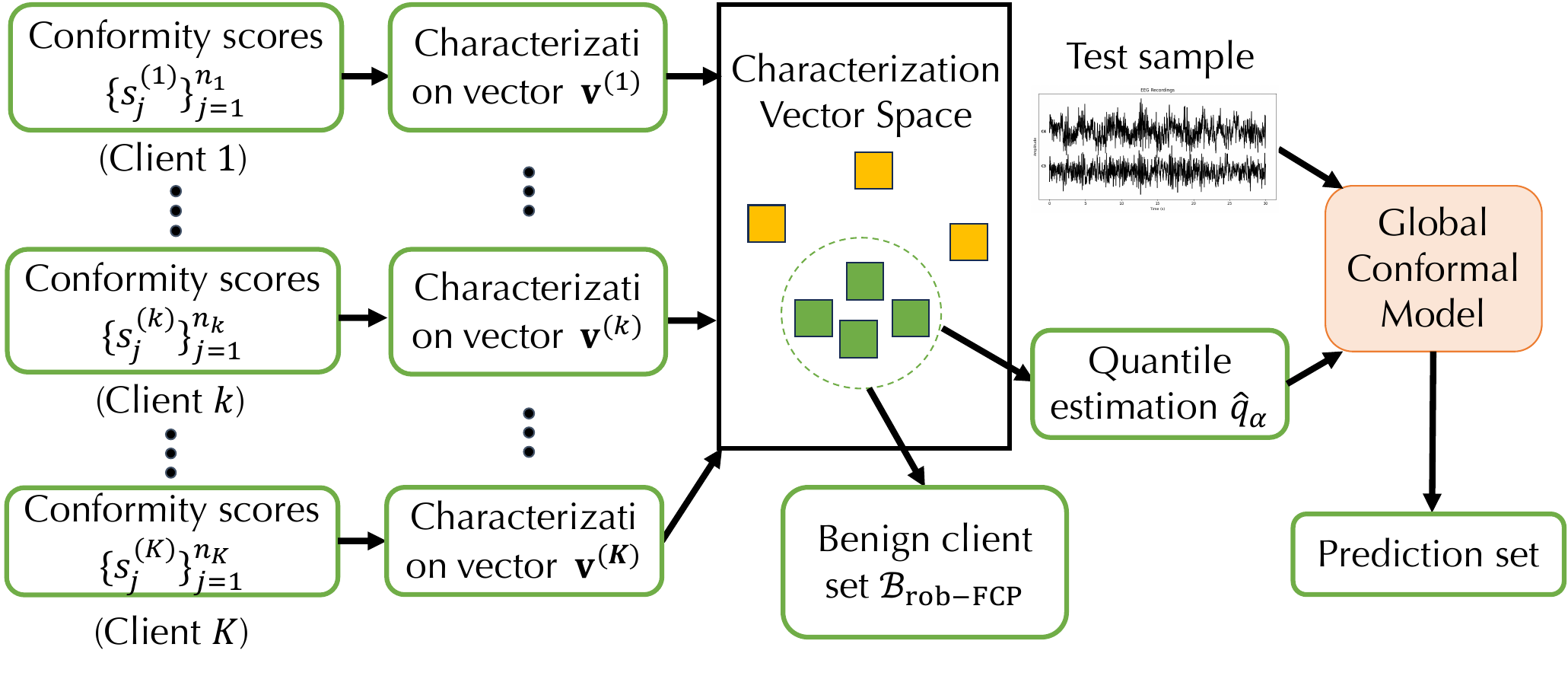}
    \vspace{-1.1em}
    \caption{Overview of \name.}
    \label{fig1}
    \vspace{-0.5em}
\end{figure}

\subsection{Threat model}
\label{sec:pro_def}

We follow the standard setup of FCP in \Cref{sec:fcp_pre} and consider the following Byzantine threat model.
Suppose that among $K$ clients, there exist $K_b$ benign clients and $K_m~(K_m=K-K_b)$ malicious (Byzantine) clients. Without loss of generality, let the clients indexed by $[K_b]=\{1,...,K_b\}$ be benign clients and the clients indexed by $[K] \backslash [K_b]=\{K_b+1,...,K\}$ be malicious clients. The $k$-th benign client ($k \in [K_b]$) leverage the collaboratively trained global model $\hat{\pi}$ to compute the conformity scores on its local calibration data and sketched the score statistics with a characterization vector $\mathbf{v}^{(k)} \in \Delta^H$ where $H$ is the granularity of the characterization statistics and then report the score vector $\mathbf{v}^{(k)}$ to the server. In contrast, $K_m$ malicious clients can submit arbitrary characterization vectors $\mathbf{v}^{(k)} (k \in [K] \backslash [K_b])$ to the server. 

Following FCP \cite{lu2023federated}, the server considers a global distribution $\mathcal{Q}$ as a weighted combination of local distributions, denoted by $\mathcal{Q} = \sum_{i=1}^K \lambda_i \mathcal{P}^{(i)}$, where $\lambda_i$ represents the weight assigned to each local distribution and is proportional to the size of local samples $n_i$: $\lambda_i \propto (n_i + 1)$. Note that the server knows the true weights of local distributions (or equivalently, quantities of local samples), which can not be manipulated by malicious clients during the conformal prediction phase. Since the weights of local distributions (or equivalently, quantities of local samples) are a known priori to the server during the federated model learning phase, the threat model is reasonable and practical, aligning with the existing Byzantine analysis literature \cite{blanchard2017machine,park2021sageflow,data2021byzantine}. For the threat model, we aim to develop a Byzantine-robust FCP framework (\name) that maintains coverage and prediction efficiency despite the existence of malicious clients. We also aim to provide rigorous coverage guarantees of \name in the Byzantine setting





\subsection{\name algorithm}
\label{sec:rob_alg}

\name first detects the set of malicious clients, then excludes their score statistics during the computation of empirical quantile of conformity scores, and finally performs federated conformal prediction with the quantile value, which is not affected by malicious clients.
\vspace{-1em}
\paragraph{Characterization of conformity scores} Let $\{s_j^{(k)}\}_{j\in[n_k]}$ be the conformity scores computed by the $k$-th client ($k \in [K]$) on its local calibration set.
Since it is challenging to detect abnormal behavior from the unstructured and unnormalized conformity scores, we characterize the local conformity scores $\{s_j^{(k)}\}_{j\in[n_k]}$ with a vector $\rvv^{(k)} \in \mathbb{R}^H$ for client $k$, where the vector dimension $H \in \mathbb{Z}^+$ implicates the granularity of the characterization.
Specifically, we can partition the range of conformity score values (e.g., $[0,1]$ for APS score \citep{romano2020classification}) into $H$ subintervals $\{[a_{h},a_{h+1})\}_{0\le h \le H-2} \cup \{[a_{H-1},a_H]\}$, where $a_h$ denotes the $h$-th cut point.\footnote{For simplicity, we abuse the last interval $[a_{H-1},a_H]$ as $[a_{H-1},a_H)$ in the future discussions.}
Thus, the $h$-th element of the characterization vector ($\rvv^{(k)}_h$) represents the probability that a conformity score falls within the specific subinterval $[a_{h-1},a_h)$:
\begin{equation}
\small
\label{eq:score2vec}
\begin{aligned}
    \rvv^{(k)}_h &= \sP_{s \sim \left\{s_j^{(k)}\right\}_{j\in[n_k]}}\left[ a_{h-1} \le s < a_{h} \right] \\ &= \dfrac{1}{n_k}\sum_{j=1}^{n_k} \sI\left[a_{h-1} \le s_j^{(k)} <  a_{h} \right],
\end{aligned}
\vspace{-0.2em}
\end{equation}
where $\sI[\cdot]$ denotes the indicator function.
The characterization vector $\rvv^{(k)}$ is designed to encapsulate the distribution of score samples via histogram statistics, reflecting a fundamental multinomial distribution. This methodology leverages the observation that conformity scores originating from homogeneous distributions typically show substantial similarity. Consequently, characterization vectors from benign clients exhibit notable resemblance, in contrast to those from malicious clients, whose score statistics are anomalous. Such a distinct pattern facilitates the reliable identification of malicious clients.

Furthermore, \name is designed with the flexibility to incorporate various methodologies for representing empirical conformity score samples as a real-valued vector $\rvv$. Among these methodologies are kernel density estimation \citep{terrell1992variable}, offering a more nuanced interpretation of histogram statistics; parametric model fitting, such as Gaussian models; and clustering-based exemplar representations, including KMeans. The empirical analysis in \Cref{sec:abla} indicates that the histogram-based approach, as formulated in \Cref{eq:score2vec}, surpasses both parametric models and clustering techniques in performance. Hence, we consider the histogram statistic in \name as our primary method of analysis.

\vspace{-0.5em}
\paragraph{Maliciousness score computation}
\name detects the malicious clients via a maliciousness score in the space of characterization vectors. 
First, we compute pairwise $\ell_p$ ($p \in \mathbb{Z}^+$) vector distances among $K$ clients:
\begin{equation}
\small
\label{eq:dis}
    d_{k_1,k_2} = \|\rvv^{(k_1)} - \rvv^{(k_2)} \|_p,~ \forall k_1, k_2 \in [K].
\end{equation}
Denote $N_{ear}(k, t)$ as the index set of the $t$-nearest neighbors of client $k$ (excluding itself), with the distance between two clients $k_1$ and $k_2$ given by \Cref{eq:dis}.
We define the maliciousness score $M(k) \in \sR$ of client $k$ $(k \in [K])$ as the averaged distance to the $K_b-1$ nearest neighbors, where $K_b$ is the number of benign clients:
\begin{equation}
\small
    M(k) = \dfrac{1}{K_b-1} \sum_{k' \in N_{ear}(k,K_b-1)} d_{k,k'}.
\end{equation}
We define the benign set identified by \name, denoted as $\gB_{\text{\name}}$, as the set containing the indices of clients with the lowest $K_b$ maliciousness scores among $\{M(k)\}_{k=1}^K$.
Subsequently, quantile estimation $\hat{q}_\alpha$ is carried out using the characterization vectors from the clients within the benign set $\gB_{\text{\name}}$. The quantile estimation $\hat{q}\alpha$ is then applied to perform federated conformal prediction on the globally trained model in a distributed manner. An overview of \name is presented in \Cref{fig1}, with the pseudocode detailed in \Cref{alg:rob_fl_confinf} in \Cref{app:alg}.


To impair the overall performance of global conformal predictions, malicious clients often submit conformity score statistics that starkly contrast with those of benign clients. This difference results in the characterization vectors of malicious clients being distinct and separable from the aggregation of benign vectors. The calculation of maliciousness scores, which is based on the average distance to the $K_b-1$ nearest neighbors, further accentuates this separation. Specifically, malicious clients tend to have higher maliciousness scores than benign clients, given the condition $K_b>K_m$, a common assumption in Byzantine resilience studies \citep{blanchard2017machine}). Leveraging this distinction, \name effectively isolates and disregards the skewed statistics introduced by malicious clients during the conformal calibration process, thereby maintaining the validity of the conformal prediction set. A theoretical analysis of \name, including rigorous coverage bounds, is provided in \Cref{sec:analysis}.

\vspace{-0.5em}
\paragraph{Effectiveness of \name against mimick attacks} 
 Malicious clients with mimic attack \cite{karimireddy2022byzantinerobust,shejwalkar2021manipulating} transmit similar gradients to benign clients in FL optimization, which is stealthy and deteriorates the optimization process by over-representing the mimicked clients in the setting with high data heterogeneity. However, in FCP, for a collaboratively trained model, we observe that the heterogeneity of distributions of non-conformity scores cannot be effectively used by mimic attacks to disturb the FCP process. The major difference between the setting in \cite{karimireddy2022byzantinerobust,shejwalkar2021manipulating} and FCP is that the former considers the FL optimization, where clients perform multi-step local updates on local data distribution, and thus the gradients among clients can show a pretty high heterogeneity due to the data heterogeneity and also the high dimensionality of the gradients. This makes a great opportunity for the attackers to hide in and still distort the FL optimization effectively. However, in the FCP setting, the model is well-trained and converges well. Thus, the heterogeneity in the space of nonconformity score vectors is not as great as the heterogeneity in the high-dimensional gradient space during optimization. 
 Note that in \name, we do not have assumptions that the malicious clients should be very different from benign clients. The principle of the effectiveness of \name is that (1) if the score vector of malicious clients is close to the benign clients, although \name may identify it as benign, it can only make a limited and bounded difference on the FCP results, and (2) if the score vector is far from the benign cluster, although it is effective to distort FCP, \name will filter it out in this case. 
 We provide the empirical validation results of the observation in \Cref{tab:mimick} in \Cref{app:res}.

\subsection{Coverage guarantee of \name}
\label{sec:analysis}

We rigorously analyze the lower and upper bounds of the prediction coverage of \name in the Byzantine setting in \Cref{thm1:improve}. 
The analysis reveals that, with an adequately large sample size of benign clients, \name is capable of reaching the desired coverage level. This finding underscores the effectiveness of \name in maintaining reliable prediction coverage, even in the presence of Byzantine clients.

\begin{theorem}[Coverage guarantees of \name in Byzantine setting]
\label{thm1:improve}
    Consider FCP setting with $K_b$ benign clients and $K_m$ malicious clients.
    The $k$-th client reports the characterization vector $\rvv^{(k)}$ and local sample size $n_k$ to the server ($k \in [K_b+K_m]$). 
    Assume that the benign characterization vector $\rvv^{(k)}$ follows multinomial distribution $\gD_k$ with event probability $\overline{\rvv}^{(k)}$ for the $k$-th client ($k \in [K_b]$).
    We use $\sigma$ to quantify the heterogeneity of benign vectors as  $\sigma = \max_{k_1 \in [K_b] ,k_2 \in [K_b]}\| \overline{\rvv}^{(k_1)} - \overline{\rvv}^{(k_2)} \|_1$.
    Let $\epsilon$ be the data sketching error as \Cref{eq:marginal_eps}.
    Under the assumption that $K_m < K_b$, the following coverage guarantee for test instance $(X_{\text{t}},Y_{\text{t}})$ holds with probability $1-\beta$:
    \begin{equation}
    \small
    \label{eq:coro_improve}
    \begin{aligned}
        & \sP\hspace{-0.2em}\left[ Y_{\text{t}} \hspace{-0.2em} \in \hspace{-0.2em} \hat{C}_\alpha(X_{\text{t}})\right] \hspace{-0.1em} \ge \hspace{-0.1em}  1 - \hspace{-0.1em} \alpha \hspace{-0.1em} - \hspace{-0.1em} P_{\text{byz}} \hspace{-0.1em} -  \hspace{-0.1em} \dfrac{N_m\sigma}{n_b(1-\tau)} \hspace{-0.1em} - \hspace{-0.1em} \dfrac{\epsilon n_b + 1}{n_b+K_b} \\
        & 
       \sP\hspace{-0.2em}\left[ Y_{\text{t}} \hspace{-0.2em} \in \hspace{-0.2em} \hat{C}_\alpha(X_{\text{t}})\right] \hspace{-0.1em} \le \hspace{-0.1em} 1 \hspace{-0.1em} - \hspace{-0.1em} \alpha \hspace{-0.1em} + \hspace{-0.1em} P_{\text{byz}} \hspace{-0.1em} + \hspace{-0.1em} \dfrac{N_m\sigma}{n_b(1-\tau)} \hspace{-0.1em} + \hspace{-0.1em} \dfrac{\epsilon n_b + (\epsilon+1) K_b}{n_b+K_b} \\
        & \text{where} \quad P_{\text{byz}} = \dfrac{H\Phi^{-1}({1-\beta/2HK_b})}{2 \sqrt{n_b}} \left( 1+\dfrac{N_m}{n_b}\dfrac{2}{1-\tau} \right)
    \end{aligned}
    \end{equation}
    where $\tau = K_m / K_b$ is the ratio between the number of malicious clients and the number of benign clients, \reb{$N_m:=\sum_{k \in [K]\backslash[K_b]} n_k$ is the total sample size of malicious clients}, $n_b:=\min_{k'\in[K_b]}n_{k'}$ is the minimal sample size of benign clients, and $\Phi^{-1}(\cdot)$ denotes the inverse of the cumulative distribution function (CDF) of standard normal distribution.
\end{theorem}

\begin{remark}
\textbf{\underline{(R1)}} \Cref{eq:coro_improve} offers the lower and upper bound of the prediction coverage 
with \name in the Byzantine setting. The coverage bounds are in relation to (a) Byzantine coverage penalty $P_{\text{byz}}$, (b) client disparity penalty $\nicefrac{N_m\sigma}{n_b(1-\tau)}$, and (c) data sketching penalty $\nicefrac{\epsilon n_b + 1}{n_b+K_b}$ or $\nicefrac{\epsilon n_b + (\epsilon+1) K_b}{n_b+K_b}$.
\textbf{\underline{(R2)}} The Byzantine coverage penalty $P_{\text{byz}}$ is induced by the presence of malicious clients. It can be exacerbated by a large ratio of malicious clients (a large $\tau$) and a large total sample size of malicious clients (a large $N_m$). However, the Byzantine coverage penalty $P_{\text{byz}}$ can be effectively reduced by a larger benign sample size $n_b$.
\textbf{\underline{(R3)}} The client disparity penalty is induced by the data heterogeneity among clients. Similarly, it can be exacerbated by a large $\tau$ and $N_m$, but reduced by a large $n_b$. We leverage the maximal pairwise vector norm to quantify the client heterogeneity, which aligns with existing Byzantine analysis \cite{park2021sageflow,data2021byzantine}.
\textbf{\underline{(R4)}} The data sketching penalty is induced by the local approximation error $\epsilon$ as \Cref{eq:marginal_eps}, with more details provided in \cite{lu2023federated}.
\textbf{\underline{(R5)}} The assumption $ K_m<K_b~(i.e., \tau<1)$ requires that the number of malicious clients is less than the number of benign clients, aligning with the break point of $\lceil K/2 \rceil$ in Byzantine analysis \citep{blanchard2017machine,yin2018byzantine,guerraoui2018hidden}.
\textbf{\underline{(R6)}} There exists a trade-off of selecting the characterization granularity $H$. According to FCP \cite{lu2023federated}, with the histogram estimate, when $H$ decreases, the data sketching becomes rough and increases the approximation error $\epsilon$. At the same time, a smaller $H$ will decrease the Byzantine coverage penalty $P_{\text{byz}}$ due to a better concentration rate. We empirically perform ablation studies on the selection of $H$ in \Cref{app:res}.
\textbf{\underline{(R7)}} We bound the concentration of the characterization vectors with the binomial proportion confidence interval \cite{wallis2013binomial}. We also provide results with more advanced concentration bounds DKW inequality \citep{dvoretzky1956asymptotic} in \Cref{app:dkw}.
\textbf{\underline{(R8)}} Asymptotically, as long as the benign sample size $n_b$ is sufficiently large, both the coverage lower bound and the upper bound reach the desired coverage level $1-\alpha$, demonstrating the robustness of \name.
\end{remark}
\textit{Proof sketch.} 
We first leverage statistical confidence intervals and union bounds to conduct concentration analysis of the characterization vectors $\rvv^{(k)}$ for benign clients ($1\le k \le K_b$). Then we consider the maliciousness scores of critical clients and relax the histogram statistics error. We finally translate the error of aggregated statistics to the error of the coverage bounds by algebra analysis.
We provide complete proofs in \Cref{app:proof_thm1}.

\section{\name with unknown numbers of malicious clients}
\label{sec:mal_ubk}
\subsection{Malicious client number estimator}

In the standard Byzantine framework \citep{blanchard2017machine,park2021sageflow,liu2023byzantine}, the defender is often assumed to have prior knowledge of the quantity of malicious clients $K_m$. This number plays a pivotal role in defense strategies: underestimating it results in the inclusion of malicious clients, leading to a degradation in overall performance, while overestimating it results in the exclusion of benign clients, thereby causing a shift in the global data distribution. However, in real-world applications, the exact count of malicious clients is typically unknown to the server. To address this gap and enhance the system's resilience in more complex Byzantine environments where the number of malicious clients is uncertain, we introduce a novel estimator for malicious client numbers for \name.

To accurately estimate the number of malicious clients $K_m$, we pivot to calculating the number of benign clients $K_b$, given the total client count $K$ is known. 
To achieve this, we aim to maximize the likelihood of benign characterization vectors while minimizing the likelihood of malicious characterization vectors over the number of benign clients $\hat{K}_b$. 
The likelihood computation necessitates a predefined distribution for benign characterization vectors.

Considering that benign characterization vectors $\rvv^{(k)}$ ($k \in [K_b]$) are sampled from a multinomial distribution, which, for substantial sample sizes, can be closely approximated by a multivariate normal distribution as \citep{severini2005elements}, we proceed under the assumption that the benign characterization vectors are samples from a multivariate normal distribution denoted as $\gN(\mu,\Sigma)$, where $\mu \in \sR^H$ represents the mean, and $\Sigma \in \sR^{H \times H}$ denotes the covariance matrix.

Then, we use expectation–maximization (EM) algorithm to effectively estimate the number of benign clients $\hat{K}_b$. In the expectation (E) step, given the current estimate of benign client number $\tilde{K}_b$, we compute the expected Gaussian mean and covariance by the observations of benign characterization vectors, which can be identified by the \name algorithm in \Cref{sec:rob_alg}. In the maximization (M) step, we maximize the likelihood of characterization vectors given the estimated Gaussian mean and covariance in the E step. Formally, 
let $I(\cdot): [K] \mapsto [K]$ be the mapping from the rank of maliciousness scores by \name to the client index.
The EM optimization step can be formulated as:

\vspace{-1em}
\begin{equation}
\small
\label{eq:opt}
\begin{aligned}
     \hat{K}_b =& \argmax_{z \in [K]}\left[ \dfrac{1}{z} \sum_{k=1}^{z} \log p(\rvv^{(I(k))};\hat{\mu}(z),\hat{\Sigma}(z)) \right. \\  & \left. - \dfrac{1}{K-z} \sum_{k=z+1}^K \log p(\rvv^{(I(k))};\hat{\mu}(z),\hat{\Sigma}(z)) \right]
\end{aligned}
\vspace{-0.5em}
\end{equation}
where $\hat{\mu}(z)$ and $\hat{\Sigma}(z)$ are the expected mean and covariance: $\hat{\mu}(z) = \nicefrac{1}{z} \sum_{k \in [z]} \rvv^{(I(k))}$, $\hat{\Sigma}(z)=\mathbb{E}_{k \in [z]}[ (\rvv^{(I(k))}-\hat{\mu}(z))^T(\rvv^{(I(k))}-\hat{\mu}(z)) ]$, and $p(\rvv;\mu,\Sigma)$ computes the likelihood of $\rvv$ given Gaussian $\gN(\mu,\Sigma)$ as $p(\rvv;\mu,\Sigma)=\exp\left( -1/2 (\rvv-\mu)^T \Sigma^{-1} (\rvv-\mu)\right)/{\sqrt{(2\pi)^H |\Sigma|}}$.
The EM optimization in \Cref{eq:opt} essentially searches for $\hat{K}_b$ such that the characteristic vectors of $\hat{K}_b$ clients with the lowest maliciousness scores (higher probability of being benign) exhibit a strong alignment with the benign normal distribution, and conversely,  the characteristic vectors of the remaining clients (more likely to be malicious) show a decreased likelihood of fitting the benign normal distribution. Note that the derived estimate of $\hat{K}_b$ can be utilized as the input parameter $\tilde{K}_b$ in subsequent iterations, allowing for the refinement of the estimation through recursive applications of the EM optimization process.

\subsection{Precision of malicious client number estimator}
In this part, we theoretically show the precision of benign client number estimate in \Cref{eq:opt}. 
\begin{theorem}[Precision of malicious client number estimator]
\label{thm2}
    Assume $\rvv^{(k)}~(k \in [K_b])$ are IID sampled from Gaussian $\gN(\mu,\Sigma)$ with mean $\mu \in \sR^H~(H \ge 2)$ and positive definite covariance matrix $\Sigma \in \sR^{H \times H}$. Let $d=\min_{k \in [K]\backslash [K_b]} \| \rvv^{(k)} - \mu \|_2$. 
    Consider EM optimization as \Cref{eq:opt} and an initial guess of benign client number $\tilde{K}_b$ such that $K_m<\tilde{K}_b\le K_b$.
    Then we have:
   \begin{equation}
   \small
   \label{eq:thm2}
   \begin{aligned}
        \sP\left[ \hat{K}_m = K_m \right] \ge& 1 - \dfrac{(3\tilde{K}_b-K_m-2)^2\text{Tr}(\Sigma)}{(\tilde{K}_b-K_m)^2d^2} \\ &- \dfrac{2(K+K_b) \text{Tr}(\Sigma)\sigma^2_{\text{max}}(\Sigma^{-1/2})}{\sigma^2_{\text{min}}(\Sigma^{-1/2}) d^2}
   \end{aligned}
    \end{equation}
    where $\sigma_{\text{max}}(\Sigma^{-1/2})$, $\sigma_{\text{min}}(\Sigma^{-1/2})$ denote the maximal and minimal eigenvalue of matrix $\Sigma^{-1/2}$, and $\text{Tr}(\Sigma)$ denotes the trace of matrix $\Sigma$.
\end{theorem}

\begin{remark}
\textbf{\underline{(R1)}} The lower bound in \Cref{eq:thm2} rises as the minimal distance between the malicious characterization vector to the benign mean $\mu$ (i.e., $d$) increases. The lower bound asymptotically approaches $1$ with a sufficiently large $d$.
    It implies that when the malicious characterization vector is far away from the benign cluster (i.e., a large $d$), the malicious client number estimator has a high precision. 
   \textbf{\underline{(R2)}}
    The lower bound in \Cref{eq:thm2} also shows that when the initial guess $\tilde{K}_b$ is closer to $K_b$, the lower bound of estimate precision is higher, demonstrating the effectiveness of iterative EM optimization with \Cref{eq:opt}. 
    \textbf{\underline{(R3)}} Note that the condition of the initial guess $K_m<\tilde{K}_b<K_b$ is satisfiable by simply setting $\tilde{K}_b=\lceil K/2 \rceil$. 
\end{remark}

\textit{Proof sketch.} We first analyze the tail bound of the multivariate normal distribution as \citep{vershynin2018high}, and then derive the probabilistic relationships between the maliciousness scores of benign clients and those of malicious clients using the tail bounds. We finally upper bound the probability of overestimation and underestimation by opening up the probability formulations.
We defer the complete proof to \Cref{app:proof_thm}.

\section{Experiments}


\begin{table*}[t]
    \centering
    \small
    \caption{Marginal coverage / average set size under different Byzantine attacks with $40\%~(K_m/K=40\%)$ malicious clients. The desired marginal coverage is $0.9$. The Dirichlet parameter $\beta$ is $0.5$.
    Results that more closely align with those observed in an all-benign-client scenario (provided in \Cref{tab:benign} in \Cref{app:res}) are highlighted in bold.
    }
    \vspace{-0.7em}
    \resizebox{1.0\textwidth}{!}{\begin{tabular}{c|cc|cc|cc}
    \toprule
       \multicolumn{1}{c}{Byzantine Attack} & \multicolumn{2}{c}{Coverage Attack} & \multicolumn{2}{c}{Efficiency Attack} & \multicolumn{2}{c}{Gaussian Attack}  \\
        \multicolumn{1}{c}{Method}  & FCP & \name & FCP & \name & FCP & \name \\
     \midrule
     MNIST & 0.805 / 1.284 & \textbf{0.899} / \textbf{1.783} & 1.000 / 10.00 & \textbf{0.902} / \textbf{1.804} &  0.941 / 2.227 & \textbf{0.923} / \textbf{2.182} \\
     CIFAR-10 & 0.829 / 1.758 & \textbf{0.897} / \textbf{2.319}  & 1.000 / 10.00  & \textbf{0.892} / \textbf{2.351} & 0.970 / 3.863 & \textbf{0.921} / \textbf{2.623} \\
     Tiny-ImageNet & 0.825 / 27.84 & \textbf{0.903} / \textbf{43.47}  & 1.000 / 200.0  & \textbf{0.904} / \textbf{43.68} & 0.942 / 61.50 & \textbf{0.928} / \textbf{54.91} \\
     SHHS & 0.835 / 1.095 & \textbf{0.901} / \textbf{1.365} & 1.000 / 6.000 & \textbf{0.901} / \textbf{1.366} & 0.937 / 1.609 & \textbf{0.900} / \textbf{1.359}\\
     PathMNIST & 0.837 / 1.055 & \textbf{0.900} / \textbf{1.355} & 1.000 / 9.000 & \textbf{0.900} / \textbf{1.344} & 1.000 / 6.935 & \textbf{0.926} / \textbf{1.585}\\
    \bottomrule
    \end{tabular}}
    \label{tab:exp_main}
    \vspace{-1em}
\end{table*}

\subsection{Experiment setup}

\paragraph{Datasets} We evaluate \name on a variety of standard datasets, including MNIST \citep{deng2012mnist}, CIFAR-10 \citep{cifar}, and Tiny-ImageNet \citep{le2015tiny}.
Our evaluation of \name also cover two realistic healthcare datasets: the Sleep Heart Health Study (SHHS) dataset \citep{zhang2018national} and a pathology dataset PathMNIST \citep{medmnistv2}. 
\vspace{-0.7em}
\reb{
\paragraph{Data partition in federated conformal prediction}
Our approach of data partition adheres to the standard federated learning evaluation framework by using the Dirichlet distribution to create different label ratios across clients \citep{yurochkin2019bayesian,lin2020ensemble,wang2020federated,gao2022feddc}.
Concretely, we sample $p_{c,j} \sim \text{Dir}(\beta)$ and allocate a portion of $p_{c,j}$ instances with class $c$ to the client $j$, where $\text{Dir}(\cdot)$ denotes the Dirichlet distribution and $\beta$ is a concentration parameter ($\beta > 0$), controlling the degree of data heterogeneity among clients.
A lower $\beta$ value results in a more heterogeneous data distribution. 
By default, we set $\beta$ to 0.5 to establish a consistent level of data heterogeneity. 
 Additionally, we explore alternative methods for generating heterogeneous data that reflect demographic variations. We segment the SHHS dataset based on five attributes (wake time, N1, N2, N3, REM), distributing instances to clients based on differing attribute intervals, thereby introducing another dimension of data diversity.
}
\vspace{-0.7em}
\reb{
\paragraph{Byzantine attacks} To evaluate the robustness of \name in the Byzantine setting, we conducted comparisons with the baseline FCP \citep{lu2023federated} under three types of Byzantine attacks: (1) \textit{coverage attack} (CovAttack) involves malicious clients reporting maximized conformity scores (e.g., $1$ for LAC score \citep{sadinle2019least}) to artificially inflate the conformity score at the targeted quantile, resulting in reduced coverage; (2) \textit{efficiency attack} (EffAttack) involves malicious clients submitting minimized conformity scores (e.g., score of $0$ for LAC score) to lower the conformity score at the quantile, thereby expanding the prediction set; (3) Gaussian Attack (GauAttack) involves malicious clients dispersing random Gaussian noise with a standard deviation of $0.5$ into the scores, thereby disrupting the conformal calibration process. 
}
\vspace{-0.7em}
\paragraph{Evaluation metric} We consider the global test data set $\gD_{\text{test}}=\{(X_i,Y_i)\}_{i=1}^{N_{\text{test}}}$. We notate $C_\alpha(X_i)$ as the conformal prediction set given test sample $X_i$ and consider the desired coverage level $1-\alpha$. We evaluate with the metrics of \textit{marginal coverage} $ \sum_{i=1}^{N_{\text{test}}} \sI\left[ Y_i \in C_\alpha(X_i) \right] /{N_{\text{test}}}$ and \textit{average set size} $\sum_{i=1}^{N_{\text{test}}} \left| C_\alpha(X_i) \right| /{N_{\text{test}}}$.
Without specification, the desired coverage level $1-\alpha$ is set $0.9$.
We provide more details of experiment setups in \Cref{sec:app_exp_set}.

The codes to reproduce all the evaluation results are publicly available at \url{https://github.com/kangmintong/Rob-FCP}.

\subsection{Evaluation results}

\begin{figure*}[t]
\centering
\begin{minipage}{0.49\linewidth}
 	\centerline{\includegraphics[width=1.0\textwidth]{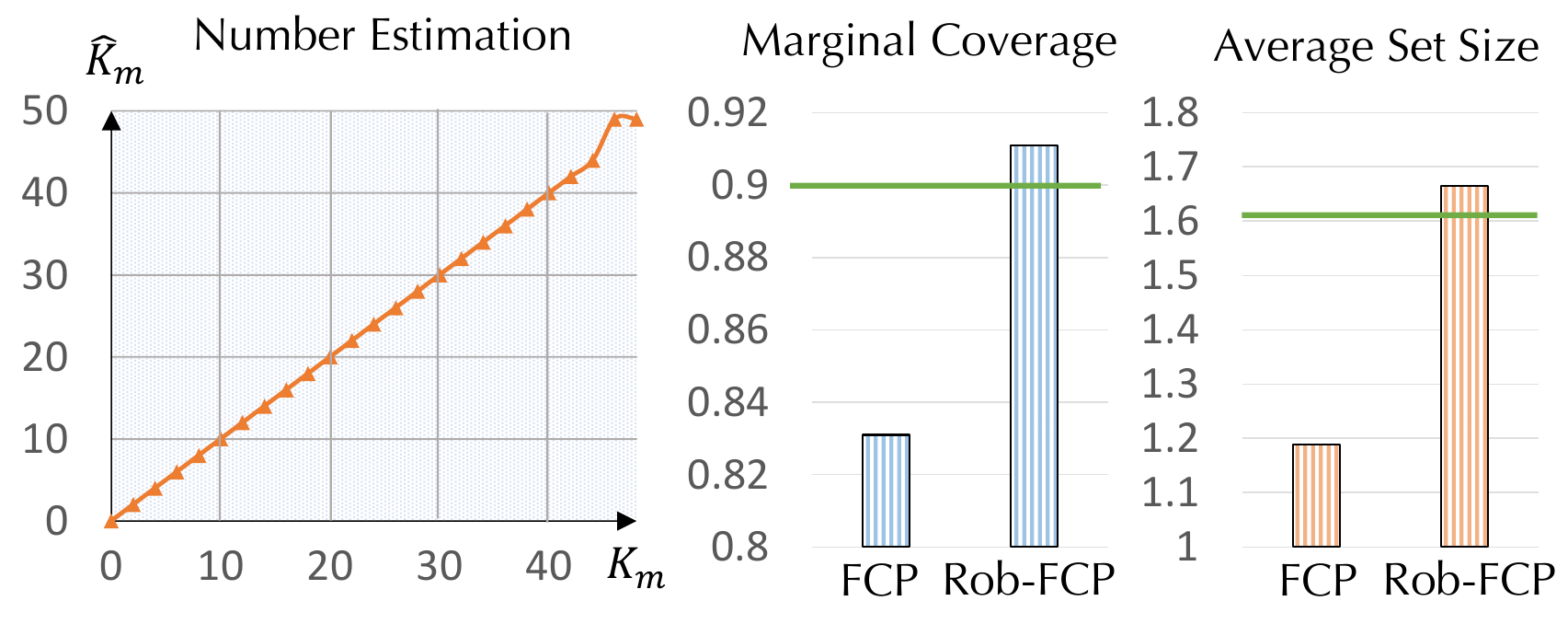}}
   \centerline{\footnotesize{(a) CIFAR-10}}
 \end{minipage}
 \hfill
 \begin{minipage}{0.49\linewidth}
 	\centerline{\includegraphics[width=1.0\textwidth]{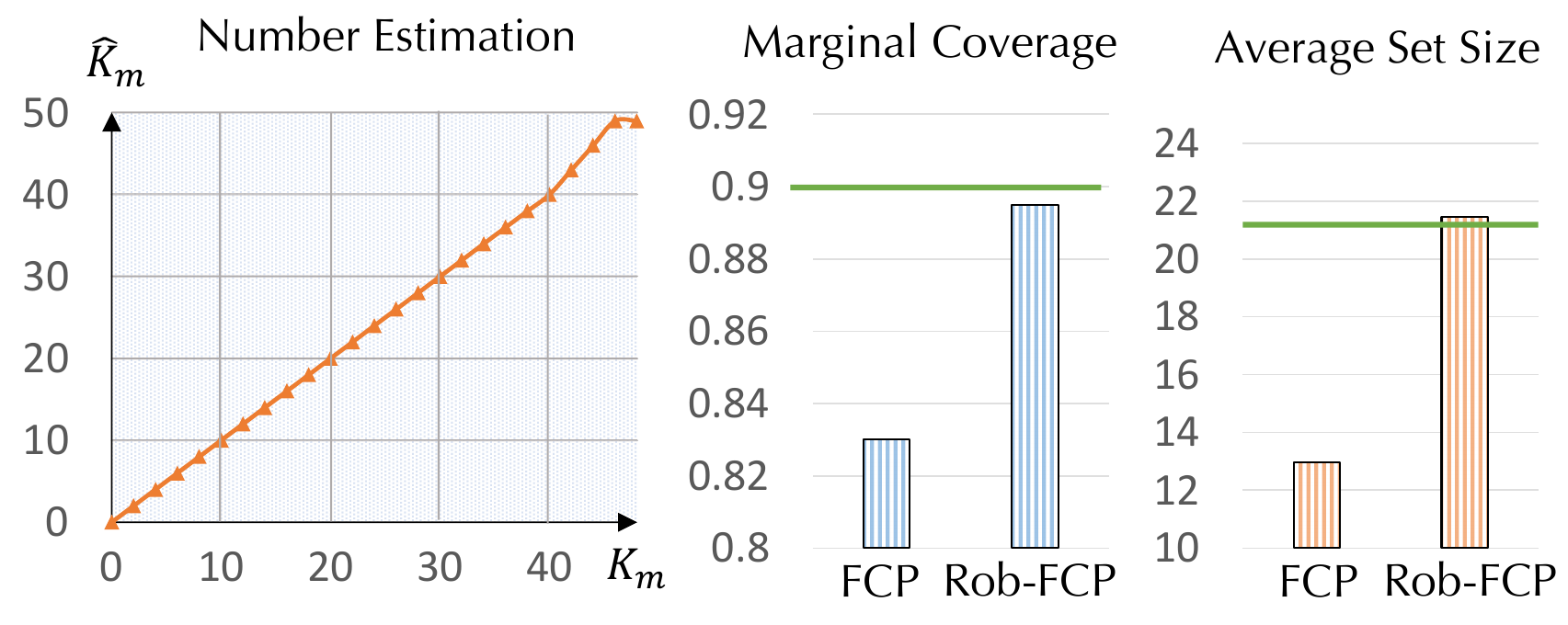}}
  \centerline{\footnotesize{(b) Tiny-ImageNet}}
 \end{minipage}
 \vspace{-0.7em}
\caption{Results of malicious client number estimation and conformal prediction performance in the setting with unknown numbers of malicious clients. The green horizontal line denotes the benign conformal performance.
\name estimates the number of malicious clients faithfully, and provides an empirical coverage rate matching the target (benign level).
}
\label{fig:exp_num_est}
\vspace{-1em}
\end{figure*}

\paragraph{Byzantine robustness of \name} 
We evaluate \name in terms of marginal coverage and average set size under coverage attack, efficiency attack, and Gaussian attacks, and compare these results against the baseline FCP.
We present the results of FCP and \name in the existence of $40\%~(K_m/K=40\%)$ malicious clients on MNIST, CIFAR-10, Tiny-ImageNet, SHHS, and PathMNIST in \Cref{tab:exp_main}.
Under Byzantine attacks, FCP shows a significant deviation from the targeted coverage level of $0.9$ and the expected benign set size.
In contrast, \name maintains comparable levels of marginal coverage and average set size, underlining its robustness.
Note that while a smaller prediction set is generally preferred for efficiency, the primary objective here is to accurately meet the targeted coverage level of $0.9$.
Further results on the resilience of \name against different portions of malicious clients (30\%, 20\%, and 10\%) are provided in \Cref{table:appendix:known} in \Cref{app:res}.
In \Cref{tab:mimick} of \Cref{app:res}, we empirically demonstrate the robustness of \name against mimick attacks \cite{karimireddy2022byzantinerobust}, which operate under a more restricted threat model that relies on knowing the score statistics of agents.
\vspace{-0.7em}
\paragraph{\name with unknown numbers of malicious clients} 
In \Cref{sec:mal_ubk}, we explore a complex Byzantine scenario where the exact count of malicious clients is not known to the defender. To address this challenge, we introduce an estimator designed to predict the number of malicious participants accurately, with theoretical guarantee as \Cref{thm2}.
Our evaluation focuses on assessing the precision of this malicious client number estimator and examining the conformal prediction performance of \name within this uncertain environment.
The results in \Cref{fig:exp_num_est} reveal that our estimator ($\hat{K}_m$) closely approximates the actual number of malicious clients ($K_m$), leading to a marginal coverage and average set size close to the benign level.
Further results across all five datasets, under a variety of Byzantine attacks, are detailed in \Cref{table:appendix:unknown} in \Cref{app:res}, confirming the effectiveness of the malicious client number estimator.


\begin{table*}[t]
    \centering
    \caption{
    \reb{Marginal coverage / average set size across varying levels of data heterogeneity, controlled by different Dirichlet parameter $\beta$. The evaluation is done under coverage attack with $40\%~(K_m/K=40\%)$ malicious clients. The desired coverage level is $0.9$. Results that more closely align with those observed in an all-benign-client scenario (provided in \Cref{tab:benign}) are highlighted in bold. }}
    \label{tab:different_beta}
    \vspace{-0.2em}
     \resizebox{1.0\textwidth}{!}{\begin{tabular}{cc|ccccc}
    \toprule
     Dataset & Method & $\beta=0.1$ & $\beta=0.3$ & $\beta=0.5$ & $\beta=0.7$ & $\beta=0.9$ \\
     \midrule
      \multirow{2}{*}{MNIST} &    FCP & 0.780 / 1.173 &  0.817 / 1.318 & 0.833 / 1.384 & 0.805 / 1.265 &  0.828 / 1.363 \\
     & \name & \textbf{0.899} / \textbf{1.806} & \textbf{0.905} / \textbf{1.809} & \textbf{0.903} / \textbf{1.827} &  \textbf{0.898} / \textbf{1.781} & \textbf{0.893} / \textbf{1.768} \\
     \hline
     \multirow{2}{*}{CIFAR-10} & FCP & 0.806 / 1.641 & 0.821 / 1.717 & 0.836 / 1.791 & 0.823 / 1.744 & 0.824 / 1.723 \\
    & \name & \textbf{0.899} / \textbf{2.260} &  \textbf{0.907} / \textbf{2.405} & \textbf{0.892} / \textbf{2.243} & \textbf{0.904} / \textbf{2.396} & \textbf{0.910} / \textbf{2.416} \\
    \hline
    \multirow{2}{*}{Tiny-ImageNet} & FCP & 0.840 / 28.625 & 0.830 / 28.192 &  0.833 / 28.340 &  0.821 / 27.140 & 0.831 / 28.751 \\
    & \name  & \textbf{0.913} / \textbf{45.872} & \textbf{0.910} / \textbf{44.972} & \textbf{0.898} / \textbf{42.571} & \textbf{0.887} / \textbf{41.219} & \textbf{0.898} / \textbf{43.298} \\
    \hline
    \multirow{2}{*}{PathMNIST} & FCP  & 0.850 / 1.106 & 0.839 / 1.065 & 0.837 / 1.055 & 0.839 / 1.065	 & 0.832 / 1.043 \\
    & \name & \textbf{0.895} / \textbf{1.311} & \textbf{0.900} / \textbf{1.355} & \textbf{0.900} / \textbf{1.355} & \textbf{0.899} / \textbf{1.354} & \textbf{0.901} / \textbf{1.363} \\
    \bottomrule
    \end{tabular}}
    \vspace{-0.7em}
\end{table*}

\begin{figure}[t]
    \centering
\includegraphics[width=1.0\linewidth]{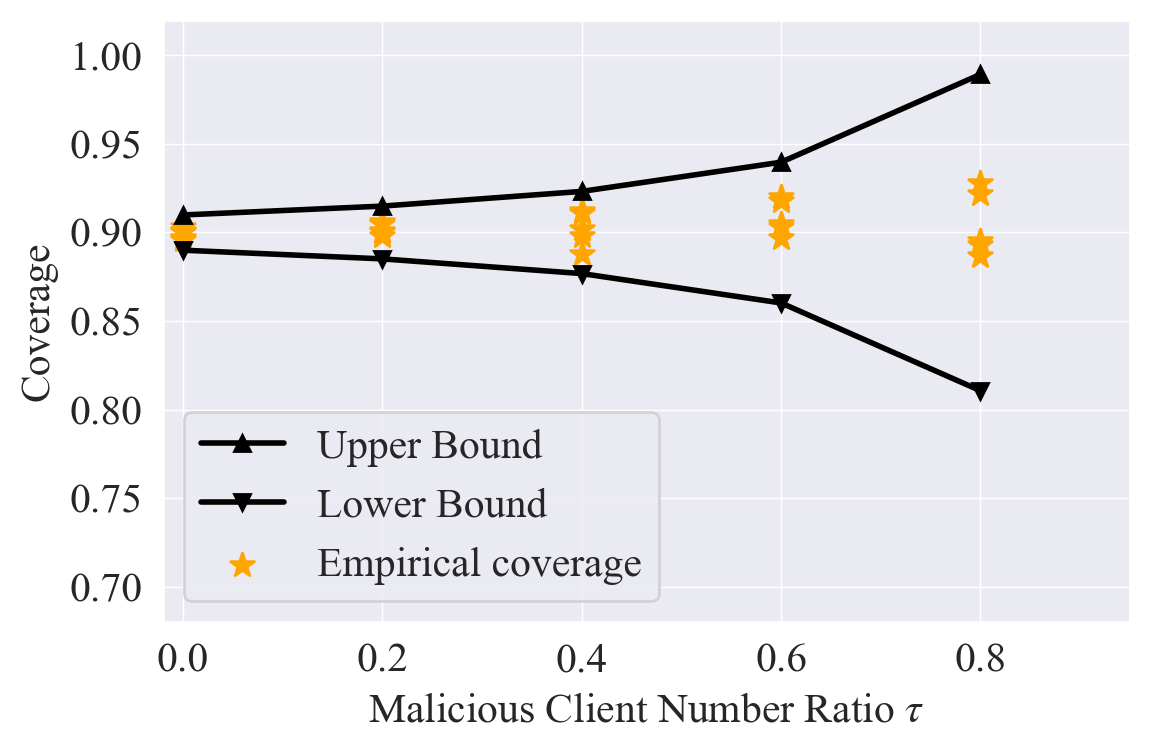}
\vspace{-1.6em}
    \caption{Upper and lower bounds of prediction coverage of \name by \Cref{thm1:improve} on Tiny-ImageNet.}
    \vspace{-1.0em}
    \label{fig:bound}
\end{figure}
\vspace{-0.95em}
\paragraph{Validation of coverage bounds of \name} In \Cref{thm1:improve}, we provide both the lower and upper bound of the coverage rate of \name, considering the ratio of malicious clients ($\tau=K_m/K_b$) and the sample sizes across clients.
In \Cref{fig:bound}, we compare these theoretical bounds of coverage rate against the observed empirical marginal coverage under Gaussian attacks on Tiny-ImageNet. 
The results demonstrate the validity and tightness of the certified coverage bounds in \Cref{thm1:improve}.

\subsection{Ablation study}
\label{sec:abla}
\paragraph{Robustness of \name across varying levels of data heterogeneity} Data heterogeneity among clients poses significant challenges to achieving precise federated conformal prediction.
To assess the resilience of \name to this issue, we conducted evaluations using various values of the Dirichlet parameter $\beta$, which modulates the degree of data heterogeneity among clients.
The results in \Cref{tab:different_beta} show that \name reliably maintains marginal coverage and average set size at levels close to the benign levels, underscoring its robustness in the face of data heterogeneity.
Furthermore, we investigate additional approaches to create heterogeneous data that mirror demographic differences. This involves dividing the SHHS dataset according to five specific attributes (wake time, N1, N2, N3, REM) and allocating instances to clients based on varying intervals of these attributes. The results in \Cref{tab:shhs_noniid} in \Cref{app:res} highlight \name's capability to effectively handle diverse forms of data heterogeneity.
\vspace{-0.7em}
\paragraph{Ablation study on conformity score distribution characterization methods} 
A pivotal aspect of \name involves the characterization of the conformity score distribution through empirical data. Our primary method utilizes histogram statistics as outlined in \Cref{eq:score2vec}. Alternatively, one could represent score samples using cluster centers derived from clustering algorithms like KMeans, or employ a parametric method such as fitting the score samples to a Gaussian distribution and characterizing them by the mean and variances of the Gaussian. Our empirical comparison of these methods, presented in \Cref{fig:abl_method_charact} and \Cref{fig:abl_method_noniid} within \Cref{app:res}, reveals that the histogram statistics approach yields superior performance. Additional ablation studies focusing on various distance measurement techniques are provided in \Cref{fig:abl_dis} in \Cref{app:res}.
\vspace{-0.7em}
\paragraph{Robustness of \name with various conformity scores} Besides applying LAC nonconformity scores, we also evaluate \name with APS conformity scores \citep{romano2020classification}.
The results in \Cref{fig:abl_cov_iid,fig:abl_cov_noniid,fig:abl_eff_iid,fig:abl_eff_noniid,fig:abl_gauss_iid,fig:abl_gauss_noniid} in \Cref{app:res} demonstrate the resilience of \name to the selection of conformity scores.
We also evaluate the runtime and show the efficiency of \name in \Cref{tab:runtime} in \Cref{app:res}.



\section{Related work}
\textbf{Conformal prediction} is a statistical tool to construct the prediction set with guaranteed prediction coverage \citep{jin2023sensitivity,solari2022multi,yang2021finite,romano2020classification,barber2021predictive,kang2024c,kang2024colep}, assuming exchangeable data.
\reb{Recently, \textbf{federated conformal prediction} (FCP) \citep{lu2021distribution,lu2023federated} adapts the conformal prediction to the federated learning and provides a rigorous guarantee of the distributed uncertainty quantification framework.} 
DP-FCP \citep{plassier2023conformal} proposes federated CP with differential privacy guarantees and provides valid coverage bounds under label shifting among clients.
\reb{
\citeauthor{humbert2023one} propose a quantile-of-quantiles estimator for federated conformal prediction with a one-round communication and provide a locally differentially private version.
}
WFCP \citep{zhu2023federated} applies FCP to wireless communication.
However, no prior works explore the robustness of FCP against Byzantine agents which can report malicious statistics to downgrade the conformal prediction performance. We are the first to propose a robust FCP method with valid and tight coverage guarantees.

\textbf{Byzantine learning} \citep{driscoll2003byzantine,awerbuch2002demand,lamport2019byzantine} refers to methods that can robustly aggregate updates from potentially malicious or faulty worker nodes in the distributed setting. Specifically, a line of works \citep{guerraoui2018hidden,pillutla2022robust,data2021byzantine,karimireddy2020byzantine,yi2022robust} studies the resilience to Byzantine failures of distributed implementations of Stochastic Gradient Descent (SGD) and proposes different metrics to identify malicious gradients such as gradient norm \citep{blanchard2017machine} and coordinate-wise trimmed mean \citep{yin2018byzantine}. However, the metrics are designed for the stability and convergence of distributed optimization and cannot be applied to the Byzantine FCP setting to provide rigorous coverage guarantees. In contrast, we propose \name to perform Byzantine-robust distributed uncertainty quantification and provide valid and tight coverage bounds theoretically.

\begin{figure}[t]
    \centering
    \includegraphics[width=\linewidth]{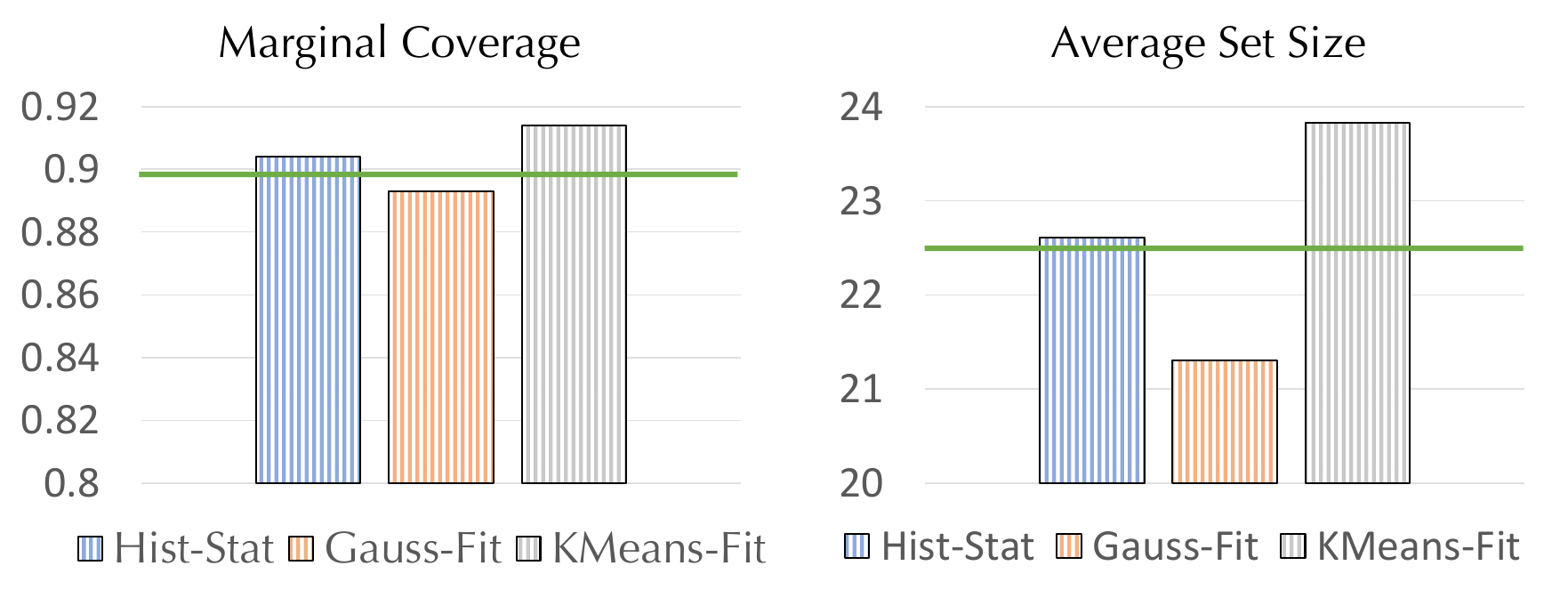}
    \vspace{-1.1em}
    \caption{Marginal coverage / average set size under coverage attack with 40\% malicious clients on Tiny-ImageNet. The green horizontal line denotes the benign marginal coverage and average set size without any malicious clients.}
    \label{fig:abl_method_charact}
    \vspace{-1.3em}
\end{figure}

\section{Conclusion}
In this paper, we propose \name, a certifiably Byzantine-robust federated conformal prediction algorithm with rigorous coverage guarantees.
\name sketches the local samples of conformity scores with characterization vectors and detects the malicious clients in the vector space. We empirically show the robustness of \name against Byzantine failures on five datasets and validate the theoretical coverage bounds.

\section*{Acknowledgements}

This work is supported by the National Science Foundation under grant No. 1910100, No. 2046726, No. 2229876, DARPA GARD, the National Aeronautics and Space Administration (NASA) under grant No. 80NSSC20M0229, the Alfred P. Sloan Fellowship, the Amazon research award, and the eBay research award.
This work is supported by NSF award SCH-2205289, SCH-2014438, IIS-1838042, NIH award R01 1R01NS107291-01.

\section*{Impact Statement}

We do not see potential ethical or societal issues about \name. In contrast, \name is a robust framework against malicious clients in the federated conformal prediction settings and can safeguard the applications of FCP in safety-critical scenarios such as healthcare and medical diagnosis.


\bibliography{iclr2024_conference}

\begin{thebibliography}{64}
\providecommand{\natexlab}[1]{#1}
\providecommand{\url}[1]{\texttt{#1}}
\expandafter\ifx\csname urlstyle\endcsname\relax
  \providecommand{\doi}[1]{doi: #1}\else
  \providecommand{\doi}{doi: \begingroup \urlstyle{rm}\Url}\fi

\bibitem[Ahmad et~al.(2018)Ahmad, Eckert, and Teredesai]{ahmad2018interpretable}
Ahmad, M.~A., Eckert, C., and Teredesai, A.
\newblock Interpretable machine learning in healthcare.
\newblock In \emph{Proceedings of the 2018 ACM international conference on bioinformatics, computational biology, and health informatics}, pp.\  559--560, 2018.

\bibitem[Andrew et~al.(2021)Andrew, Thakkar, McMahan, and Ramaswamy]{andrew2021differentially}
Andrew, G., Thakkar, O., McMahan, B., and Ramaswamy, S.
\newblock Differentially private learning with adaptive clipping.
\newblock \emph{Advances in Neural Information Processing Systems}, 34:\penalty0 17455--17466, 2021.

\bibitem[Awerbuch et~al.(2002)Awerbuch, Holmer, Nita-Rotaru, and Rubens]{awerbuch2002demand}
Awerbuch, B., Holmer, D., Nita-Rotaru, C., and Rubens, H.
\newblock An on-demand secure routing protocol resilient to byzantine failures.
\newblock In \emph{Proceedings of the 1st ACM workshop on Wireless security}, pp.\  21--30, 2002.

\bibitem[Balasubramanian et~al.(2014)Balasubramanian, Ho, and Vovk]{balasubramanian2014conformal}
Balasubramanian, V., Ho, S.-S., and Vovk, V.
\newblock \emph{Conformal prediction for reliable machine learning: theory, adaptations and applications}.
\newblock Newnes, 2014.

\bibitem[Barber et~al.(2021)Barber, Candes, Ramdas, and Tibshirani]{barber2021predictive}
Barber, R.~F., Candes, E.~J., Ramdas, A., and Tibshirani, R.~J.
\newblock Predictive inference with the jackknife+.
\newblock 2021.

\bibitem[Blanchard et~al.(2017)Blanchard, El~Mhamdi, Guerraoui, and Stainer]{blanchard2017machine}
Blanchard, P., El~Mhamdi, E.~M., Guerraoui, R., and Stainer, J.
\newblock Machine learning with adversaries: Byzantine tolerant gradient descent.
\newblock \emph{Advances in neural information processing systems}, 30, 2017.

\bibitem[Bonawitz et~al.(2019)Bonawitz, Eichner, Grieskamp, Huba, Ingerman, Ivanov, Kiddon, Kone{\v{c}}n{\`y}, Mazzocchi, McMahan, et~al.]{bonawitz2019towards}
Bonawitz, K., Eichner, H., Grieskamp, W., Huba, D., Ingerman, A., Ivanov, V., Kiddon, C., Kone{\v{c}}n{\`y}, J., Mazzocchi, S., McMahan, B., et~al.
\newblock Towards federated learning at scale: System design.
\newblock \emph{Proceedings of machine learning and systems}, 1:\penalty0 374--388, 2019.

\bibitem[Carnap \& Jeffrey(1980)Carnap and Jeffrey]{carnap1980studies}
Carnap, R. and Jeffrey, R.~C.
\newblock \emph{Studies in inductive logic and probability}, volume~2.
\newblock Univ of California Press, 1980.

\bibitem[Data \& Diggavi(2021)Data and Diggavi]{data2021byzantine}
Data, D. and Diggavi, S.
\newblock Byzantine-resilient high-dimensional sgd with local iterations on heterogeneous data.
\newblock In \emph{International Conference on Machine Learning}, pp.\  2478--2488. PMLR, 2021.

\bibitem[Deng(2012)]{deng2012mnist}
Deng, L.
\newblock The mnist database of handwritten digit images for machine learning research [best of the web].
\newblock \emph{IEEE signal processing magazine}, 29\penalty0 (6):\penalty0 141--142, 2012.

\bibitem[Driscoll et~al.(2003)Driscoll, Hall, Sivencrona, and Zumsteg]{driscoll2003byzantine}
Driscoll, K., Hall, B., Sivencrona, H., and Zumsteg, P.
\newblock Byzantine fault tolerance, from theory to reality.
\newblock In \emph{International Conference on Computer Safety, Reliability, and Security}, pp.\  235--248. Springer, 2003.

\bibitem[Dunning(2021)]{Dunning2021100049}
Dunning, T.
\newblock The t-digest: Efficient estimates of distributions.
\newblock \emph{Software Impacts}, 7:\penalty0 100049, 2021.
\newblock ISSN 2665-9638.
\newblock \doi{https://doi.org/10.1016/j.simpa.2020.100049}.
\newblock URL \url{https://www.sciencedirect.com/science/article/pii/S2665963820300403}.

\bibitem[Dvoretzky et~al.(1956)Dvoretzky, Kiefer, and Wolfowitz]{dvoretzky1956asymptotic}
Dvoretzky, A., Kiefer, J., and Wolfowitz, J.
\newblock Asymptotic minimax character of the sample distribution function and of the classical multinomial estimator.
\newblock \emph{The Annals of Mathematical Statistics}, pp.\  642--669, 1956.

\bibitem[Erickson et~al.(2017)Erickson, Korfiatis, Akkus, and Kline]{erickson2017machine}
Erickson, B.~J., Korfiatis, P., Akkus, Z., and Kline, T.~L.
\newblock Machine learning for medical imaging.
\newblock \emph{Radiographics}, 37\penalty0 (2):\penalty0 505--515, 2017.

\bibitem[Gao et~al.(2022)Gao, Fu, Li, Chen, Xu, and Xu]{gao2022feddc}
Gao, L., Fu, H., Li, L., Chen, Y., Xu, M., and Xu, C.-Z.
\newblock Feddc: Federated learning with non-iid data via local drift decoupling and correction.
\newblock In \emph{Proceedings of the IEEE/CVF conference on computer vision and pattern recognition}, pp.\  10112--10121, 2022.

\bibitem[Guerraoui et~al.(2018)Guerraoui, Rouault, et~al.]{guerraoui2018hidden}
Guerraoui, R., Rouault, S., et~al.
\newblock The hidden vulnerability of distributed learning in byzantium.
\newblock In \emph{International Conference on Machine Learning}, pp.\  3521--3530. PMLR, 2018.

\bibitem[He et~al.(2016)He, Zhang, Ren, and Sun]{He_2016_CVPR}
He, K., Zhang, X., Ren, S., and Sun, J.
\newblock Deep residual learning for image recognition.
\newblock In \emph{Proceedings of the IEEE Conference on Computer Vision and Pattern Recognition (CVPR)}, June 2016.

\bibitem[Humbert et~al.(2023)Humbert, Le~Bars, Bellet, and Arlot]{humbert2023one}
Humbert, P., Le~Bars, B., Bellet, A., and Arlot, S.
\newblock One-shot federated conformal prediction.
\newblock In \emph{International Conference on Machine Learning}, pp.\  14153--14177. PMLR, 2023.

\bibitem[Jin et~al.(2023)Jin, Ren, and Cand{\`e}s]{jin2023sensitivity}
Jin, Y., Ren, Z., and Cand{\`e}s, E.~J.
\newblock Sensitivity analysis of individual treatment effects: A robust conformal inference approach.
\newblock \emph{Proceedings of the National Academy of Sciences}, 120\penalty0 (6):\penalty0 e2214889120, 2023.

\bibitem[Kairouz et~al.(2021)Kairouz, McMahan, Avent, Bellet, Bennis, Bhagoji, Bonawitz, Charles, Cormode, Cummings, et~al.]{kairouz2021advances}
Kairouz, P., McMahan, H.~B., Avent, B., Bellet, A., Bennis, M., Bhagoji, A.~N., Bonawitz, K., Charles, Z., Cormode, G., Cummings, R., et~al.
\newblock Advances and open problems in federated learning.
\newblock \emph{Foundations and Trends{\textregistered} in Machine Learning}, 14\penalty0 (1--2):\penalty0 1--210, 2021.

\bibitem[Kang et~al.(2024{\natexlab{a}})Kang, G{\"u}rel, Li, and Li]{kang2024colep}
Kang, M., G{\"u}rel, N.~M., Li, L., and Li, B.
\newblock Colep: Certifiably robust learning-reasoning conformal prediction via probabilistic circuits.
\newblock \emph{arXiv preprint arXiv:2403.11348}, 2024{\natexlab{a}}.

\bibitem[Kang et~al.(2024{\natexlab{b}})Kang, G{\"u}rel, Yu, Song, and Li]{kang2024c}
Kang, M., G{\"u}rel, N.~M., Yu, N., Song, D., and Li, B.
\newblock C-rag: Certified generation risks for retrieval-augmented language models.
\newblock \emph{arXiv preprint arXiv:2402.03181}, 2024{\natexlab{b}}.

\bibitem[Karimireddy et~al.(2020)Karimireddy, He, and Jaggi]{karimireddy2020byzantine}
Karimireddy, S.~P., He, L., and Jaggi, M.
\newblock Byzantine-robust learning on heterogeneous datasets via bucketing.
\newblock \emph{arXiv preprint arXiv:2006.09365}, 2020.

\bibitem[Karimireddy et~al.(2022)Karimireddy, He, and Jaggi]{karimireddy2022byzantinerobust}
Karimireddy, S.~P., He, L., and Jaggi, M.
\newblock Byzantine-robust learning on heterogeneous datasets via bucketing.
\newblock In \emph{International Conference on Learning Representations}, 2022.
\newblock URL \url{https://openreview.net/forum?id=jXKKDEi5vJt}.

\bibitem[Kompa et~al.(2021)Kompa, Snoek, and Beam]{kompa2021second}
Kompa, B., Snoek, J., and Beam, A.~L.
\newblock Second opinion needed: communicating uncertainty in medical machine learning.
\newblock \emph{NPJ Digital Medicine}, 4\penalty0 (1):\penalty0 4, 2021.

\bibitem[Kone{\v{c}}n{\`y} et~al.(2016)Kone{\v{c}}n{\`y}, McMahan, Yu, Richt{\'a}rik, Suresh, and Bacon]{konevcny2016federated}
Kone{\v{c}}n{\`y}, J., McMahan, H.~B., Yu, F.~X., Richt{\'a}rik, P., Suresh, A.~T., and Bacon, D.
\newblock Federated learning: Strategies for improving communication efficiency.
\newblock \emph{arXiv preprint arXiv:1610.05492}, 2016.

\bibitem[Krizhevsky et~al.()Krizhevsky, Nair, and Hinton]{cifar}
Krizhevsky, A., Nair, V., and Hinton, G.
\newblock Cifar-10 (canadian institute for advanced research).
\newblock URL \url{http://www.cs.toronto.edu/~kriz/cifar.html}.

\bibitem[Lamport et~al.(2019)Lamport, Shostak, and Pease]{lamport2019byzantine}
Lamport, L., Shostak, R., and Pease, M.
\newblock The byzantine generals problem.
\newblock In \emph{Concurrency: the works of leslie lamport}, pp.\  203--226. 2019.

\bibitem[Le \& Yang(2015)Le and Yang]{le2015tiny}
Le, Y. and Yang, X.
\newblock Tiny imagenet visual recognition challenge.
\newblock \emph{CS 231N}, 7\penalty0 (7):\penalty0 3, 2015.

\bibitem[Li et~al.(2022{\natexlab{a}})Li, Diao, Chen, and He]{li2022federated}
Li, Q., Diao, Y., Chen, Q., and He, B.
\newblock Federated learning on non-iid data silos: An experimental study.
\newblock In \emph{2022 IEEE 38th International Conference on Data Engineering (ICDE)}, pp.\  965--978. IEEE, 2022{\natexlab{a}}.

\bibitem[Li et~al.(2022{\natexlab{b}})Li, Choi, Chung, Kushman, Schrittwieser, Leblond, Eccles, Keeling, Gimeno, Dal~Lago, et~al.]{li2022competition}
Li, Y., Choi, D., Chung, J., Kushman, N., Schrittwieser, J., Leblond, R., Eccles, T., Keeling, J., Gimeno, F., Dal~Lago, A., et~al.
\newblock Competition-level code generation with alphacode.
\newblock \emph{Science}, 378\penalty0 (6624):\penalty0 1092--1097, 2022{\natexlab{b}}.

\bibitem[Lin et~al.(2020)Lin, Kong, Stich, and Jaggi]{lin2020ensemble}
Lin, T., Kong, L., Stich, S.~U., and Jaggi, M.
\newblock Ensemble distillation for robust model fusion in federated learning.
\newblock \emph{Advances in Neural Information Processing Systems}, 33:\penalty0 2351--2363, 2020.

\bibitem[Liu et~al.(2023)Liu, Chen, Lyu, Wu, Wu, and Chen]{liu2023byzantine}
Liu, Y., Chen, C., Lyu, L., Wu, F., Wu, S., and Chen, G.
\newblock Byzantine-robust learning on heterogeneous data via gradient splitting.
\newblock 2023.

\bibitem[Lu \& Kalpathy-Cramer(2021)Lu and Kalpathy-Cramer]{lu2021distribution}
Lu, C. and Kalpathy-Cramer, J.
\newblock Distribution-free federated learning with conformal predictions.
\newblock \emph{arXiv preprint arXiv:2110.07661}, 2021.

\bibitem[Lu et~al.(2023)Lu, Yu, Karimireddy, Jordan, and Raskar]{lu2023federated}
Lu, C., Yu, Y., Karimireddy, S.~P., Jordan, M.~I., and Raskar, R.
\newblock Federated conformal predictors for distributed uncertainty quantification.
\newblock \emph{arXiv preprint arXiv:2305.17564}, 2023.

\bibitem[McMahan et~al.(2017)McMahan, Moore, Ramage, Hampson, and y~Arcas]{mcmahan2017communication}
McMahan, B., Moore, E., Ramage, D., Hampson, S., and y~Arcas, B.~A.
\newblock Communication-efficient learning of deep networks from decentralized data.
\newblock In \emph{Artificial intelligence and statistics}, pp.\  1273--1282. PMLR, 2017.

\bibitem[McMahan et~al.(2016)McMahan, Moore, Ramage, and y~Arcas]{mcmahan2016federated}
McMahan, H.~B., Moore, E., Ramage, D., and y~Arcas, B.~A.
\newblock Federated learning of deep networks using model averaging.
\newblock \emph{arXiv preprint arXiv:1602.05629}, 2:\penalty0 2, 2016.

\bibitem[Papadopoulos et~al.(2002)Papadopoulos, Proedrou, Vovk, and Gammerman]{papadopoulos2002inductive}
Papadopoulos, H., Proedrou, K., Vovk, V., and Gammerman, A.
\newblock Inductive confidence machines for regression.
\newblock In \emph{Machine Learning: ECML 2002: 13th European Conference on Machine Learning Helsinki, Finland, August 19--23, 2002 Proceedings 13}, pp.\  345--356. Springer, 2002.

\bibitem[Park et~al.(2021)Park, Han, Choi, and Moon]{park2021sageflow}
Park, J., Han, D.-J., Choi, M., and Moon, J.
\newblock Sageflow: Robust federated learning against both stragglers and adversaries.
\newblock \emph{Advances in neural information processing systems}, 34:\penalty0 840--851, 2021.

\bibitem[Pillutla et~al.(2022)Pillutla, Kakade, and Harchaoui]{pillutla2022robust}
Pillutla, K., Kakade, S.~M., and Harchaoui, Z.
\newblock Robust aggregation for federated learning.
\newblock \emph{IEEE Transactions on Signal Processing}, 70:\penalty0 1142--1154, 2022.

\bibitem[Plassier et~al.(2023)Plassier, Makni, Rubashevskii, Moulines, and Panov]{plassier2023conformal}
Plassier, V., Makni, M., Rubashevskii, A., Moulines, E., and Panov, M.
\newblock Conformal prediction for federated uncertainty quantification under label shift.
\newblock \emph{arXiv preprint arXiv:2306.05131}, 2023.

\bibitem[Romano et~al.(2020)Romano, Sesia, and Candes]{romano2020classification}
Romano, Y., Sesia, M., and Candes, E.
\newblock Classification with valid and adaptive coverage.
\newblock In Larochelle, H., Ranzato, M., Hadsell, R., Balcan, M., and Lin, H. (eds.), \emph{Advances in Neural Information Processing Systems}, volume~33, pp.\  3581--3591. Curran Associates, Inc., 2020.
\newblock URL \url{https://proceedings.neurips.cc/paper/2020/file/244edd7e85dc81602b7615cd705545f5-Paper.pdf}.

\bibitem[Sadinle et~al.(2019)Sadinle, Lei, and Wasserman]{sadinle2019least}
Sadinle, M., Lei, J., and Wasserman, L.
\newblock Least ambiguous set-valued classifiers with bounded error levels.
\newblock \emph{Journal of the American Statistical Association}, 114\penalty0 (525):\penalty0 223--234, 2019.

\bibitem[Severini(2005)]{severini2005elements}
Severini, T.~A.
\newblock \emph{Elements of distribution theory}, volume~17.
\newblock Cambridge University Press, 2005.

\bibitem[Shafer \& Vovk(2008)Shafer and Vovk]{shafer2008tutorial}
Shafer, G. and Vovk, V.
\newblock A tutorial on conformal prediction.
\newblock \emph{Journal of Machine Learning Research}, 9\penalty0 (3), 2008.

\bibitem[Shejwalkar \& Houmansadr(2021)Shejwalkar and Houmansadr]{shejwalkar2021manipulating}
Shejwalkar, V. and Houmansadr, A.
\newblock Manipulating the byzantine: Optimizing model poisoning attacks and defenses for federated learning.
\newblock In \emph{NDSS}, 2021.

\bibitem[Smith et~al.(2017)Smith, Chiang, Sanjabi, and Talwalkar]{smith2017federated}
Smith, V., Chiang, C.-K., Sanjabi, M., and Talwalkar, A.~S.
\newblock Federated multi-task learning.
\newblock \emph{Advances in neural information processing systems}, 30, 2017.

\bibitem[Solari \& Djordjilovi{\'c}(2022)Solari and Djordjilovi{\'c}]{solari2022multi}
Solari, A. and Djordjilovi{\'c}, V.
\newblock Multi split conformal prediction.
\newblock \emph{Statistics \& Probability Letters}, 184:\penalty0 109395, 2022.

\bibitem[Terrell \& Scott(1992)Terrell and Scott]{terrell1992variable}
Terrell, G.~R. and Scott, D.~W.
\newblock Variable kernel density estimation.
\newblock \emph{The Annals of Statistics}, pp.\  1236--1265, 1992.

\bibitem[Vaswani et~al.(2017)Vaswani, Shazeer, Parmar, Uszkoreit, Jones, Gomez, Kaiser, and Polosukhin]{vaswani2017attention}
Vaswani, A., Shazeer, N., Parmar, N., Uszkoreit, J., Jones, L., Gomez, A.~N., Kaiser, {\L}., and Polosukhin, I.
\newblock Attention is all you need.
\newblock \emph{Advances in neural information processing systems}, 30, 2017.

\bibitem[Vershynin(2018)]{vershynin2018high}
Vershynin, R.
\newblock \emph{High-dimensional probability: An introduction with applications in data science}, volume~47.
\newblock Cambridge university press, 2018.

\bibitem[Vovk et~al.(2005)Vovk, Gammerman, and Shafer]{vovk2005algorithmic}
Vovk, V., Gammerman, A., and Shafer, G.
\newblock \emph{Algorithmic learning in a random world}, volume~29.
\newblock Springer, 2005.

\bibitem[Wallis(2013)]{wallis2013binomial}
Wallis, S.
\newblock Binomial confidence intervals and contingency tests: mathematical fundamentals and the evaluation of alternative methods.
\newblock \emph{Journal of Quantitative Linguistics}, 20\penalty0 (3):\penalty0 178--208, 2013.

\bibitem[Wang et~al.(2020)Wang, Yurochkin, Sun, Papailiopoulos, and Khazaeni]{wang2020federated}
Wang, H., Yurochkin, M., Sun, Y., Papailiopoulos, D., and Khazaeni, Y.
\newblock Federated learning with matched averaging.
\newblock \emph{arXiv preprint arXiv:2002.06440}, 2020.

\bibitem[Yang et~al.(2023)Yang, Shi, Wei, Liu, Zhao, Ke, Pfister, and Ni]{medmnistv2}
Yang, J., Shi, R., Wei, D., Liu, Z., Zhao, L., Ke, B., Pfister, H., and Ni, B.
\newblock Medmnist v2-a large-scale lightweight benchmark for 2d and 3d biomedical image classification.
\newblock \emph{Scientific Data}, 10\penalty0 (1):\penalty0 41, 2023.

\bibitem[Yang et~al.(2019)Yang, Liu, Chen, and Tong]{yang2019federated}
Yang, Q., Liu, Y., Chen, T., and Tong, Y.
\newblock Federated machine learning: Concept and applications.
\newblock \emph{ACM Transactions on Intelligent Systems and Technology (TIST)}, 10\penalty0 (2):\penalty0 1--19, 2019.

\bibitem[Yang \& Kuchibhotla(2021)Yang and Kuchibhotla]{yang2021finite}
Yang, Y. and Kuchibhotla, A.~K.
\newblock Finite-sample efficient conformal prediction.
\newblock \emph{arXiv preprint arXiv:2104.13871}, 2021.

\bibitem[Yi et~al.(2022)Yi, Wu, Zhang, Zhu, Qi, Sun, and Xie]{yi2022robust}
Yi, J., Wu, F., Zhang, H., Zhu, B., Qi, T., Sun, G., and Xie, X.
\newblock Robust quantity-aware aggregation for federated learning.
\newblock \emph{arXiv preprint arXiv:2205.10848}, 2022.

\bibitem[Yin et~al.(2018)Yin, Chen, Kannan, and Bartlett]{yin2018byzantine}
Yin, D., Chen, Y., Kannan, R., and Bartlett, P.
\newblock Byzantine-robust distributed learning: Towards optimal statistical rates.
\newblock In \emph{International Conference on Machine Learning}, pp.\  5650--5659. PMLR, 2018.

\bibitem[Yurochkin et~al.(2019)Yurochkin, Agarwal, Ghosh, Greenewald, Hoang, and Khazaeni]{yurochkin2019bayesian}
Yurochkin, M., Agarwal, M., Ghosh, S., Greenewald, K., Hoang, N., and Khazaeni, Y.
\newblock Bayesian nonparametric federated learning of neural networks.
\newblock In \emph{International conference on machine learning}, pp.\  7252--7261. PMLR, 2019.

\bibitem[Zhang et~al.(2018)Zhang, Cui, Mueller, Tao, Kim, Rueschman, Mariani, Mobley, and Redline]{zhang2018national}
Zhang, G.-Q., Cui, L., Mueller, R., Tao, S., Kim, M., Rueschman, M., Mariani, S., Mobley, D., and Redline, S.
\newblock The national sleep research resource: towards a sleep data commons.
\newblock \emph{Journal of the American Medical Informatics Association}, 25\penalty0 (10):\penalty0 1351--1358, 2018.

\bibitem[Zhang et~al.(2022)Zhang, Chen, Hong, Wu, and Yi]{zhang2022understanding}
Zhang, X., Chen, X., Hong, M., Wu, Z.~S., and Yi, J.
\newblock Understanding clipping for federated learning: Convergence and client-level differential privacy.
\newblock In \emph{International Conference on Machine Learning, ICML 2022}, 2022.

\bibitem[Zheng et~al.(2021)Zheng, Chen, Long, and Su]{zheng2021federated}
Zheng, Q., Chen, S., Long, Q., and Su, W.
\newblock Federated f-differential privacy.
\newblock In \emph{International Conference on Artificial Intelligence and Statistics}, pp.\  2251--2259. PMLR, 2021.

\bibitem[Zhu et~al.(2023)Zhu, Zecchin, Park, Guo, Feng, and Simeone]{zhu2023federated}
Zhu, M., Zecchin, M., Park, S., Guo, C., Feng, C., and Simeone, O.
\newblock Federated inference with reliable uncertainty quantification over wireless channels via conformal prediction.
\newblock \emph{arXiv preprint arXiv:2308.04237}, 2023.

\end{thebibliography}
\bibliographystyle{icml2024}

\newpage
\appendix
\onecolumn

\DoToC

\newpage
\section{Limitations and future works}

One possible limitation of \name may lie in the restriction of the targeted Byzantine threat model. We mainly consider the Byzantine setting where a certain ratio of malicious clients reports arbitrary conformity score statistics. In such a Byzantine case, the break point is $\lceil K/2 \rceil$, indicating that any algorithm cannot tolerate $\lceil K/2 \rceil$ or more malicious clients. However, in practice, malicious clients have the flexibility of only manipulating partial conformity scores. In this case, the potential break point is a function of the maximal ratio of manipulated scores for each client and can be larger than $\lceil K/2 \rceil$. Therefore, it is interesting for future work to analyze the break point of robust FCP algorithms with respect to the total manipulation sizes and budgets of manipulation sizes for each client.
Another threat model worthy of exploration in future work is the adversarial setting in FCP. In the adversarial setting, malicious clients can only manipulate the data samples instead of the conformity scores to downgrade the FCP performance. Therefore, potential defenses can consider adversarial conformal training procedures to collaboratively train a robust FCP model against perturbations in the data space.

\reb{
To provide differential privacy guarantees of Rob-FCP, one practical approach is to add privacy-preserving noises to the characterization vectors before uploading them to the server. Essentially, we can view the characterization vector as the gradient in the setting of FL with differential privacy (DP) and add Gaussian noises to the characterization vector with differential privacy guarantees as a function of the scale of noises, which can be achieved by drawing analogy from the FL with DP setting \citep{zheng2021federated,andrew2021differentially,zhang2022understanding}. Therefore, practically implementing the differential-private version of Rob-FCP is possible and straightforward.
}



\section{Additional related work}
\label{app:related}
\textbf{Byzantine learning} \citep{driscoll2003byzantine,awerbuch2002demand,lamport2019byzantine} refers to methods that can robustly aggregate updates from potentially malicious or faulty worker nodes in the distributed setting. Specifically, a line of works \citep{guerraoui2018hidden,pillutla2022robust,data2021byzantine,karimireddy2020byzantine,yi2022robust} studies the resilience to Byzantine failures of distributed implementations of Stochastic Gradient Descent (SGD) and proposes different metrics to identify malicious gradients such as gradient norm \citep{blanchard2017machine} and coordinate-wise trimmed mean \citep{yin2018byzantine}. However, the metrics are designed for the stability and convergence of distributed optimization and cannot be applied to the Byzantine FCP setting to provide rigorous coverage guarantees. In contrast, we propose \name to perform Byzantine-robust distributed uncertainty quantification and provide valid and tight coverage bounds theoretically.

\section{Omitted proofs}
\label{app:proof}
\subsection{Proof of \Cref{thm1:improve}}
\label{app:proof_thm1}

Before proving \Cref{thm1:improve}, we first prove the following lemma.
\begin{lemma}
\label{thm1}
    For $K$ clients including $K_b$ benign clients and $K_m:=K-K_b$ malicious clients, each client reports a characterization vector $\rvv^{(k)} \in \Delta^H$ $(k \in [K])$ and a quantity $n_k \in \mathbb{Z}^+$ $(k \in [K])$ to the server. 
    Suppose that the reported characterization vectors of benign clients are sampled from the same underlying multinomial distribution $\gD$, while those of malicious clients can be arbitrary. 
    Let $\epsilon$ be the estimation error of the data sketching by characterization vectors as illustrated in \Cref{eq:marginal_eps}.
    Under the assumption that $K_m < K_b$, the following holds with probability $1-\beta$:
   \begin{equation}
    \begin{aligned}
        & \sP\left[ Y_{\text{test}} \in \hat{C}_\alpha(X_{\text{test}})\right] \ge 1-\alpha -\dfrac{\epsilon n_b + 1}{n_b+K_b} - \dfrac{H\Phi^{-1}({1-\beta/2HK_b})}{2 \sqrt{n_b}} \left( 1+\dfrac{N_m}{n_b}\dfrac{2}{1-\tau} \right), \\
        & \sP\left[ Y_{\text{test}} \in \hat{C}_\alpha(X_{\text{test}})\right] \le 1-\alpha + \epsilon + \dfrac{K_b}{n_b+K_b} + \dfrac{H\Phi^{-1}({1-\beta/2HK_b})}{2 \sqrt{n_b}} \left( 1+\dfrac{N_m}{n_b}\dfrac{2}{1-\tau} \right).
    \end{aligned}
    \end{equation}
    where $\tau = K_m / K_b$ is the ratio of the number of malicious clients and the number of benign clients, $N_m:=\sum_{k \in [K]\backslash[K_b]} n_k$ is the total sample size of malicious clients, $n_b:=\min_{k'\in[K_b]}n_{k'}$ is the minimal sample size of benign clients, and $\Phi^{-1}(\cdot)$ denotes the inverse of the cumulative distribution function (CDF) of standard normal distribution.
\end{lemma}
\begin{proof}
    The proof consists of 3 parts: (a) concentration analysis of the characterization vectors $\rvv^{(k)}$ for benign clients ($1\le k \le K_b$), (b) analysis of the algorithm of the identification of malicious clients, and (c) analysis of the error of the coverage bound.

    \textbf{Part (a)}: concentration analysis of the characterization vectors $\rvv^{(k)}$ for benign clients ($1\le k \le K_b$).

    Let $\rvv_h^{(k)}$ be the $h$-th element of vector $\rvv^{(k)}$. 
    By definition, since $\rvv^{(k)}$ is sampled from a multinomial distribution, $\rvv_h^{(k)}$ denotes the success rate estimate of a Bernoulli distribution.
    We denote the event probabilities of the multinomial distribution $\gD$ as $\overline{\rvv}$.
    Therefore, the true success rate of the Bernoulli distribution at the $h$-th position is $\overline{\rvv}_h$.
    According to \reb{the binomial proportion confidence interval \cite{wallis2013binomial}}, we have:
    \begin{equation}
    \label{eq:bpci}
        \sP\left[ \left| \rvv^{(k)}_h - \overline{\rvv}_h \right| > \Phi^{-1}({1-\beta/2HK_b}) \dfrac{\sqrt{n_{ks} n_{kf}}}{n_k \sqrt{n_k}} \right] \le \beta/HK_b,
    \end{equation}
    where $\beta/HK_b$ is the probability confidence, $\Phi^{-1}(\cdot)$ denotes the inverse of the CDF of the standard normal distribution, and $n_{ks}$ and $n_{kf}:=n_k-n_{ks}$ are the number of success and failures in $n_k$ Bernoulli trials, respectively.
    Applying the inequality $n_{ks}n_{kf} \le n_k^2/4$ in \Cref{eq:bpci}, the following holds:
    \begin{equation}
        \sP\left[ \left| \rvv^{(k)}_h - \overline{\rvv}_h \right| > \dfrac{\Phi^{-1}({1-\beta/2HK_b})}{2 \sqrt{n_k}} \right] \le \beta/HK_b.
    \end{equation}
    Applying the union bound for $H$ elements in vector $\rvv^{(k)}$ and $K_b$ characterization vectors of benign clients, the following holds with probability $1-\beta$:
    \begin{equation}
        \left| \rvv^{(k)}_h - \overline{\rvv}_h \right| \le \dfrac{\Phi^{-1}({1-\beta/2HK_b})}{2 \sqrt{\min_{k'\in[K_b]}n_{k'}}},~\forall k \in [K_b],~\forall h\in [H],
    \end{equation}
    from which we can derive the bound of difference for $\ell_1$ norm distance as:
    \begin{equation}
        \left\| \rvv^{(k)} - \overline{\rvv} \right\|_1 \le r(\beta):= \dfrac{H\Phi^{-1}({1-\beta/2HK_b})}{2 \sqrt{\min_{k'\in[K_b]}n_{k'}}},~\forall k \in [K_b],
    \end{equation}
    where $r(\beta)$ is the perturbation radius of random vector $\rvv$ given confidence level $1-\beta$.
    $\forall k_1,k_2 \in [K_b]$, the following holds with probability $1-\beta$ due to the triangular inequality:
    \begin{align}
        &\left\| \rvv^{(k_1)} - \rvv^{(k_2)} \right\|_1 \le \left\| \rvv^{(k_1)} - \overline{\rvv} \right\|_1 + \left\| \rvv^{(k_2)} - \overline{\rvv} \right\|_1 \le 2r(\beta).
    \end{align}
    Furthermore, due to the fact that $\|\rvv\|_p \le \|\rvv\|_1$ for any integer $p\ge1$, the following holds with probability $1-\beta$:
    \begin{align}
        &\left\| \rvv^{(k)} - \overline{\rvv} \right\|_p \le \left\| \rvv^{(k)} - \overline{\rvv} \right\|_1 \le r(\beta), \label{eq:benign_ineq} \\
        &\left\| \rvv^{(k_1)} - \rvv^{(k_2)} \right\|_p \le \left\| \rvv^{(k_1)} - \rvv^{(k_2)} \right\|_1 \le 2r(\beta). \label{eq:benign_ineq_2}
    \end{align}

    \textbf{Part (b)}: analysis of the algorithm of the identification of malicious clients.

    Let $N(k,n)$ be the set of the index of $n$ nearest clients to the $k$-th client based on the metrics of $\ell_p$ norm distance in the space of characterization vectors. Then the maliciousness scores $M(k)$ for the $k$-th client $(k \in [K])$ can be defined as:
    \begin{equation}
        M(k) := \dfrac{1}{K_b-1} \sum_{k' \in N(k,K_b-1)} \left\| \rvv^{(k)} - \rvv^{(k')} \right\|_p.
    \end{equation}
    Let $\gB$ be the set of the index of benign clients identified by \Cref{alg:rob_fl_confinf} by selecting the clients associated with the lowest $K_b$ maliciousness scores.
    We will consider the following cases separately: (1) $\gB$ contains exactly $K_b$ benign clients, and (2) $\gB$ contains at least one malicious client indexed by $m$.

    \textit{Case (1)}: $\gB$ $(|\gB|=K_b)$ contains exactly $K_b$ benign clients. We can derive as follows:
    \begin{align}
        \left\| \sum_{k=1}^{K_b} \dfrac{n_k}{N_b} \rvv^{(k)} - \overline{\rvv} \right\|_p &\le  \sum_{k=1}^{K_b} \dfrac{n_k}{N_b} \left\| \rvv^{(k)} - \overline{\rvv} \right\|_p && \text{[triangular inequality]} \\
        &\le \sum_{k=1}^{K_b} \dfrac{n_k}{N_b} r(\beta) && \text{[by \Cref{eq:benign_ineq}]} \\
        &= r(\beta),
    \end{align}
    where $N_b:=\sum_{k \in [K_b]} n_k$ is the total sample size of benign clients.

    \textit{Case (2)}: $\gB$ $(|\gB|=K_b)$ contains at least one malicious client indexed by $m$. Since we assume $K_m < K_b$, there are at most $K_b-1$ malicious clients in $\gB$. Therefore, there is at least $1$ benign client in $[K] \backslash \gB$ indexed by $b$.
    We can derive the lower bound of the maliciousness score for the $m$-th client $M(m)$ as:
    \begin{align}
        M(m) &= \dfrac{1}{K_b-1} \sum_{k' \in N(m,K_b-1)} \left\| \rvv^{(m)} - \rvv^{(k')} \right\|_p \\
        &\ge \dfrac{1}{K_b-1} \sum_{k' \in N(m,K_b-1), k' \in [K_b]} \left\| \rvv^{(m)} - \rvv^{(k')} \right\|_p.\label{ineq:relax}
    \end{align}
    Since there are at least $K_b-K_m$ benign clients in $\gB$ (there are at most $K_m$ malicious clients in $\gB$), there exists one client indexed by $b_b~(b_b \in \gB)$ such that:
    \begin{equation}
        \left\| \rvv^{(m)} - \rvv^{(b_b)} \right\|_p \le \dfrac{(K_b-1)M(m)}{K_b-K_m} \label{eq:fac1}
    \end{equation}
    
    We can derive the upper bound of the maliciousness score for the $b$-th benign client $M(b)$ as:
    \begin{align}
        M(b) &= \dfrac{1}{K_b-1} \sum_{k' \in N(b,K_b-1)} \left\| \rvv^{(b)} - \rvv^{(k')} \right\|_p \\
        &\le 2r(\beta) && \text{[by \Cref{eq:benign_ineq_2}]} \label{eq:fac2}
    \end{align}

    Since the $m$-th client is included in $\gB$ and identified as a benign client, while the $b$-th client is not in $\gB$, the following holds according to the procedure in \Cref{alg:rob_fl_confinf}:
    \begin{equation}
        M(b) \ge M(m),
    \end{equation}
    from which we can derive the following by combining \Cref{eq:fac1} and \Cref{eq:fac2}:
    \begin{align}
        \left\| \rvv^{(m)} - \rvv^{(b_b)} \right\|_p \le \dfrac{(K_b-1)2r(\beta)}{K_b-K_m}
    \end{align}
    Then, we can derive the upper bound of $\left\| \rvv^{(m)} - \overline{\rvv} \right\|_p,~ \forall m \in \gB ~ \text{and} ~ K_b < m \le K$ as follows:
    \begin{align}
        \left\| \rvv^{(m)} - \overline{\rvv} \right\|_p &\le \left\| \rvv^{(m)} - \rvv^{(b_b)} \right\|_p + \left\| \rvv^{(b_b)} - \overline{\rvv} \right\|_p \\
        &\le \dfrac{2(K_b-1)r(\beta)}{K_b-K_m} + r(\beta)
    \end{align}

    Finally, we can derive as follows:
    \begin{align}
         \left\| \sum_{k \in \gB} \dfrac{n_k}{N_\gB} \rvv^{(k)} - \overline{\rvv} \right\|_p &\le  \sum_{k \in \gB} \dfrac{n_k}{N_\gB} \left\| \rvv^{(k)} - \overline{\rvv} \right\|_p \\
         &\le \sum_{k \in \gB, k \in [K_b]} \dfrac{n_k}{N_\gB} \left\| \rvv^{(k)} - \overline{\rvv} \right\|_p + \sum_{k \in \gB, k \in [K]\backslash[K_b]} \dfrac{n_k}{N_\gB} \left\| \rvv^{(k)} - \overline{\rvv} \right\|_p \\
         &\le \sum_{k \in \gB, k \in [K_b]} \dfrac{n_k}{N_\gB} r(\beta) + \sum_{k \in \gB, k \in [K]\backslash[K_b]} \dfrac{n_k}{N_\gB} \left[  \dfrac{2(K_b-1)r(\beta)}{K_b-K_m} + r(\beta) \right] \\
         &\le r(\beta) + \sum_{k \in \gB, k \in [K]\backslash[K_b]} \dfrac{n_k}{N_\gB} \dfrac{2(K_b-1)r(\beta)}{K_b-K_m} \\
         &\le r(\beta) \left( 1 + \dfrac{N_m}{\min_{k' \in [K_b]}n_{k'}}\dfrac{2}{1-\tau} \right),
    \end{align}
    where $N_m:=\sum_{k\in[K]\backslash[K_b]}n_k$ is the total sample size of malicious clients, $N_\gB$ is the total sample size of clients in $\gB$, and $\tau:=\dfrac{K_m}{K_b}$ is the ratio of the number of malicious clients to the number of benign clients.

    Combining \textit{case (1)} and \textit{case (2)}, we can conclude that:
    \begin{align}
        \left\| \sum_{k \in \gB} \dfrac{n_k}{N_\gB} \rvv^{(k)} - \overline{\rvv} \right\|_p &\le \max \left\{1,  1 + \dfrac{N_m}{\min_{k' \in [K_b]}n_{k'}}\dfrac{2}{1-\tau} \right\} r(\beta) \\
        &= \left(  1 + \dfrac{N_m}{\min_{k' \in [K_b]}n_{k'}}\dfrac{2}{1-\tau} \right) r(\beta)
    \end{align}

    \textbf{Part (c)}: analysis of the error of the coverage bound. In this part, we attempt to translate the error of aggregated vectors induced by malicious clients to the error of the bound of marginal coverage.
    Let $F_1(q,\rvv):=\sum_{j=1}^H \mathbb{I}\left[ a_j < q \right] \rvv_j$, where $q \in [0,1]$ and $a_j$ is the $j$-th partition point used to construct the characterization vector $\rvv \in \Delta^H$.
    Let $F_2(q,\rvv):=\sum_{j=1}^H \mathbb{I}\left[ a_{j-1} < q \right] \rvv_j$.
    Then by definition, we know that $F_1(q_\alpha,\overline{\rvv}) \le \sP\left[ Y_{\text{test}} \in \hat{C}_\alpha(X_{\text{test}})\right]  \le F_2(q_\alpha,\overline{\rvv})$, where $q_\alpha$ is the true $(1-\alpha)$ quantile value of the non-conformity scores,  $\overline{\rvv}$ is the event probability of the multinormial distribution $\gD$, and $\hat{C}_\alpha(X_{\text{test}})$ is the conformal prediction set of input $X_{\text{test}}$ using the true benign calibrated conformity score $q_\alpha$ and statistics of score distribution $\overline{\rvv}$.

    Let $\hat{q}_\alpha$ be the quantile estimate during calibration.
    FCP \citep{lu2023federated} proves that if the rank of quantile estimate $\hat{q}_\alpha$ is between $(1-\alpha-\epsilon)(N+K)$ and $(1-\alpha+\epsilon)(N+K)$, then we have:
    \begin{equation}
        F_1(\hat{q}_\alpha,\overline{\rvv}) \ge 1-\alpha -\dfrac{\epsilon N_\gB + 1}{N_\gB+K_b},~~ F_2(\hat{q}_\alpha,\overline{\rvv})\le 1-\alpha + \epsilon +\dfrac{K_b}{N_\gB+K_b}.
        \label{eq:term1}
    \end{equation}

    Now we start deriving the error of $F_1(\cdot,\cdot)$ induced by the malicious clients. 
    Let $\hat{\rvv}:=\sum_{k \in \gB} \dfrac{n_k}{N} \rvv^{(k)}$ be the estimated mean of characterization vector.
    Based on the results in part (b), we can derive as follows:
    \begin{align}
        \left| F_1(\hat{q}_\alpha,\overline{\rvv}) - F_1(\hat{q}_\alpha,\hat{\rvv}) \right| &= \left| \sum_{j=1}^H \mathbb{I}\left[ a_j < \hat{q}_\alpha \right] \overline{\rvv}_j - \sum_{j=1}^H \mathbb{I}\left[ a_j < \hat{q}_\alpha \right] \hat{\rvv}_j  \right| \\
        &\le \sum_{j=1}^H \mathbb{I}\left[ a_j < \hat{q}_\alpha \right] \left| \overline{\rvv}_j -  \hat{\rvv}_j \right| \\
        &\le \left\| \overline{\rvv} - \hat{\rvv} \right\|_1 \\
        &\le \left( 1+\dfrac{N_m}{\min_{k' \in [K_b] n_{k'}}}\dfrac{2}{1-\tau} \right) r(\beta) \label{eq:term2}
    \end{align}
    From triangular inequalities, we have:
    \begin{equation}
         F_1(\hat{q}_\alpha,\overline{\rvv}) - \left| F_1(\hat{q}_\alpha,\overline{\rvv}) - F_1(\hat{q}_\alpha,\hat{\rvv}) \right| \le F_1(\hat{q}_\alpha,\hat{\rvv}) \le \sP\left[ Y_{\text{test}} \in \hat{C}_\alpha(X_{\text{test}})\right].
    \end{equation}
    Similarly, we can derive that $\left| F_2(\hat{q}_\alpha,\overline{\rvv}) - F_2(\hat{q}_\alpha,\hat{\rvv}) \right| \le \left( 1+\dfrac{N_m}{\min_{k' \in [K_b] n_{k'}}}\dfrac{2}{1-\tau} \right) r(\beta)$ and have:
    \begin{equation}
         F_2(\hat{q}_\alpha,\overline{\rvv}) + \left| F_2(\hat{q}_\alpha,\overline{\rvv}) - F_2(\hat{q}_\alpha,\hat{\rvv}) \right| \ge F_2(\hat{q}_\alpha,\hat{\rvv}) \ge \sP\left[ Y_{\text{test}} \in \hat{C}_\alpha(X_{\text{test}})\right].
    \end{equation}
    Plugging in the terms in \Cref{eq:term1,eq:term2} and leveraging the fact $N_\gB \ge n_b$, we finally conclude that the following holds with probability $1-\beta$:
    \begin{equation}
    \begin{aligned}
        & \sP\left[ Y_{\text{test}} \in \hat{C}_\alpha(X_{\text{test}})\right] \ge 1-\alpha -\dfrac{\epsilon n_b + 1}{n_b+K_b} - \dfrac{H\Phi^{-1}({1-\beta/2HK_b})}{2 \sqrt{n_b}} \left( 1+\dfrac{N_m}{n_b}\dfrac{2}{1-\tau} \right), \\
        & \sP\left[ Y_{\text{test}} \in \hat{C}_\alpha(X_{\text{test}})\right] \le 1-\alpha + \epsilon + \dfrac{K_b}{n_b+K_b} + \dfrac{H\Phi^{-1}({1-\beta/2HK_b})}{2 \sqrt{n_b}} \left( 1+\dfrac{N_m}{n_b}\dfrac{2}{1-\tau} \right).
    \end{aligned}
    \end{equation}
    where $\tau = K_m / K_b$ is the ratio of the number of malicious clients and the number of benign clients, $N_m:=\sum_{k \in [K]\backslash[K_b]} n_k$ is the total sample size of malicious clients, $n_b:=\min_{k'\in[K_b]}n_{k'}$ is the minimal sample size of benign clients, and $\Phi^{-1}(\cdot)$ denotes the inverse of the cumulative distribution function (CDF) of standard normal distribution.

\end{proof}

Next, we start proving \Cref{thm1:improve}.
\begin{theorem}[Restatement of \Cref{thm1:improve}]
Consider FCP setting with $K_b$ benign clients and $K_m$ malicious clients.
    The $k$-th client reports the characterization vector $\rvv^{(k)}$ and local sample size $n_k$ to the server. 
    Assume that the benign characterization vector $\rvv^{(k)}$ is sampled from multinomial distribution $\gD_k$ with the event probability $\overline{\rvv}^{(k)}$ for the $k$-th client ($k \in [K_b]$).
    We use $\sigma$ to quantify the heterogeneity of benign vectors as  $\sigma = \max_{k_1,k_2 \in [K_b]}\| \overline{\rvv}^{(k_1)} - \overline{\rvv}^{(k_2)} \|_1$.
    Let $\epsilon$ be the data sketching error as \Cref{eq:marginal_eps}.
    Under the assumption that $K_m < K_b$, the following holds for test instance $(X_{\text{t}},Y_{\text{t}})$ with probability $1-\beta$:
    \begin{equation}
    \small
    \begin{aligned}
        & \sP\hspace{-0.2em}\left[ Y_{\text{t}} \hspace{-0.2em} \in \hspace{-0.2em} \hat{C}_\alpha(X_{\text{t}})\right] \hspace{-0.1em} \ge \hspace{-0.1em}  1 - \hspace{-0.1em} \alpha \hspace{-0.1em} - \hspace{-0.1em} P_{\text{byz}} \hspace{-0.1em} -  \hspace{-0.1em} \dfrac{N_m\sigma}{n_b(1-\tau)} \hspace{-0.1em} - \hspace{-0.1em} \dfrac{\epsilon n_b + 1}{n_b+K_b} \\
        & 
       \sP\hspace{-0.2em}\left[ Y_{\text{t}} \hspace{-0.2em} \in \hspace{-0.2em} \hat{C}_\alpha(X_{\text{t}})\right] \hspace{-0.1em} \le \hspace{-0.1em} 1 \hspace{-0.1em} - \hspace{-0.1em} \alpha \hspace{-0.1em} + \hspace{-0.1em} P_{\text{byz}} \hspace{-0.1em} + \hspace{-0.1em} \dfrac{N_m\sigma}{n_b(1-\tau)} \hspace{-0.1em} + \hspace{-0.1em} \dfrac{\epsilon n_b + (\epsilon+1) K_b}{n_b+K_b} \\
        & \text{where} \quad P_{\text{byz}} = \dfrac{H\Phi^{-1}({1-\beta/2HK_b})}{2 \sqrt{n_b}} \left( 1+\dfrac{N_m}{n_b}\dfrac{2}{1-\tau} \right)
    \end{aligned}
    \end{equation}
    where $\tau = K_m / K_b$ is the ratio of the number of malicious clients and the number of benign clients, \reb{$N_m:=\sum_{k \in [K]\backslash[K_b]} n_k$ is the total sample size of malicious clients}, $n_b:=\min_{k'\in[K_b]}n_{k'}$ is the minimal sample size of benign clients, and $\Phi^{-1}(\cdot)$ denotes the inverse of the cumulative distribution function (CDF) of standard normal distribution.
\end{theorem}
\begin{proof}
    The general structure of the proof follows the proof of \Cref{thm1}. We will omit similar derivation and refer to the proof of \Cref{thm1} for details.
    The proof consists of 3 parts: (a) concentration analysis of the characterization vectors $\rvv^{(k)}$ for benign clients ($1\le k \le K_b$), (b) analysis of the algorithm of the identification of malicious clients, and (c) analysis of the error of the coverage bound.

    \textbf{Part (a)}: concentration analysis of the characterization vectors $\rvv^{(k)}$ for benign clients ($1\le k \le K_b$).

    Let $\overline{\rvv}^{(k)}$ be the event probability of the multinormial distribution $\gD^{(k)}$ for $k \in [K_b]$.
    By applying binomial proportion approximate normal confidence interval and union bound as in Part (a) in the proof of \Cref{thm1}, with confidence $1-\beta$, we have:
    \begin{equation}
    \label{eq:concen_noniid}
        \left\| \rvv^{(k)} - \overline{\rvv}^{(k)} \right\|_1 \le r(\beta):= \dfrac{H\Phi^{-1}({1-\beta/2HK_b})}{2 \sqrt{\min_{k'\in[K_b]}n_{k'}}},~\forall k \in [K_b],
    \end{equation}
    where $r(\beta)$ is the perturbation radius of random vector $\rvv^{(k)}$ given confidence level $1-\beta$.
    $\forall k_1,k_2 \in [K_b]$, we can upper bound the $\ell_p$ norm distance between $\rvv^{(k_1)}$ and $\rvv^{(k_2)}$ as:
    \begin{align}
        \left\| \rvv^{(k_1)} - \rvv^{(k_2)} \right\|_p &\le \left\| \rvv^{(k_1)} -  \overline{\rvv}^{(k_1)} \right\|_p + \left\| \overline{\rvv}^{(k_1)} - \overline{\rvv}^{(k_2)} \right\|_p + \left\| \rvv^{(k_2)} - \overline{\rvv}^{(k_1)} \right\|_p \\
        &\le \left\| \rvv^{(k_1)} -  \overline{\rvv}^{(k_1)} \right\|_1 + \left\| \overline{\rvv}^{(k_1)} - \overline{\rvv}^{(k_2)} \right\|_p + \left\| \rvv^{(k_2)} - \overline{\rvv}^{(k_1)} \right\|_1 \\
        &\le 2r(\beta) + \sigma \label{ineq:tmp},
    \end{align}
    where \Cref{ineq:tmp} holds by \Cref{eq:concen_noniid}.

    \textbf{Part (b)}: analysis of the algorithm of the identification of malicious clients.

    Let $N(k,n)$ be the set of the index of $n$ nearest clients to the $k$-th client based on the metrics of $\ell_p$ norm distance in the space of characterization vectors. Then the maliciousness scores $M(k)$ for the $k$-th client $(k \in [K])$ can be defined as:
    \begin{equation}
        M(k) := \dfrac{1}{K_b-1} \sum_{k' \in N(k,K_b-1)} \left\| \rvv^{(k)} - \rvv^{(k')} \right\|_p.
    \end{equation}
    Let $\gB$ be the set of the index of benign clients identified by \Cref{alg:rob_fl_confinf} by selecting the clients associated with the lowest $K_b$ maliciousness scores.
    We will consider the following cases separately: (1) $\gB$ contains exactly $K_b$ benign clients, and (2) $\gB$ contains at least one malicious client indexed by $m$.

    \textit{Case (1)}: $\gB$ $(|\gB|=K_b)$ contains exactly $K_b$ benign clients. We can derive as follows:
    \begin{align}
        \left\| \sum_{k=1}^{K_b} \dfrac{n_k}{N_b} \rvv^{(k)} - \sum_{k=1}^{K_b} \dfrac{n_k}{N_b} \overline{\rvv}^{(k)} \right\|_p &\le  \sum_{k=1}^{K_b} \dfrac{n_k}{N_b} \left\| \rvv^{(k)} - \overline{\rvv}^{(k)} \right\|_p  \\
        &\le \sum_{k=1}^{K_b} \dfrac{n_k}{N_b} r(\beta)  \\
        &= r(\beta),
    \end{align}
    where $N_b:=\sum_{k \in [K_b]} n_k$ is the total sample size of benign clients.

    \textit{Case (2)}: $\gB$ $(|\gB|=K_b)$ contains at least one malicious client indexed by $m$. Since we assume $K_m < K_b$, there are at most $K_b-1$ malicious clients in $\gB$. Therefore, there is at least $1$ benign client in $[K] \backslash \gB$ indexed by $b$. From the fact that $M(m) \le M(b)$ and expanding the definitions the maliciousness score as Part (b) in the proof of \Cref{thm1}, we get that $\exists b_b \in \gB, b_b \in [K_b]$:
    \begin{equation}
        \left\| \rvv^{(m)} - \rvv^{(b_b)} \right\|_p \le \dfrac{(K_b-1)(2r(\beta)+\sigma)}{K_b-K_m}
    \end{equation}

    Therefore, we can upper bound the distance between the estimated global event probability vector $\sum_{k \in \gB} \dfrac{n_k}{N_\gB} \rvv^{(k)}$ and the benign global event probability vector $\sum_{k \in [K_b]} \dfrac{n_k}{N_b} \overline{\rvv}^{(k)}$.

    We first show that $\forall k \in [K_b]$, we have:
    \begin{align}
        \left\| \rvv^{(k)} - \sum_{k \in [K_b]} \dfrac{n_k}{N_b} \overline{\rvv}^{(k)} \right\|_p &\le \sum_{k \in [K_b]} \dfrac{n_k}{N_b} \left\| \rvv^{(k)} - \overline{\rvv}^{(k)} \right\|_p \le r(\beta).
    \end{align}
    Then, we can derive as follows:
    \begin{align}
         & \left\| \sum_{k \in \gB} \dfrac{n_k}{N_\gB} \rvv^{(k)} - \sum_{k \in [K_b]} \dfrac{n_k}{N_b} \overline{\rvv}^{(k)} \right\|_p \\
         \le& \sum_{k \in \gB, k \in [K_b]} \dfrac{n_k}{N_\gB} \left\| \rvv^{(k)} - \sum_{k \in [K_b]} \dfrac{n_k}{N_b} \overline{\rvv}^{(k)} \right\|_p + \sum_{k \in \gB, k \in [K]\backslash[K_b]} \dfrac{n_k}{N_\gB} \left\| \rvv^{(k)} - \sum_{k \in [K_b]} \dfrac{n_k}{N_b} \overline{\rvv}^{(k)} \right\|_p \\
         \le& \sum_{k \in \gB, k \in [K_b]} \dfrac{n_k}{N_\gB} r(\beta) + \sum_{k \in \gB, k \in [K]\backslash[K_b]} \dfrac{n_k}{N_\gB} \left[\left\| \rvv^{(k)} - \rvv^{(b_b)} \right\|_p + \left\| \rvv^{(b_b)} - \sum_{k \in [K_b]} \dfrac{n_k}{N_b} \overline{\rvv}^{(k)} \right\|_p \right] \\
         \le& \sum_{k \in \gB, k \in [K_b]} \dfrac{n_k}{N_\gB} r(\beta) + \sum_{k \in \gB, k \in [K]\backslash[K_b]} \dfrac{n_k}{N_\gB} \left[  \dfrac{(K_b-1)(2r(\beta)+\sigma)}{K_b-K_m} + r(\beta) \right] \\
         \le& r(\beta) + \sum_{k \in \gB, k \in [K]\backslash[K_b]} \dfrac{n_k}{N_\gB} \dfrac{(K_b-1)(2r(\beta)+\sigma)}{K_b-K_m} \\
         \le& r(\beta) \left( 1 + \dfrac{N_m}{n_b}\dfrac{2}{1-\tau} \right) + \dfrac{N_m}{n_b}\dfrac{\sigma}{1-\tau} \label{ineq:err}.
    \end{align}

\textbf{Part (c)}: analysis of the error of the coverage bound.

Let $F_1(q,\rvv):=\sum_{j=1}^H \mathbb{I}\left[ a_j < q \right] \rvv_j$, where $q \in [0,1]$ and $a_j$ is the $j$-th partition point used to construct the characterization vector $\rvv \in \Delta^H$.
Let $F_2(q,\rvv):=\sum_{j=1}^H \mathbb{I}\left[ a_{j-1} < q \right] \rvv_j$.
This part follows the same procedure to translate the error of aggregated vectors induced by malicious clients to the error of the bound of marginal coverage. The only difference is that considering data heterogeneity, the error of aggregated vectors formulated in \Cref{ineq:err} needs additional correction by the client data disparity.
Therefore, by analyzing the connection between characterization vector and coverage similarly in Part (3) in the proof of \Cref{thm1}, we have:
\begin{align}
    \left| F_1(\hat{q}_\alpha,\overline{\rvv}) - F_1(\hat{q}_\alpha,\hat{\rvv}) \right| \le r(\beta) \left( 1 + \dfrac{N_m}{n_b}\dfrac{2}{1-\tau} \right) + \dfrac{N_m}{n_b}\dfrac{\sigma}{1-\tau}, \\
     \left| F_2(\hat{q}_\alpha,\overline{\rvv}) - F_2(\hat{q}_\alpha,\hat{\rvv}) \right| \le r(\beta) \left( 1 + \dfrac{N_m}{n_b}\dfrac{2}{1-\tau} \right) + \dfrac{N_m}{n_b}\dfrac{\sigma}{1-\tau},
\end{align}
where $\overline{\rvv}:=\sum_{k \in [K_b]} \dfrac{n_k}{N_b} \overline{\rvv}^{(k)}$ and $\hat{\rvv}:=\sum_{k \in \gB} \dfrac{n_k}{N_\gB} \rvv^{(k)}$. On the other hand, from triangular inequalities, we have:
\begin{align}
    F_1(\hat{q}_\alpha,\overline{\rvv}) - \left| F_1(\hat{q}_\alpha,\overline{\rvv}) - F_1(\hat{q}_\alpha,\hat{\rvv}) \right| \le F_1(\hat{q}_\alpha,\hat{\rvv}) \le \sP\left[ Y_{\text{test}} \in \hat{C}_\alpha(X_{\text{test}})\right],\\
     F_2(\hat{q}_\alpha,\overline{\rvv}) + \left| F_2(\hat{q}_\alpha,\overline{\rvv}) - F_2(\hat{q}_\alpha,\hat{\rvv}) \right| \ge F_2(\hat{q}_\alpha,\hat{\rvv}) \ge \sP\left[ Y_{\text{test}} \in \hat{C}_\alpha(X_{\text{test}})\right].
\end{align}
    Plugging in the terms, we finally conclude that the following holds with probability $1-\beta$:
    \begin{equation}
    \small
    \begin{aligned}
        & \sP\left[ Y_{\text{test}} \in \hat{C}_\alpha(X_{\text{test}})\right] \ge 1-\alpha -\dfrac{\epsilon n_b + 1}{n_b+K_b} - \dfrac{H\Phi^{-1}({1-\beta/2HK_b})}{2 \sqrt{n_b}} \left( 1+\dfrac{N_m}{n_b}\dfrac{2}{1-\tau} \right) - \dfrac{N_m}{n_b}\dfrac{\sigma}{1-\tau}, \\
        & \sP\left[ Y_{\text{test}} \in \hat{C}_\alpha(X_{\text{test}})\right] \le 1-\alpha + \epsilon + \dfrac{K_b}{n_b+K_b} + \dfrac{H\Phi^{-1}({1-\beta/2HK_b})}{2 \sqrt{n_b}} \left( 1+\dfrac{N_m}{n_b}\dfrac{2}{1-\tau} \right) + \dfrac{N_m}{n_b}\dfrac{\sigma}{1-\tau}.
    \end{aligned}
    \end{equation}
    where $\tau = K_m / K_b$ is the ratio of the number of malicious clients and the number of benign clients, $N_m:=\sum_{k \in [K]\backslash[K_b]}n_k$ is the total sample size of malicious clients, $n_b:=\min_{k'\in[K_b]}n_{k'}$ is the minimal sample size of benign clients, and $\Phi^{-1}(\cdot)$ denotes the inverse of CDF of the standard normal distribution.

\end{proof}

\subsection{Proof of \Cref{thm2}}
\label{app:proof_thm}
\begin{theorem}[Restatement of \Cref{thm2}]
    Assume $\rvv^{(k)}~(k \in [K_b])$ are IID sampled from Gaussian $\gN(\mu,\Sigma)$ with mean $\mu \in \sR^H$ and positive definite covariance matrix $\Sigma \in \sR^{H \times H}$. Let $d:=\min_{k \in [K]\backslash [K_b]} \| \rvv^{(k)} - \mu \|_2$. 
    Suppose that we use $\ell_2$ norm to measure vector distance and leverage the malicious client number estimator with an initial guess of a number of benign clients $\tilde{K}_b$ such that $K_m<\tilde{K}_b\le K_b$.
    Then we have:
   \begin{equation}
         \sP\left[ \hat{K}_m = K_m \right] \ge 1 - \dfrac{(3\tilde{K}_b-K_m-2)^2\text{Tr}(\Sigma)}{(\tilde{K}_b-K_m)^2d^2} - \dfrac{2(K+K_b) \text{Tr}(\Sigma)\sigma^2_{\text{max}}(\Sigma^{-1/2})}{\sigma^2_{\text{min}}(\Sigma^{-1/2}) d^2},
    \end{equation}
    where $\sigma_{\text{max}}(\Sigma^{-1/2})$, $\sigma_{\text{min}}(\Sigma^{-1/2})$ denote the maximal and minimal eigenvalue of matrix $\Sigma^{-1/2}$, and $\text{Tr}(\Sigma)$ denotes the trace of matrix $\Sigma$.
\end{theorem}
\begin{proof}
    From the concentration inequality of multivariate Gaussian distribution \citep{vershynin2018high}, the following holds for $\rvv^{(k)} \sim \gN(\mu,\Sigma)$:
    \begin{equation}
        \sP\left[ \| \rvv^{(k)} - \mu \|_2 \le \sqrt{\dfrac{1}{\delta} \text{Tr}(\Sigma)} \right] \ge 1 - \delta.
    \end{equation}
    Applying union bound for all benign clients $k \in [K_b]$, the following concentration bound holds:
    \begin{equation}
        \sP\left[ \| \rvv^{(k)} - \mu \|_2 \le \sqrt{\dfrac{K_b}{\delta} \text{Tr}(\Sigma)},~\forall k \in [K_b] \right] \ge 1 - \delta, 
    \end{equation}
    Let the perturbation radius $r:=\dfrac{\tilde{K}_b-K_m}{3\tilde{K}_b-K_m-2}d$. Then we can derive that:
    \begin{equation}
        \sP\left[ \| \rvv^{(k)} - \mu \|_2 \le r:=\dfrac{\tilde{K}_b-K_m}{3\tilde{K}_b-K_m-2}d,~\forall k \in [K_b] \right] \ge 1 - \dfrac{(3\tilde{K}_b-K_m-2)^2\text{Tr}(\Sigma)}{(\tilde{K}_b-K_m)^2d^2}:= 1- \delta.  \label{eq:conc}
    \end{equation}
    The following discussion is based on the fact that $\| \rvv^{(k)} - \mu \|_2 \le r:=\dfrac{\tilde{K}_b-K_m}{3\tilde{K}_b-K_m-2}d,~\forall k \in [K_b]$, and the confidence $1-\delta$ will be incorporated in the final statement.
    Let $N(k,n)$ be the index set of $n$ nearest neighbors of client $k$ in the characterization vector space with the metric of $\ell_2$ norm distance.
    We consider the maliciousness score $M(b)$ of any benign client $b \in [K_b]$:
    \begin{align}
         M(b) &= \dfrac{1}{\tilde{K}_b-1} \sum_{k' \in N(b,\tilde{K}_b-1)} \left\| \rvv^{(b)} - \rvv^{(k')} \right\|_2 \\
         &\le \max_{k' \in [K_b]} \left\| \rvv^{(b)} - \rvv^{(k')} \right\|_2 \label{ineq:tmp1} \\
         &\le \max_{k' \in [K_b]} \left\{ \left\| \rvv^{(b)} - \mu \right\|_2 + \left\| \mu - \rvv^{(k')} \right\|_2 \right\} \\
         &\le \dfrac{2(\tilde{K}_b-K_m)}{3\tilde{K}_b-K_m-2}d. \label{ineq:temp2}
    \end{align}
    \Cref{ineq:tmp1} holds since the average of distances to $\tilde{K}_b-1$ nearest vectors is upper bounded by the average of distances to arbitrary $\tilde{K}_b-1$ benign clients, which is upper bounded by the maximal distance to benign clients. \Cref{ineq:temp2} holds by plugging in the results in \Cref{eq:conc}.

    We consider the maliciousness score $M(m)$ of any malicious client $m \in [K]\backslash[K_b]$:
    \begin{align}
         M(m) &= \dfrac{1}{\tilde{K}_b-1} \sum_{k' \in N(m,\tilde{K}_b-1)} \left\| \rvv^{(m)} - \rvv^{(k')} \right\|_2 \\
          &\ge \dfrac{1}{\tilde{K}_b-1} \sum_{k' \in N(m,\tilde{K}_b-1), k' \in [K_b]} \left\| \rvv^{(m)} - \rvv^{(k')} \right\|_2 \\
          &\ge \dfrac{1}{\tilde{K}_b-1} \sum_{k' \in N(m,\tilde{K}_b-1), k' \in [K_b]} \left[ \left\| \rvv^{(m)} - \mu \right\|_2 - \left\| \mu - \rvv^{(k')} \right\|_2 \right]\\
         &\ge \dfrac{1}{\tilde{K}_b-1} (\tilde{K}_b-K_m) \left(d-\dfrac{\tilde{K}_b-K_m}{3\tilde{K}_b-K_m-2}d \right) \label{ineq:exp1} \\
         &\ge \dfrac{2(\tilde{K}_b-K_m)}{3\tilde{K}_b-K_m-2}d \label{ineq:facc1}.
    \end{align}
    \Cref{ineq:exp1} holds since $d:=\min_{k \in [K]\backslash [K_b]} \| \rvv^{(k)} - \mu \|_2$ by definition.
    Therefore, from \Cref{ineq:temp2} and \Cref{ineq:facc1}, we can conclude that with probability $1-\delta$, $M(m) \ge M(b),~ \forall b \in [K_b], m \in [K]\backslash [K_b]$, which implies that $\forall k \in [K_b], I(k) \in [K_b]$ and $\forall k \in [K]-[K_b], I(k) \in [K]\backslash [K_b]$.

    Recall that the estimate of the number of benign clients $\hat{K}_b$ is given by:
    \begin{equation}
        \hat{K}_b = \argmax_{z \in [K]}\left[ \dfrac{1}{z} \sum_{k=1}^{z} \log p(\rvv^{(I(k))};\mu,\Sigma) - \dfrac{1}{K-z} \sum_{k=z+1}^K \log p(\rvv^{(I(k))};\mu,\Sigma) \right].
        \label{eq:def_m}
    \end{equation}

    For ease of notation, let $T(z):= \dfrac{1}{z} \sum_{k=1}^{z} \log p(\rvv^{(I(k))};\mu,\Sigma) - \dfrac{1}{K-z} \sum_{k=z+1}^K \log p(\rvv^{(I(k))};\mu,\Sigma)$ for $z \in [K]$ and $d_k:=\rvv^{(I(k))}-\mu$ for $k \in [K]$.
    Then we can upper bound the probability of an underestimate of the number of malicious clients $\sP\left[ \hat{K}_b < K_b \right]$ as follows:
    \begin{align}
        & \sP\left[ \hat{K}_b < K_b \right]\\
        =& \sP\left[T(\hat{K}_b) > T(K_b)\right]\\
         \le&    \sP\left[ \dfrac{-(K_b-\hat{K}_b)}{K_b\hat{K}_b}\sum_{k=1}^{\hat{K}_b} \log p(\rvv^{(I(k)})) + \dfrac{K-\hat{K}_b+K_b}{K_b(K-\hat{K}_b)} \sum_{k=\hat{K}_b+1}^{K_b} \log p(\rvv^{(I(k)}))  \right. \nonumber \\
        &< \left. \dfrac{K_b-\hat{K}_b}{(K-K_b)(K-\hat{K}_b)} \sum_{k=K_b+1}^K \log p(\rvv^{(I(k))})\right] \label{ineq:tbexp1} \\
        \le& \sP\left[\dfrac{K-\hat{K}_b+K_b}{K_b(K-\hat{K}_b)} \sum_{k=\hat{K}_b+1}^{K_b}-d_k^T \Sigma^{-1} d_k < \dfrac{K_b-\hat{K}_b}{(K-K_b)(K-\hat{K}_b)} \sum_{k=K_b+1}^K -d_k^T \Sigma^{-1} d_k \right] \label{ineq:tbexp2} \\
        \le& \sP\left[ \dfrac{K_b-\hat{K}_b}{(K-K_b)} \sum_{k=K_b+1}^K \| d_k^T \Sigma^{-1/2} \|_2^2 < \dfrac{K-\hat{K}_b+K_b}{K_b}\sum_{k=\hat{K}_b+1}^{K_b} \| d_k^T \Sigma^{-1/2} \|_2^2  \right] \\
        \le& \sP\left[ \dfrac{K_b-\hat{K}_b}{(K-K_b)} \sum_{k=K_b+1}^K \sigma^2_{\text{min}}(\Sigma^{-1/2}) \| d_k^T \|_2^2 < \dfrac{K-\hat{K}_b+K_b}{K_b} \sum_{k=\hat{K}_b+1}^{K_b}  \sigma^2_{\text{max}}(\Sigma^{-1/2}) \| d_k^T \|_2^2  \right] \label{ineq:tbexp3} \\
        \le& \sP\left[ \sigma^2_{\text{min}}(\Sigma^{-1/2}) d^2 < \dfrac{K-\hat{K}_b+K_b}{K_b} \sigma^2_{\text{max}}(\Sigma^{-1/2}) \max_{k \in [K_b]} \| d_k^T \|_2^2  \right] \\
        \le& \sP\left[ \max_{k \in [K_b]} \| d_k^T \|_2 > \sqrt{\dfrac{K_b}{K+K_b}} \dfrac{\sigma_{\text{min}}(\Sigma^{-1/2})d}{\sigma_{\text{max}}(\Sigma^{-1/2})} \right] \\
        \le& \dfrac{(K+K_b) \text{Tr}(\Sigma)\sigma^2_{\text{max}}(\Sigma^{-1/2})}{\sigma^2_{\text{min}}(\Sigma^{-1/2}) d^2}
    \end{align}

    \Cref{ineq:tbexp1} holds by plugging in the definitions in \Cref{eq:def_m} and rearranging the terms. \Cref{ineq:tbexp2} holds by dropping the positive term $ \dfrac{-(K_b-\hat{K}_b)}{K_b\hat{K}_b}\sum_{k=1}^{\hat{K}_b} \log p(\rvv^{(I(k))})$ and rearranging log-likelihood terms of multivariate Gaussian with $d_k$. \Cref{ineq:tbexp3} holds by leveraging the fact that $\sigma_{\text{min}}(\Sigma^{-1/2}) \| d^T_k \|_2 \le \| d^T_k \Sigma^{-1/2} \|_2 \le \sigma_{\text{max}}(\Sigma^{-1/2}) \| d^T_k \|_2$.

    Similarly, we can upper bound the probability of overestimation of the number of malicious clients $\sP\left[ \hat{K}_b > K_b \right]$ as:
    \begin{equation}
         \sP\left[ \hat{K}_b > K_b \right] \le \dfrac{(K+K_b) \text{Tr}(\Sigma)\sigma^2_{\text{max}}(\Sigma^{-1/2})}{\sigma^2_{\text{min}}(\Sigma^{-1/2}) d^2}.
    \end{equation}

    We can finally conclude that:
    \begin{equation}
         \sP\left[ \hat{K}_b = K_b \right] \ge 1 - \dfrac{(3\tilde{K}_b-K_m-2)^2\text{Tr}(\Sigma)}{(\tilde{K}_b-K_m)^2d^2} - \dfrac{2(K+K_b) \text{Tr}(\Sigma)\sigma^2_{\text{max}}(\Sigma^{-1/2})}{\sigma^2_{\text{min}}(\Sigma^{-1/2}) d^2}.
    \end{equation}

\end{proof}

\reb{
\section{Improvements with DKW inequality}
\label{app:dkw}
\subsection{Improvement of \Cref{thm1} with DKW inequality}
\begin{theorem}[Improvement of \Cref{thm1}]
\label{thm1_imp}
    For $K$ clients including $K_b$ benign clients and $K_m:=K-K_b$ malicious clients, each client reports a characterization vector $\rvv^{(k)} \in \Delta^H$ $(k \in [K])$ and a quantity $n_k \in \mathbb{Z}^+$ $(k \in [K])$ to the server. 
    Suppose that the reported characterization vectors of benign clients are sampled from the same underlying multinomial distribution $\gD$, while those of malicious clients can be arbitrary. 
    Let $\epsilon$ be the estimation error of the data sketching by characterization vectors as illustrated in \Cref{eq:marginal_eps}.
    Under the assumption that $K_m < K_b$, the following holds with probability $1-\beta$:
   \begin{equation}
   \small
    \begin{aligned}
        & \sP\left[ Y_{\text{test}} \in \hat{C}_\alpha(X_{\text{test}})\right] \ge 1-\alpha -\dfrac{\epsilon n_b + 1}{n_b+K_b} - H\sqrt{\dfrac{\ln(2K_b/\beta)}{2n_b}} \left( 1+\dfrac{N_m}{n_b}\dfrac{2}{1-\tau} \right), \\
        & \sP\left[ Y_{\text{test}} \in \hat{C}_\alpha(X_{\text{test}})\right] \le 1-\alpha + \epsilon + \dfrac{K_b}{n_b+K_b} + H\sqrt{\dfrac{\ln(2K_b/\beta)}{2n_b}} \left( 1+\dfrac{N_m}{n_b}\dfrac{2}{1-\tau} \right),
    \end{aligned}
    \end{equation}
    where $\tau = K_m / K_b$ is the ratio of the number of malicious clients and the number of benign clients, $N_m:=\sum_{k \in [K]\backslash[K_b]} n_k$ is the total sample size of malicious clients, and $n_b:=\min_{k'\in[K_b]}n_{k'}$ is the minimal sample size of benign clients.
\end{theorem}
}
\reb{
\begin{proof}
The proof structure follows the proof of \Cref{thm1} and consists of 3 parts: (a) concentration analysis of the characterization vectors $\rvv^{(k)}$ for benign clients ($1\le k \le K_b$), (b) analysis of the algorithm of the identification of malicious clients, and (c) analysis of the error of the coverage bound. Part (b) and (c) are exactly the same as the proof \Cref{thm1} and the only difference lies in the use of a more advanced concentration bound in part (a), 
    which provides concentration analysis of the characterization vectors $\rvv^{(k)}$ for benign clients ($1\le k \le K_b$).
    Let $\rvv_h^{(k)}$ be the $h$-th element of vector $\rvv^{(k)}$. 
    According to the Dvoretzky–Kiefer–Wolfowitz (DKW) inequality, we have:
    \begin{equation}
    \label{eq:dkw}
        \sP\left[ \left| \rvv^{(k)}_h - \overline{\rvv}_h \right| > \beta \right] \le 2\exp\left\{ -2H\beta^2\right\},~\forall h \in \{1,2,..,H\}.
    \end{equation}
    Applying the union bound for $K_b$ characterization vectors of benign clients, the following holds with probability $1-\beta$:
    \begin{equation}
        \left| \rvv^{(k)}_h - \overline{\rvv}_h \right| \le \sqrt{\dfrac{\ln(2K_b/\beta)}{2n_b}},~\forall k \in [K_b],~\forall h\in [H],
    \end{equation}
    from which we can derive the bound of difference for $\ell_1$ norm distance as:
    \begin{equation}
        \left\| \rvv^{(k)} - \overline{\rvv} \right\|_1 \le r(\beta):= H\sqrt{\dfrac{\ln(2K_b/\beta)}{2n_b}},~\forall k \in [K_b],
    \end{equation}
    where $r(\beta)$ is the perturbation radius of random vector $\rvv$ given confidence level $1-\beta$.
    $\forall k_1,k_2 \in [K_b]$, the following holds with probability $1-\beta$ due to the triangular inequality:
    \begin{align}
        &\left\| \rvv^{(k_1)} - \rvv^{(k_2)} \right\|_1 \le \left\| \rvv^{(k_1)} - \overline{\rvv} \right\|_1 + \left\| \rvv^{(k_2)} - \overline{\rvv} \right\|_1 \le 2r(\beta).
    \end{align}
    Furthermore, due to the fact that $\|\rvv\|_p \le \|\rvv\|_1$ for any integer $p\ge1$, the following holds with probability $1-\beta$:
    \begin{align}
        &\left\| \rvv^{(k)} - \overline{\rvv} \right\|_p \le \left\| \rvv^{(k)} - \overline{\rvv} \right\|_1 \le r(\beta),  \\
        &\left\| \rvv^{(k_1)} - \rvv^{(k_2)} \right\|_p \le \left\| \rvv^{(k_1)} - \rvv^{(k_2)} \right\|_1 \le 2r(\beta).
    \end{align}    
    Then following the part (b) and (c) in the proof of \Cref{thm1}, we can finally conclude that:
    \begin{equation}
   \small
    \begin{aligned}
        & \sP\left[ Y_{\text{test}} \in \hat{C}_\alpha(X_{\text{test}})\right] \ge 1-\alpha -\dfrac{\epsilon n_b + 1}{n_b+K_b} - H\sqrt{\dfrac{\ln(2K_b/\beta)}{2n_b}} \left( 1+\dfrac{N_m}{n_b}\dfrac{2}{1-\tau} \right), \\
        & \sP\left[ Y_{\text{test}} \in \hat{C}_\alpha(X_{\text{test}})\right] \le 1-\alpha + \epsilon + \dfrac{K_b}{n_b+K_b} + H\sqrt{\dfrac{\ln(2K_b/\beta)}{2n_b}} \left( 1+\dfrac{N_m}{n_b}\dfrac{2}{1-\tau} \right),
    \end{aligned}
    \end{equation}
\end{proof}
}
\reb{
\subsection{Improvement of \Cref{thm1:improve} with DKW inequality}
\begin{theorem}[Improvement of \Cref{thm1:improve} with DKW inequality]
\label{coro_imp}
Under the same definitions and conditions in \Cref{thm1}, the following holds with probability $1-\beta$:
    \begin{equation}
    \small
    \begin{aligned}
        & \sP\left[ Y_{\text{test}} \in \hat{C}_\alpha(X_{\text{test}})\right] \ge 1-\alpha -\dfrac{\epsilon n_b + 1}{n_b+K_b} -  H\sqrt{\dfrac{\ln(2K_b/\beta)}{2n_b}} \left( 1+\dfrac{N_m}{n_b}\dfrac{2}{1-\tau} \right) - \dfrac{N_m}{n_b}\dfrac{\sigma}{1-\tau}, \\
        & \sP\left[ Y_{\text{test}} \in \hat{C}_\alpha(X_{\text{test}})\right] \le 1-\alpha + \epsilon  + \dfrac{K_b}{n_b+K_b} +  H\sqrt{\dfrac{\ln(2K_b/\beta)}{2n_b}} \left( 1+\dfrac{N_m}{n_b}\dfrac{2}{1-\tau} \right) + \dfrac{N_m}{n_b}\dfrac{\sigma}{1-\tau}.
    \end{aligned}
    \end{equation}
\end{theorem}
\reb{
\begin{proof}
    We conclude the proof by leveraging the concentration analysis in the proof of \Cref{thm1_imp} and part (b) and part (c) in the proof of \Cref{thm1:improve}.
\end{proof}
}
}

\reb{
\section{Analysis of Rob-FCP with an overestimated number of benign clients $K'_b$}
\label{app:over_num}
\begin{theorem}[\Cref{thm1} with an overestimated number of benign clients]
\label{thm1_over}
    For $K$ clients including $K_b$ benign clients and $K_m:=K-K_b$ malicious clients, each client reports a characterization vector $\rvv^{(k)} \in \Delta^H$ $(k \in [K])$ and a quantity $n_k \in \mathbb{Z}^+$ $(k \in [K])$ to the server. 
    Suppose that the reported characterization vectors of benign clients are sampled from the same underlying multinomial distribution $\gD$, while those of malicious clients can be arbitrary. 
    Let $\epsilon$ be the estimation error of the data sketching by characterization vectors as illustrated in \Cref{eq:marginal_eps}.
    \textbf{Let $K'_b>K_b$ be the overestimated number of benign clients.} We also assume benign clients and malicious clients have the same sample sizes.
    Under the assumption that $K_m < K_b$, the following holds with probability $1-\beta$:
   \begin{equation}
   \small
    \begin{aligned}
        & \sP\left[ Y_{\text{test}} \in \hat{C}_\alpha(X_{\text{test}})\right] \ge 1-\alpha -\dfrac{\epsilon n_b + 1}{n_b+K_b} - \left[ 1 - \dfrac{K_b}{K'_b} \left( 1 - \dfrac{H\Phi^{-1}({1-\beta/2HK_b})}{2 \sqrt{n_b}} \right) \right], \\
        & \sP\left[ Y_{\text{test}} \in \hat{C}_\alpha(X_{\text{test}})\right] \le 1-\alpha + \epsilon + \dfrac{K_b}{n_b+K_b} +\left[ 1 - \dfrac{K_b}{K'_b} \left( 1 - \dfrac{H\Phi^{-1}({1-\beta/2HK_b})}{2 \sqrt{n_b}} \right) \right],
    \end{aligned}
    \end{equation}
    where $\tau = K_m / K_b$ is the ratio of the number of malicious clients and the number of benign clients, \reb{$N_m:=\sum_{k \in [K]\backslash[K_b]} n_k$ is the total sample size of malicious clients}, $n_b:=\min_{k'\in[K_b]}n_{k'}$ is the minimal sample size of benign clients, and $\Phi^{-1}(\cdot)$ denotes the inverse of the cumulative distribution function (CDF) of standard normal distribution.
\end{theorem}
}
\reb{
\begin{proof}
The proof consists of 3 parts: (a) concentration analysis of the characterization vectors $\rvv^{(k)}$ for benign clients ($1\le k \le K_b$), (b) analysis of the algorithm of the identification of malicious clients, and (c) analysis of the error of the coverage bound.
Part (a) and (c) follow that of \Cref{thm1}, and thus, we provide the details of part (b) here.
 Let $N(k,n)$ be the set of the index of $n$ nearest clients to the $k$-th client based on the metrics of $\ell_p$ norm distance in the space of characterization vectors. Then the maliciousness scores $M(k)$ for the $k$-th client $(k \in [K])$ can be defined as:
    \begin{equation}
        M(k) := \dfrac{1}{K_b-1} \sum_{k' \in N(k,K_b-1)} \left\| \rvv^{(k)} - \rvv^{(k')} \right\|_p.
    \end{equation}
\reb{
    Let $\gB$ be the set of the index of benign clients identified by \Cref{alg:rob_fl_confinf} by selecting the clients associated with the lowest $K'_b$ maliciousness scores.
    We will consider the following cases separately: (1) $\gB$ contains exactly $K_b$ benign clients, and (2) $\gB$ contains at least one malicious client indexed by $m$.
    }
\reb{
    \textit{Case (1)}: $\gB$ $(|\gB|=K'_b)$ contains all $K_b$ benign clients. We can derive as follows:
    \begin{align}
         \left\| \sum_{k \in \gB} \dfrac{n_k}{N_\gB} \rvv^{(k)} - \overline{\rvv} \right\|_p &\le  \sum_{k \in \gB} \dfrac{n_k}{N_\gB} \left\| \rvv^{(k)} - \overline{\rvv} \right\|_p \\
         &\le \sum_{k \in \gB, k \in [K_b]} \dfrac{n_k}{N_\gB} \left\| \rvv^{(k)} - \overline{\rvv} \right\|_p + \sum_{k \in \gB, k \in [K]\backslash[K_b]} \dfrac{n_k}{N_\gB} \left\| \rvv^{(k)} - \overline{\rvv} \right\|_p \\
         &\le \sum_{k \in \gB, k \in [K_b]} \dfrac{n_k}{N_\gB} r(\beta) + \sum_{k \in \gB, k \in [K]\backslash[K_b]} \dfrac{n_k}{N_\gB} \times 1 \\
         &= \dfrac{K_b}{K'_b} r(\beta) + \left( 1 - \dfrac{K_b}{K'_b} \right) \\
         &= 1 - \dfrac{K_b}{K'_b}\left(1-r(\beta) \right)
    \end{align}
}
\reb{
    \textit{Case (2)}: $\gB$ $(|\gB|=K'_b)$ does not contain all benign clients, which implicates that for any malicious client $m \in \gB$,
    we can derive the lower bound of the maliciousness score for the $m$-th client $M(m)$ as:
    \begin{align}
        M(m) &= \dfrac{1}{K'_b-1} \sum_{k' \in N(m,K'_b-1)} \left\| \rvv^{(m)} - \rvv^{(k')} \right\|_p \\
        &\ge \dfrac{1}{K'_b-1} \sum_{k' \in N(m,K'_b-1), k' \in [K_b]} \left\| \rvv^{(m)} - \rvv^{(k')} \right\|_p.
    \end{align}
    Since there are at least $K'_b-K_m$ benign clients in $\gB$ (there are at most $K_m$ malicious clients in $\gB$), there exists one client indexed by $b_b~(b_b \in \gB)$ such that:
    \begin{equation}
        \left\| \rvv^{(m)} - \rvv^{(b_b)} \right\|_p \le \dfrac{(K'_b-1)M(m)}{K'_b-K_m}
    \end{equation}
    }
    We can derive the upper bound of the maliciousness score for the $b$-th benign client (one benign client not in $\gB$) $M(b)$ as:
    \begin{align}
        M(b) &= \dfrac{1}{K'_b-1} \sum_{k' \in N(b,K'_b-1)} \left\| \rvv^{(b)} - \rvv^{(k')} \right\|_p \\
        &\le \dfrac{K_b-1}{K'_b-1}2r(\beta)+ \dfrac{K_b-K'_b}{K'_b-1}
    \end{align}
    Since the $m$-th client is included in $\gB$ and identified as a benign client, while the $b$-th client is not in $\gB$, the following holds according to the procedure in \Cref{alg:rob_fl_confinf}:
    \begin{equation}
        M(b) \ge M(m),
    \end{equation}
    Then, we can derive the upper bound of $\left\| \rvv^{(m)} - \overline{\rvv} \right\|_p,~ \forall m \in \gB ~ \text{and} ~ K_b < m \le K$ as follows:
    \begin{align}
        \left\| \rvv^{(m)} - \overline{\rvv} \right\|_p &\le \left\| \rvv^{(m)} - \rvv^{(b_b)} \right\|_p + \left\| \rvv^{(b_b)} - \overline{\rvv} \right\|_p \\
        &\le \dfrac{(K_b-1)2r(\beta) + K_b-K'_b}{K'_b-K_m}
    \end{align}
    Finally, we can derive as follows:
    \begin{align}
         \left\| \sum_{k \in \gB} \dfrac{n_k}{N_\gB} \rvv^{(k)} - \overline{\rvv} \right\|_p &\le  \sum_{k \in \gB} \dfrac{n_k}{N_\gB} \left\| \rvv^{(k)} - \overline{\rvv} \right\|_p \\
         &\le \sum_{k \in \gB, k \in [K_b]} \dfrac{n_k}{N_\gB} \left\| \rvv^{(k)} - \overline{\rvv} \right\|_p + \sum_{k \in \gB, k \in [K]\backslash[K_b]} \dfrac{n_k}{N_\gB} \left\| \rvv^{(k)} - \overline{\rvv} \right\|_p \\
         &\le \sum_{k \in \gB, k \in [K_b]} \dfrac{n_k}{N_\gB} r(\beta) + \sum_{k \in \gB, k \in [K]\backslash[K_b]} \dfrac{n_k}{N_\gB} \dfrac{(K_b-1)2r(\beta) + K_b-K'_b}{K'_b-K_m} \\
         &\le \dfrac{K_b}{K'_b}r(\beta) + \dfrac{K'_b-K_b}{K'_b} \dfrac{(K_b-1)2r(\beta)+K_b-K'_b}{K'_b-K_m}
    \end{align}
    Combining \textit{case (1)} and \textit{case (2)}, we can conclude that:
    \begin{equation}
    \small
    \begin{aligned}
    \left\| \sum_{k \in \gB} \dfrac{n_k}{N_\gB} \rvv^{(k)} - \overline{\rvv} \right\|_p &\le \max \left\{1 - \dfrac{K_b}{K'_b}\left(1-r(\beta) \right),  \dfrac{K_b}{K'_b}r(\beta) + \dfrac{K'_b-K_b}{K'_b} \dfrac{(K_b-1)2r(\beta)+K_b-K'_b}{K'_b-K_m} \right\}  \\
        &= 1 - \dfrac{K_b}{K'_b}\left(1-r(\beta) \right)        
    \end{aligned}
    \end{equation}
    Finally, by applying the analysis of part (a) and (c) in the proof of \Cref{thm1}, we can conclude that:
    \begin{equation}
   \small
    \begin{aligned}
        & \sP\left[ Y_{\text{test}} \in \hat{C}_\alpha(X_{\text{test}})\right] \ge 1-\alpha -\dfrac{\epsilon n_b + 1}{n_b+K_b} - \left[ 1 - \dfrac{K_b}{K'_b} \left( 1 - \dfrac{H\Phi^{-1}({1-\beta/2HK_b})}{2 \sqrt{n_b}} \right) \right], \\
        & \sP\left[ Y_{\text{test}} \in \hat{C}_\alpha(X_{\text{test}})\right] \le 1-\alpha + \epsilon + \dfrac{K_b}{n_b+K_b} +\left[ 1 - \dfrac{K_b}{K'_b} \left( 1 - \dfrac{H\Phi^{-1}({1-\beta/2HK_b})}{2 \sqrt{n_b}} \right) \right],
    \end{aligned}
    \end{equation}
\end{proof}
}

\section{Algorithm of \name}
\label{app:alg}

We provide the complete pseudocodes of malicious client identification in \name in \Cref{alg:rob_fl_confinf}.
First, we characterize the conformity scores $\{s_j^{(k)}\}_{j\in[n_k]}$ with a vector $\rvv^{(k)} \in \mathbb{R}^H$ for client $k$ $(k \in [K])$ via histogram statistics as \Cref{eq:score2vec}. 
Then, we compute the pairwise $\ell_p$-norm ($p \in \mathbb{Z}^+$) vector distance and the maliciousness scores for clients, which are the averaged vector distance to the clients in the $K_b-1$ nearest neighbors, where $K_b$ is the number of benign clients.
Finally, the benign set identified by \name $\gB_{\text{\name}}$ is the set of the index of the clients with the lowest $K_b$ maliciousness scores in $\{M(k)\}_{k=1}^K$.

\begin{algorithm}[t]
    \caption{Malicious client identification}\label{alg:rob_fl_confinf}
    \begin{algorithmic}[1]
        \STATE {\bfseries Input:} number of clients $K$, number of benign clients $K_b$, sets of scores for $K$ clients $\left\{s_j^{(k)}\right\}_{j\in[n_k], k \in [K]}$, parameter $p$ in $\ell_p$ norm distance.
        \STATE {\bfseries Output:} set of benign clients $\gB_{\text{\name}}$.

        \FOR{$k=1$ {\bfseries to} $K$}
        \STATE Characterize the conformity score observations $\left\{s_j^{(k)}\right\}_{j\in[n_k]}$ with a vector $\rvv^{(k)}$ for client $k$ as \Cref{eq:score2vec}.
        \ENDFOR

        \FOR{$k_1=1$ {\bfseries to} $K$}
        \FOR{$k_2=1$ {\bfseries to} $K$}
        \STATE Compute the vector distance $d_{k_1,k_2} \gets \|\rvv^{(k_1)} - \rvv^{(k_2)} \|_p$.
        \ENDFOR
        \ENDFOR

        \FOR{$k=1$ {\bfseries to} $K$}
        \STATE Compute the set of index of $K_b-1$ nearest neighbors for client $k$: $N_{ear}(k,K_b-1)$.
        \STATE Compute maliciousness scores of client $k$ as $M(k) \gets \dfrac{1}{K_b-1} \sum_{k' \in N_{ear}(k,K_b-1)} d_{k,k'}$.
        \ENDFOR

        \STATE Compute the index set of benign clients $\gB_{\text{\name}}$ as the associated index of the lowest $K_b$ maliciousness scores in $\left\{M(k))\right\}_{k=1}^K$.

    \end{algorithmic}
\end{algorithm}

\section{Experiments}

\subsection{Experiment setup}
\label{sec:app_exp_set}
\textbf{Datasets}. We evaluate \name on computer vision datasets including MNIST \citep{deng2012mnist}, CIFAR-10 \citep{cifar}, and Tiny-ImageNet (T-ImageNet) \citep{le2015tiny}.
We additionally evaluate \name on two realistic healthcare datasets, including SHHS \citep{zhang2018national} and PathMNIST \citep{medmnistv2}. 
The MNIST dataset consists of a collection of 70,000 handwritten digit images, each of which is labeled with the corresponding digit (0 through 9) that the image represents. 
CIFAR-10 consists of 60,000 32x32 color images, each belonging to one of the following 10 classes: airplane, automobile, bird, cat, deer, dog, frog, horse, ship, and truck.
Tiny-ImageNet consists of 200 different classes, each represented by 500 training images, making a total of 100,000 training images. Additionally, it has 10,000 validation images and 10,000 test images, with 50 images per class for both validation and test sets.
Each image in Tiny-ImageNet is a 64x64 colored image.
SHHS (the Sleep Heart Health Study) is a large-scale multi-center study to determine consequences of sleep-disordered breathing.
We use the EEG recordings from SHHS for the sleep-staging task, where every 30-second-epoch is classified into Wake, N1, N2, N3 and REM stages.
2,514 patients (2,545,869 samples) were used for training the DNN, and 2,514 patients (2,543,550 samples) were used for calibration and testing.
PathMNIST is a 9-class classification dataset consisting of 107,180 hematoxylin and eosin stained histological images.
89,996 images were used to train the DNN and 7,180 were used for calibration and testing.

\noindent\textbf{Training and evaluation strategy.}
Except for SHHS, we partition the datasets by sampling the proportion of each label from Dirichlet distribution parameterized by $\beta$ for every agent, following the literature \citep{li2022federated}.
For SHHS, we assign the patients to different clients according to the proportion of their time being awake.
The parameter of the Dirichlet distribution is fixed as $0.5$ across the evaluations.
We pretrain the models with standard FedAvg algorithm \citep{mcmahan2016federated}.
We use the same collaboratively pretrained model for conformal prediction for different methods for fair comparisons.
We perform conformal prediction with nonconformity scores LAC \citep{sadinle2019least} and APS \citep{romano2020classification}. 
Without specification, we use the LAC score by default across evaluations.
Given a pretrained estimator $\hat{\pi}: \sR^d \mapsto \Delta^C$ with $d$-dimensional input and $C$ classes, the LAC non-conformity score is formulated as:
\begin{equation}
    S^{\text{LAC}}_{\hat{\pi}_y}(x,y) = 1-\hat{\pi}_{y}(x).
\end{equation}
The APS non-conformity score is formulated as:
\begin{equation}
     S^{\text{APS}}_{\hat{\pi}_y}(x,y) = \sum\nolimits_{j \in \gY} \hat{\pi}_{j}(x) \sI{[\hat{\pi}_{j}(x) > \hat{\pi}_{y}(x)]} + \hat{\pi}_{y}(x) u,
\end{equation}
where {$\sI{[\cdot]}$} is the indicator function and {$u$} is uniformly sampled over the interval {$[0,1]$}.

\textbf{Byzantine attacks.} To evaluate the robustness of \name in the Byzantine setting, we compare \name with the baseline FCP \citep{lu2023federated} under three types of Byzantine attacks: (1) \textit{coverage attack} (CovAttack) which reports the largest conformity scores to induce a larger conformity score at the desired quantile and a lower coverage accordingly, (2) \textit{efficiency attack} (EffAttack) which reports the smallest conformity scores to induce a lower conformity score at the quantile and a larger prediction set, and (3) Gaussian Attack (GauAttack) which injects random Gaussian noises to the scores to perturb the conformal calibration.
The gaussian noises are sampled from a univariate Gaussian $\gN(0,0.5)$ with zero mean and $0.5$ variance.

\subsection{Additional evaluation results}
\label{app:res}
\vspace{-3em}

\begin{table}[t]
    \centering
    \caption{Benign conformal prediction results (marginal coverage / average set size) without any malicious clients.}
    \begin{tabular}{c|cc}
    \toprule
    Data Partition & $\beta=0.0$ & $\beta=0.5$ \\
    \midrule
     MNIST   &  0.898 / 0.900 &  0.902 / 1.828 \\
     CIFAR-10    & 0.901 / 1.597 & 0.898 / 2.308 \\
     Tiny-ImageNet    & 0.901 / 21.92 & 0.899 / 42.35 \\
     SHHS    & 0.898 / 1.352 & 0.897 / 1.351 \\
     PathMNIST    & 0.904 / 1.242 & 0.901 / 1.361 \\
     \bottomrule
    \end{tabular}
    \label{tab:benign}
\end{table}

\reb{
\begin{table}[t]
    \centering
    \caption{
    \reb{Marginal coverage / average set size on SHHS with heterogeneous data partition based on different attributes: wake time, N1, N2, N3, REM. The evaluation is done under different Coverage attack with $40\%~(K_m/K=40\%)$ malicious clients. The desired marginal coverage is $0.9$.  }}
    \label{tab:shhs_noniid}
    \footnotesize
    \begin{tabular}{c|ccccc}
    \toprule
     & wake time & N1 & N2 & N3 & REM \\
     \midrule
      FCP (SHHS) & 0.835 / 1.098 & 0.841 / 1.104 & 0.841 / 1.104 & 0.837 / 1.105 & 0.840 / 1.107 \\
 \name (SHHS) & 0.901 / 1.367 & 0.902 / 1.358 & 0.902 / 1.355 & 0.902 / 1.375 & 0.900 / 1.356 \\
\bottomrule
    \end{tabular}
\end{table}
}

\reb{
\begin{table}[t]
    \centering
    \caption{
    \reb{Runtime of RobFCP quantile computation with $40\%$ malicious clients. The valuation is done on a RTX A6000 GPU. }}
    \label{tab:runtime}
    \begin{tabular}{c|ccccc}
    \toprule
     & MNIST & CIFAR-10 & Tiny-ImageNet & SHHS & PathMNIST \\
     \midrule
      Runtime (seconds) & 0.5284 & 0.5169 & 0.5563 & 0.2227 & 0.3032 \\
\bottomrule
    \end{tabular}
\end{table}
}


\paragraph{Robustness of \name across varying levels of data heterogeneity} Data heterogeneity among clients poses significant challenges to achieving precise federated conformal prediction.
To assess the resilience of \name to this issue, we conducted evaluations using various values of the Dirichlet parameter $\beta$, which modulates the degree of data heterogeneity among clients.
The results in \Cref{tab:different_beta} show that \name reliably maintains marginal coverage and average set size at levels close to those anticipated, underscoring its robustness in the face of data skewness.
Furthermore, we investigate additional approaches to create heterogeneous data that mirror demographic differences. This involves dividing the SHHS dataset according to five specific attributes (wake time, N1, N2, N3, REM) and allocating instances to clients based on varying intervals of these attributes. The findings, detailed in \Cref{tab:shhs_noniid}, highlight \name's capability to effectively handle diverse forms of data heterogeneity.
\reb{
\paragraph{Runtime of \name} We evaluate the runtime of quantile computation in \name in \Cref{tab:runtime}, which indicates the efficiency of federated conformal prediction with \name.
}
\reb{
\begin{table}[t]
    \centering
    \caption{
    \reb{Marginal coverage / average set size under different Coverage attack with underestimated and overestimated numbers of malicious clients on TinyImageNet. The true ratio of malicious clients is $40\%~(K_m/K=25\%)$, while we evaluate Rob-FCP with different ratios of malicious clients $K’_m/K$ ranging from $5\%$ to $45\%$. The desired marginal coverage is $0.9$.  }}
    \small
    \label{tab:number}
    \begin{tabular}{c|cccccccccc}
    \toprule
     $K’_m/K$ & $5\%$ & $10\%$ & $15\%$ & $20\%$ & $25\%$ &  $30\%$ & $35\%$ & $40\%$ & $45\%$ \\
     \midrule
      Coverage & 0.8682 & 0.8756 & 0.8812 & 0.8884 & 0.9078 & 0.8936 & 0.8921/ & 0.8834 & 0.8803\\
      Set Size & 35.875 & 37.130 & 38.372 & 40.643 & 44.578 & 42.173 & 42.023 & 38.346 & 38.023 \\
\bottomrule
    \end{tabular}
\end{table}
}
\reb{
\paragraph{Results with an overestimate or underestimate of the number of malicious clients}
In \Cref{tab:number}, we provided evaluations of Rob-FCP with incorrect numbers of malicious clients. The results show that either underestimated numbers or overestimated numbers would harm the performance to different extents. Specifically, an underestimate of the number of malicious clients will definitely lead to the inclusion of malicious clients in the identified set $\mathcal{B}$ and downgrade the quality of conformal prediction. On the other hand, an overestimated number will lead to the exclusion of some benign clients. The neglect of non-conformity scores of those clients will lead to a distribution shift from the true data distribution in the calibration process, breaking the data exchangeability assumption of conformal prediction, and a downgraded performance. Therefore, correctly estimating the number of malicious clients is of significance, and this is why we propose the malicious client number estimator, which is sound both theoretically and empirically to achieve the goal.
}

\paragraph{Benign conformal performance}
The benign conformal prediction performance (marginal coverage / average set size) without any malicious clients is provided in \Cref{tab:benign}.
As expected, the coverage of the prediction sets is very close to the target (0.9). 
In the setting with data heterogeneity across clients (i.e., $\beta=0.5$), the predictive performance of the base global model is typically worse, leading to a larger average size of the prediction sets.

\paragraph{Byzantine robustness of \name with known $K_m$} 
We evaluate the marginal coverage and average set size of \name under coverage, efficiency, and Gaussian attack and compare the results with the baseline FCP.
We present results of FCP and \name in existence of $10\%,20\%,30\%~(K_m/K=10\%,20\%,30\%)$ malicious clients on MNIST, CIFAR-10, Tiny-ImageNet (T-ImageNet), SHHS, and PathMNIST in \Cref{table:appendix:known}.
The coverage of FCP deviates drastically from the desired coverage level $0.9$ under Byzantine attacks, along with a deviation from the benign set size.
In contrast, \name achieves comparable marginal coverage and average set size to the benign conformal performance.

\paragraph{Byzantine robustness of \name with unknown $K_m$} 
Similar to above, we evaluate the marginal coverage and average set size of \name under verious attacks and compare the results with the FCP.
We present results in existence of $10\%,20\%,30\%,40\%~(K_m/K=10\%,20\%,30\%,40\%)$ malicious clients in \Cref{table:appendix:unknown}, where the number of the malicious clients is unknown to the algorithm.
Again, the coverage of FCP as well as the size of the prediction set deviates drastically from the benign set setting, but \name achieves comparable marginal coverage and average set size to the benign performance.

\begin{table}[htbp]
\centering
\caption{Marginal coverage / average set size on Tiny-ImageNet with the desired level 0.9. The evaluation is conducted with different ratios of malicious clients $K_m/K$ and different degrees of data heterogeneity $\beta$ under mimic attack (MA). Mimic attack can not effectively distort the coverage for different data heterogeneity; Rob-FCP also maintains the coverage robustly.}
\label{tab:mimick}
\resizebox{1.0\linewidth}{!}{
\begin{tabular}{c|cccccc}
\toprule
& $\beta=0.0$ & $\beta=0.1$ & $\beta=0.3$ & $\beta=0.5$ & $\beta=0.7$ & $\beta=0.9$ \\
\hline
Benign & 0.898 / 21.728 & 0.896 / 43.038 & 0.903 / 43.864 & 0.899 / 42.352 & 0.902 / 43.843 & 0.904 / 43.919 \\
\hline
MA ($K_m/K=10\%$) & 0.900 / 22.251 & 0.904 / 44.684 & 0.905 / 44.701 & 0.891 / 42.277 & 0.898 / 42.939 & 0.906 / 44.240 \\
MA + Rob-FCP ($K_m/K=10\%$) & 0.903 / 23.823 & 0.893 / 41.994 & 0.899 / 43.169 & 0.901 / 43.734 & 0.909 / 44.811 & 0.897 / 42.632 \\
\hline
MA ($K_m/K=20\%$) & 0.895 / 22.243 & 0.894 / 42.738 & 0.894 / 42.412 & 0.893 / 41.633 & 0.904 / 44.878 & 0.906 / 43.941 \\
MA + Rob-FCP ($K_m/K=20\%$) & 0.902 / 22.651 & 0.895 / 41.575 & 0.901 / 42.770 & 0.905 / 44.124 & 0.899 / 43.793 & 0.897 / 42.589 \\
\hline
MA ($K_m/K=30\%$) & 0.905 / 23.414 & 0.910 / 46.940 & 0.883 / 37.411 & 0.899 / 42.766 & 0.888 / 41.012 & 0.896 / 41.846 \\
MA + Rob-FCP ($K_m/K=30\%$) & 0.898 / 22.839 & 0.912 / 47.390 & 0.906 / 44.882 & 0.893 / 41.481 & 0.906 / 45.372 & 0.897 / 42.813 \\
\hline
MA ($K_m/K=40\%$) & 0.892 / 19.629 & 0.901 / 44.468 & 0.896 / 42.526 & 0.908 / 46.017 & 0.911 / 46.445 & 0.914 / 47.553 \\
MA + Rob-FCP ($K_m/K=40\%$) & 0.899 / 20.952 & 0.904 / 45.023 & 0.905 / 43.368 & 0.908 / 46.518 & 0.915 / 47.023 & 0.892 / 40.561 \\
\bottomrule
\end{tabular}}
\end{table}

\paragraph{Robustness of \name against mimic attacks} We also evaluate the performance of \name against the mimic attack strategy \cite{karimireddy2022byzantinerobust}, wherein malicious clients replicate the score statistics of a randomly chosen benign client. It's critical to note that such strategies presuppose that the attackers have knowledge of the benign clients' score statistics, implying a more restricted threat model. We conduct the evaluations on Tiny-ImageNet with $1-\alpha=0.9$ with different ratios of malicious clients $K_m/K$. The results in \Cref{tab:mimick} show that (1) across various degrees of data heterogeneity, merely approximating the scores of benign clients is insufficient to significantly impair the performance of conformal prediction; and (2) Rob-FCP still maintains the desired coverage under such attacks. 

\paragraph{Robustness of \name with different conformity scores} Besides applying LAC nonconformity scores, we also evaluate \name with APS scores \citep{romano2020classification}.
The results in \Cref{fig:abl_cov_iid,fig:abl_cov_noniid,fig:abl_eff_iid,fig:abl_eff_noniid,fig:abl_gauss_iid,fig:abl_gauss_noniid} demonstrate the Byzantine robustness of \name with APS scores.

\paragraph{Ablation study of different conformity score distribution characterization} 
One key step in \name is to characterize the conformity score distribution based on empirical observations. We adopt the histogram statistics approach as \Cref{eq:score2vec}. 
\name also flexibly allows for alternative approaches to characterizing the empirical conformity score samples with a real-valued vector $\rvv$.
We can fit a parametric model (e.g., Gaussian model) to the empirical scores and concatenate the parameters as the characterization vector $\rvv$. Another alternative is to characterize the score samples with exemplars approximated by clustering algorithms such as KMeans.
We empirically compare different approaches in \Cref{fig:abl_method_noniid} and show that the histogram statistics approach achieves the best performance. 

\paragraph{Ablation study of the distance measurement} 
In \name, we need to compute the distance between characterization vectors with measurement $d(\cdot,\cdot)$. We evaluate \name with $\ell_1$, $\ell_2$, $\ell_\infty$-norm based vector distance as \Cref{eq:dis} and an alternative Cosine similarity in \Cref{fig:abl_dis}. The results show that the effectiveness of \name is agnostic to these commonly used distance measurements. We adopt $\ell_2$-norm vector distance for consistency across evaluations.

\begin{table}[h!]
\centering
\caption{Marginal coverage / average set size of Rob-FCP on Tiny-ImageNet with the desired level 0.9 under Gaussian Attack with standard deviation $0.5$ inexistence of $40\%$ malicious clients.}
\begin{tabular}{c|ccccc}
\toprule
$H$ & 2 & 10 & 100 & 1000 & 10000 \\ \hline
Marginal coverage / average set size & 0.718 / 12.293 & 0.888 / 40.250 & 0.901 / 43.349 & 0.907 / 44.677 & 0.803 / 26.343 \\ \bottomrule
\end{tabular}
\label{tab:H}
\end{table}

\paragraph{Ablation study of the selection of histogram granularity $H$} We also add empirical evaluations to validate the trade-off of the selection of $H$ in \Cref{tab:H}. The results in \Cref{tab:H} demonstrate the empirical trade-off of the selection of dimensionality $H$ and show that Rob-FCP remains effective for a broad range of $H$ ($H=10$ to $H=1000$).

\def \TableAppendixKnown{
\begin{table}[ht]
\caption{
Marginal coverage / average set size under different Byzantine attacks with $10\%$, $20\%$, and $30\%$ malicious clients.
\name consistently recovers the coverage (and average size of prediction set) of benign conformal prediction (\Cref{tab:benign}), while the performance of FCP generally deteriorates as the percentage of malicious clients increases. $\beta$ denotes the Dirichlet parameter for the partition of client data.
}
\label{table:appendix:known}
\centering
\begin{small}
   \resizebox{1\columnwidth}{!}{
\begin{tabular}{p{0.005cm}c|cc|cc|cc}
\toprule
\multicolumn{2}{c}{Attack} & \multicolumn{2}{c}{Coverage Attack} & \multicolumn{2}{c}{Efficiency Attack} & \multicolumn{2}{c}{Gaussian Attack}  \\
        \multicolumn{2}{c}{Method}  & FCP & \name & FCP & \name & FCP & \name \\
\midrule
\multicolumn{2}{c}{$K_m/K = 10\%$}\\
\midrule
\multirow{5}{*}{\rotatebox{90}{$\beta=0.0$}} & MNIST & 0.896 / 0.898 & \textbf{0.899} / \textbf{0.900} & 0.999 / 4.034 & \textbf{0.904} / \textbf{0.909} & 0.947 / 0.960 & \textbf{0.905} / \textbf{0.910} \\
     &CIFAR-10 & 0.887 / 1.499 & \textbf{0.900} / \textbf{1.588} & 1.000 / 7.991  & \textbf{0.892} / \textbf{1.556} & 0.906 / 1.633 & \textbf{0.892} / \textbf{1.565} \\
     &T-ImageNet & 0.873 / 18.44 & \textbf{0.901} / \textbf{22.36} & 0.999 / 148.7  & \textbf{0.895} / \textbf{21.28} & 0.916 / 23.98 & \textbf{0.909} / \textbf{23.80} \\
     &SHHS & 0.889 / 1.303 & \textbf{0.900} / \textbf{1.359} & 0.999 / 5.338 & \textbf{0.900} / \textbf{1.359} & 0.909 / 1.409 & \textbf{0.900} / \textbf{1.360} \\
     &PathMNIST & 0.892 / 1.184 & \textbf{0.905 / 1.249} &     1.000 / 6.271 & \textbf{0.902 / 1.235} &         0.941 / 1.504 & \textbf{0.903 / 1.240} \\
     \midrule
     \multirow{5}{*}{\rotatebox{90}{$\beta=0.5$}} & MNIST & 0.892 / 1.747 & \textbf{0.897} / \textbf{1.813} & 1.000 / 9.319 & \textbf{0.896} / \textbf{1.813} &  0.892 / 1.798 & \textbf{0.902} / \textbf{1.794} \\
     &CIFAR-10 & 0.887 / 1.209 & \textbf{0.894} / \textbf{2.287}  & 1.000 / 8.808  & \textbf{0.908} / \textbf{2.347} & 0.918 / 2.515 & \textbf{0.911} / \textbf{2.378} \\
     &T-ImageNet & 0.892 / 41.03 & \textbf{0.905} / \textbf{44.81}  & 0.997 / 146.7  & \textbf{0.902} / \textbf{44.29} & 0.917 / 47.47 & \textbf{0.900} / \textbf{44.74} \\
     &SHHS & 0.889 / 1.304 & \textbf{0.900} / \textbf{1.358} & 1.000 / 5.981 & \textbf{0.900} / \textbf{1.359} & 0.909 / 1.412 & \textbf{0.901} / \textbf{1.361}\\
     &PathMNIST & 0.892 / 1.290 & \textbf{0.902 / 1.361} &     0.996 / 5.149 & \textbf{0.900 / 1.348} &         0.938 / 1.739 & \textbf{0.904 / 1.374}\\
\midrule
\multicolumn{2}{c}{$K_m/K = 20\%$}\\
\midrule
\multirow{5}{*}{\rotatebox{90}{$\beta=0.0$}} & MNIST & 0.873 / 0.876 & \textbf{0.893} / \textbf{0.897} & 1.000 / 10.00 & \textbf{0.895} / \textbf{0.899} & 0.967 / 0.988 & \textbf{0.900} / \textbf{0.905} \\
     &CIFAR-10 & 0.869 / 1.398 & \textbf{0.888} / \textbf{1.532} & 1.000 / 10.00  & \textbf{0.913} / \textbf{1.659} & 0.916 / 1.725 & \textbf{0.903} / \textbf{1.633} \\
     &T-ImageNet & 0.874 / 17.787 & \textbf{0.900} / \textbf{22.23} & 1.000 / 200.0  & \textbf{0.903} / \textbf{22.50} & 0.908 / 23.12 & \textbf{0.904} / \textbf{22.94} \\
     &SHHS & 0.876 / 1.243 & \textbf{0.900} / \textbf{1.359} & 1.000 / 5.984 & \textbf{0.900} / \textbf{1.356} & 0.918 / 1.467 & \textbf{0.900} / \textbf{1.360} \\
     &PathMNIST & 0.880 / 1.134 & \textbf{0.905 / 1.251} &     1.000 / 8.335 & \textbf{0.904 / 1.244} &         0.983 / 2.434 & \textbf{0.903 / 1.236} \\
     \midrule
     \multirow{5}{*}{\rotatebox{90}{$\beta=0.5$}} & MNIST & 0.857 / 1.534 & \textbf{0.896} / \textbf{1.765} & 1.000 / 9.089 & \textbf{0.902} / \textbf{1.836} &  0.915 / 1.945 & \textbf{0.912} / \textbf{1.904} \\
     &CIFAR-10 & 0.866 / 2.038 & \textbf{0.896} / \textbf{2.314}  & 1.000 / 10.00  & \textbf{0.908} / \textbf{2.366} & 0.938 / 2.895 & \textbf{0.892} / \textbf{2.256} \\
     &T-ImageNet & 0.860 / 33.99 & \textbf{0.902} / \textbf{44.69}  & 1.000 / 199.0  & \textbf{0.904} / \textbf{44.72} & 0.922 / 49.44 & \textbf{0.912} / \textbf{48.27} \\
     &SHHS & 0.874 / 1.236 & \textbf{0.901} / \textbf{1.363} & 1.000 / 5.985 & \textbf{0.901} / \textbf{1.363} & 0.917 / 1.463 & \textbf{0.900} / \textbf{1.358} \\
     &PathMNIST & 0.876 / 1.210 & \textbf{0.901 / 1.355} &     1.000 / 7.395 & \textbf{0.902 / 1.366} &         0.980 / 2.905 & \textbf{0.900 / 1.348} \\
\midrule
\multicolumn{2}{c}{$K_m/K = 30\%$}\\
\midrule
\multirow{5}{*}{\rotatebox{90}{$\beta=0.0$}} & MNIST & 0.851 / 0.854 & \textbf{0.908} / \textbf{0.914} & 1.000 / 10.00 & \textbf{0.911} / \textbf{0.917} & 0.977 / 1.009 & \textbf{0.900} / \textbf{0.905} \\
     &CIFAR-10 & 0.852 / 1.307 & \textbf{0.895} / \textbf{1.583} & 1.000 / 10.00  & \textbf{0.894} / \textbf{1.563} & 0.909 / 1.672 & \textbf{0.903} / \textbf{1.602} \\
     &T-ImageNet & 0.862 / 15.66 & \textbf{0.904} / \textbf{22.61} & 1.000 / 200.0  & \textbf{0.907} / \textbf{22.85} & 0.907 / 23.89 & \textbf{0.906} / \textbf{24.15} \\
     &SHHS & 0.859 / 1.176 & \textbf{0.901} / \textbf{1.364} & 1.000 / 6.000 & \textbf{0.900} / \textbf{1.356} & 0.926 / 1.526 & \textbf{0.900} / \textbf{1.359} \\
     &PathMNIST & 0.863 / 1.064 & \textbf{0.906 / 1.252} & 1.000 / 9.000 & \textbf{0.903 / 1.241} & 1.000 / 6.531 & \textbf{0.906 / 1.255} \\
     \midrule
     \multirow{5}{*}{\rotatebox{90}{$\beta=0.5$}} & MNIST & 0.849 / 1.451 & \textbf{0.913} / \textbf{1.890} & 1.000 / 10.00 & \textbf{0.875} / \textbf{1.650} &  0.925 / 2.010 & \textbf{0.919} / \textbf{1.958} \\
     &CIFAR-10 & 0.844 / 1.870 & \textbf{0.900} / \textbf{2.294}  & 1.000 / 10.00  & \textbf{0.912} / \textbf{2.408} & 0.950 / 3.152 & \textbf{0.901} / \textbf{2.327} \\
     &T-ImageNet & 0.864 / 33.41 & \textbf{0.895} / \textbf{43.12}  & 1.000 / 200.0  & \textbf{0.906} / \textbf{43.46} & \textbf{0.923} / \textbf{52.23} & 0.932 / 55.78 \\
     &SHHS & 0.857 / 1.169 & \textbf{0.900} / \textbf{1.358} & 1.000 / 6.000 & \textbf{0.900} / \textbf{1.358} & 0.927 / 1.530 & \textbf{0.898} / \textbf{1.350} \\
     &PathMNIST & 0.860 / 1.141 & \textbf{0.900 / 1.344} & 1.000 / 9.000 & \textbf{0.903 / 1.368} & 1.000 / 6.287 & \textbf{0.903 / 1.373} \\
\bottomrule
\bottomrule
\end{tabular}
}
\end{small}
\end{table}
}
\TableAppendixKnown
\def \TableAppendixUnknown{
\begin{table}[ht]
\caption{
Marginal coverage / average set size under different Byzantine attacks with $10\%$, $20\%$, $30\%$ and $40\%$ malicious clients with unknown numbers of malicious clients.
\name consistently recovers the coverage (and average size of prediction set) of benign conformal prediction (\Cref{tab:benign}), while the performance of FCP generally deteriorates as the percentage of malicious clients increases. $\beta$ denotes the Dirichlet parameter for the partition of client data.
}
\label{table:appendix:unknown}
\centering
\begin{small}
   \resizebox{1\columnwidth}{!}{
\begin{tabular}{p{0.005cm}c|cc|cc|cc}
\toprule
\multicolumn{2}{c}{Attack} & \multicolumn{2}{c}{Coverage Attack} & \multicolumn{2}{c}{Efficiency Attack} & \multicolumn{2}{c}{Gaussian Attack}  \\
        \multicolumn{2}{c}{Method}  & FCP & \name & FCP & \name & FCP & \name \\
\midrule
\multicolumn{2}{c}{$K_m/K = 10\%$}\\
\midrule
\multirow{5}{*}{\rotatebox{90}{$\beta=0.0$}} & MNIST & 0.896 / 0.898 & \textbf{0.901} / \textbf{0.905} & 0.999 / 4.034 & \textbf{0.890} / \textbf{0.895} & 0.947 / 0.960 & \textbf{0.895} / \textbf{0.900} \\
     &CIFAR-10 & 0.887 / 1.499 & \textbf{0.903} / \textbf{1.612} & 1.000 / 7.991  & \textbf{0.920} / \textbf{1.689} & 0.906 / 1.633 & \textbf{0.890} / \textbf{1.543} \\
     &T-ImageNet & 0.873 / 18.44 & \textbf{0.908} / \textbf{22.52} & 0.999 / 148.7  & \textbf{0.890} / \textbf{20.93} & 0.916 / 23.98 & \textbf{0.897} / \textbf{21.64} \\
     &SHHS & 0.889 / 1.303 & \textbf{0.902} / \textbf{1.365} & 0.999 / 5.338 & \textbf{0.903} / \textbf{1.368} & 0.909 / 1.409 & \textbf{0.902} / \textbf{1.367} \\
     &PathMNIST & 0.892 / 1.184 & \textbf{0.899} / \textbf{1.237} & 1.000 / 6.271 & \textbf{0.905} / \textbf{1.253} & 0.905 / 1.253 & \textbf{0.901} / \textbf{1.239} \\
     \midrule
     \multirow{5}{*}{\rotatebox{90}{$\beta=0.5$}} & MNIST & 0.892 / 1.747 & \textbf{0.895} / \textbf{1.798} & 1.000 / 9.319 & \textbf{0.900} / \textbf{1.780} &  0.892 / 1.798 & \textbf{0.896} / \textbf{1.800} \\
     &CIFAR-10 & 0.887 / 1.209 & \textbf{0.890} / \textbf{2.221}  & 1.000 / 8.808  & \textbf{0.900} / \textbf{2.304} & 0.918 / 2.515 & \textbf{0.905} / \textbf{2.418} \\
     &T-ImageNet & 0.892 / 41.03 & \textbf{0.903} / \textbf{43.94}  & 0.997 / 146.7  & \textbf{0.898} / \textbf{43.01} & 0.917 / 47.47 & \textbf{0.915} / \textbf{47.35} \\
     &SHHS & 0.889 / 1.304 & \textbf{0.902} / \textbf{1.367} & 1.000 / 5.981 & \textbf{0.902} / \textbf{1.364} & 0.909 / 1.412 & \textbf{0.900} / \textbf{1.357}\\
     &PathMNIST & \textbf{0.892} / 1.290 & 0.909 / \textbf{1.394} & 0.996 / 5.149 & \textbf{0.901} / \textbf{1.376} & \textbf{0.905} / 1.387 & 0.907 / \textbf{1.375}\\
\midrule
\multicolumn{2}{c}{$K_m/K = 20\%$}\\
\midrule
\multirow{5}{*}{\rotatebox{90}{$\beta=0.0$}} & MNIST & 0.873 / 0.876 & \textbf{0.898} / \textbf{0.903} & 1.000 / 10.00 & \textbf{0.906} / \textbf{0.912} & 0.967 / 0.988 & \textbf{0.904} / \textbf{0.908} \\
     &CIFAR-10 & 0.869 / 1.398 & \textbf{0.888} / \textbf{1.512} & 1.000 / 10.00  & \textbf{0.902} / \textbf{1.603} & 0.916 / 1.725 & \textbf{0.905} / \textbf{1.623} \\
     &T-ImageNet & 0.874 / 17.787 & \textbf{0.904} / \textbf{22.47} & 1.000 / 200.0  & \textbf{0.907} / \textbf{22.76} & 0.908 / 23.12 & \textbf{0.904} / \textbf{22.88} \\
     &SHHS & 0.876 / 1.243 & \textbf{0.902} / \textbf{1.365} & 1.000 / 5.984 & \textbf{0.902} / \textbf{1.366} & 0.918 / 1.467 & \textbf{0.902} / \textbf{1.363} \\
     &PathMNIST & 0.880 / 1.134 & \textbf{0.900} / \textbf{1.229} & 1.000 / 8.335 & \textbf{0.902} / \textbf{1.241} & 0.909 / 1.273 & \textbf{0.898} / \textbf{1.229} \\
     \midrule
     \multirow{5}{*}{\rotatebox{90}{$\beta=0.5$}} & MNIST & 0.857 / 1.534 & \textbf{0.901} / \textbf{1.832} & 1.000 / 9.089 & \textbf{0.881} / \textbf{1.713} &  0.915 / 1.945 & \textbf{0.908} / \textbf{1.889} \\
     &CIFAR-10 & 0.866 / 2.038 & \textbf{0.900} / \textbf{2.344}  & 1.000 / 10.00  & \textbf{0.897} / \textbf{2.312} & 0.938 / 2.895 & \textbf{0.929} / \textbf{2.702} \\
     &T-ImageNet & 0.860 / 33.99 & \textbf{0.905} / \textbf{44.38}  & 1.000 / 199.0  & \textbf{0.894} / \textbf{42.30} & 0.922 / 49.44 & \textbf{0.906} / \textbf{46.38} \\
     &SHHS & 0.874 / 1.236 & \textbf{0.901} / \textbf{1.362} & 1.000 / 5.985 & \textbf{0.903} / \textbf{1.369} & 0.917 / 1.463 & \textbf{0.902} / \textbf{1.365} \\
     &PathMNIST & 0.876 / 1.210 & \textbf{0.907} / \textbf{1.388} & 1.000 / 7.395 & \textbf{0.903} / \textbf{1.362} & 0.905 / 1.382 & \textbf{0.902} / \textbf{1.362} \\
\midrule
\multicolumn{2}{c}{$K_m/K = 30\%$}\\
\midrule
\multirow{5}{*}{\rotatebox{90}{$\beta=0.0$}} & MNIST & 0.851 / 0.854 & \textbf{0.905} / \textbf{0.912} & 1.000 / 10.00 & \textbf{0.907} / \textbf{0.913} & 0.977 / 1.009 & \textbf{0.903} / \textbf{0.908} \\
     &CIFAR-10 & 0.852 / 1.307 & \textbf{0.904} / \textbf{1.612} & 1.000 / 10.00  & \textbf{0.891} / \textbf{1.544} & 0.909 / 1.672 & \textbf{0.903} / \textbf{1.578} \\
     &T-ImageNet & 0.862 / 15.66 & \textbf{0.902} / \textbf{21.92} & 1.000 / 200.0  & \textbf{0.903} / \textbf{22.19} & 0.907 / 23.89 & \textbf{0.906} / \textbf{23.77} \\
     &SHHS & 0.859 / 1.176 & \textbf{0.903} / \textbf{1.372} & 1.000 / 6.000 & \textbf{0.902} / \textbf{1.366} & 0.926 / 1.526 & \textbf{0.903} / \textbf{1.368} \\
     &PathMNIST & 0.863 / 1.064 & \textbf{0.902} / \textbf{1.239} & 1.000 / 9.000 & \textbf{0.898} / \textbf{1.221} & 0.907 / 1.263 & \textbf{0.905} / \textbf{1.246} \\
     \midrule
     \multirow{5}{*}{\rotatebox{90}{$\beta=0.5$}} & MNIST & 0.849 / 1.451 & \textbf{0.920} / \textbf{1.947} & 1.000 / 10.00 & \textbf{0.900} / \textbf{1.779} &  0.925 / 2.010 & \textbf{0.911} / \textbf{1.943} \\
     &CIFAR-10 & 0.844 / 1.870 & \textbf{0.899} / \textbf{2.360}  & 1.000 / 10.00  & \textbf{0.891} / \textbf{2.264} & 0.950 / 3.152 & \textbf{0.896} / \textbf{2.300} \\
     &T-ImageNet & 0.864 / 33.41 & \textbf{0.895} / \textbf{42.79}  & 1.000 / 200.0  & \textbf{0.908} / \textbf{44.74} & 0.923 / 52.23 & \textbf{0.920} / \textbf{50.70} \\
     &SHHS & 0.857 / 1.169 & \textbf{0.902} / \textbf{1.368} & 1.000 / 6.000 & \textbf{0.904} / \textbf{1.374} & 0.927 / 1.530 & \textbf{0.903} / \textbf{1.370} \\
     &PathMNIST & 0.860 / 1.141 & \textbf{0.895} / \textbf{1.337} & 1.000 / 9.000 & \textbf{0.902} / \textbf{1.376} & 0.910 / 1.418 & \textbf{0.899} / \textbf{1.352} \\
\midrule
\multicolumn{2}{c}{$K_m/K = 40\%$}\\
\midrule
\multirow{5}{*}{\rotatebox{90}{$\beta=0.0$}} & MNIST & 0.832 / 0.834 & \textbf{0.891} / \textbf{0.892} & 1.000 / 10.00 & \textbf{0.895} / \textbf{0.901} & 0.979 / 1.025 & \textbf{0.899} / \textbf{0.904} \\
     &CIFAR-10 & 0.831 / 1.189 & \textbf{0.913} / \textbf{1.666} & 1.000 / 10.00  & \textbf{0.902} / \textbf{1.608} & 0.916 / 1.733 & \textbf{0.905} / \textbf{1.612} \\
     &T-ImageNet & 0.830 / 12.97 & \textbf{0.888} / \textbf{21.45} & 1.000 / 200.0  & \textbf{0.905} / \textbf{22.99} & 0.918 / 25.69 & \textbf{0.903} / \textbf{23.42} \\
     &SHHS & 0.834 / 1.093 & \textbf{0.902} / \textbf{1.363} & 1.000 / 6.000 & \textbf{0.903} / \textbf{1.369} & 0.937 / 1.611 & \textbf{0.902} / \textbf{1.368} \\
     &PathMNIST & 0.840 / 0.997 & \textbf{0.901} / \textbf{1.246} & 1.000 / 9.000 & \textbf{0.898} / \textbf{1.237} & 0.914 / 1.302 & \textbf{0.899} / \textbf{1.250} \\
     \midrule
     \multirow{5}{*}{\rotatebox{90}{$\beta=0.5$}} & MNIST & 0.805 / 1.284 & \textbf{0.911} / \textbf{1.929} & 1.000 / 10.00 & \textbf{0.910} / \textbf{1.906} &  0.941 / 2.227 & \textbf{0.929} / \textbf{2.084} \\
     &CIFAR-10 & 0.829 / 1.758 & \textbf{0.893} / \textbf{2.270}  & 1.000 / 10.00  & \textbf{0.888} / \textbf{2.203} & 0.970 / 3.863 & \textbf{0.923} / \textbf{2.635} \\
     &T-ImageNet & 0.825 / 27.84 & \textbf{0.906} / \textbf{45.18}  & 1.000 / 200.0  & \textbf{0.903} / \textbf{42.62} & 0.942 / 61.50 & \textbf{0.937} / \textbf{59.61} \\
     &SHHS & 0.835 / 1.095 & \textbf{0.902} / \textbf{1.364} & 1.000 / 6.000 & \textbf{0.904} / \textbf{1.375} & 0.937 / 1.609 & \textbf{0.903} / \textbf{1.371}\\
     &PathMNIST &0.837 / 1.055 & \textbf{0.903} / \textbf{1.378} & 1.000 / 9.000 & \textbf{0.909} / \textbf{1.398} & 0.915 / 1.464 & \textbf{0.914} / \textbf{1.488}\\
\bottomrule
\end{tabular}
}
\end{small}
\end{table}
}
\TableAppendixUnknown

\begin{figure}[t]
    \centering
    \includegraphics[width=\linewidth]{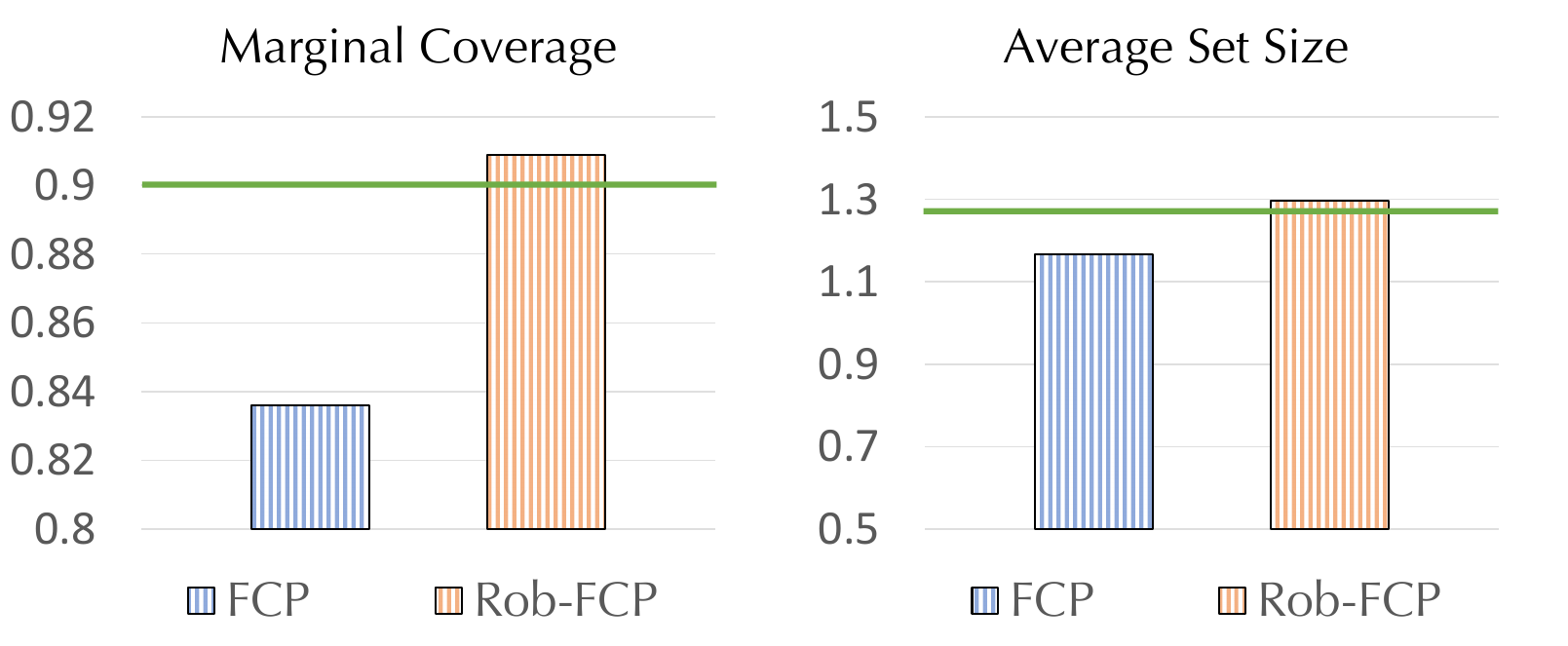}
    \caption{Marginal coverage / average set size under coverage attack with 40\% malicious clients with $\beta=0.0$ on CIFAR-10. The green horizontal line represents the benign marginal coverage and average set size without any malicious clients. }
    \label{fig:abl_cov_iid}
\end{figure}

\begin{figure}[t]
    \centering
    \includegraphics[width=\linewidth]{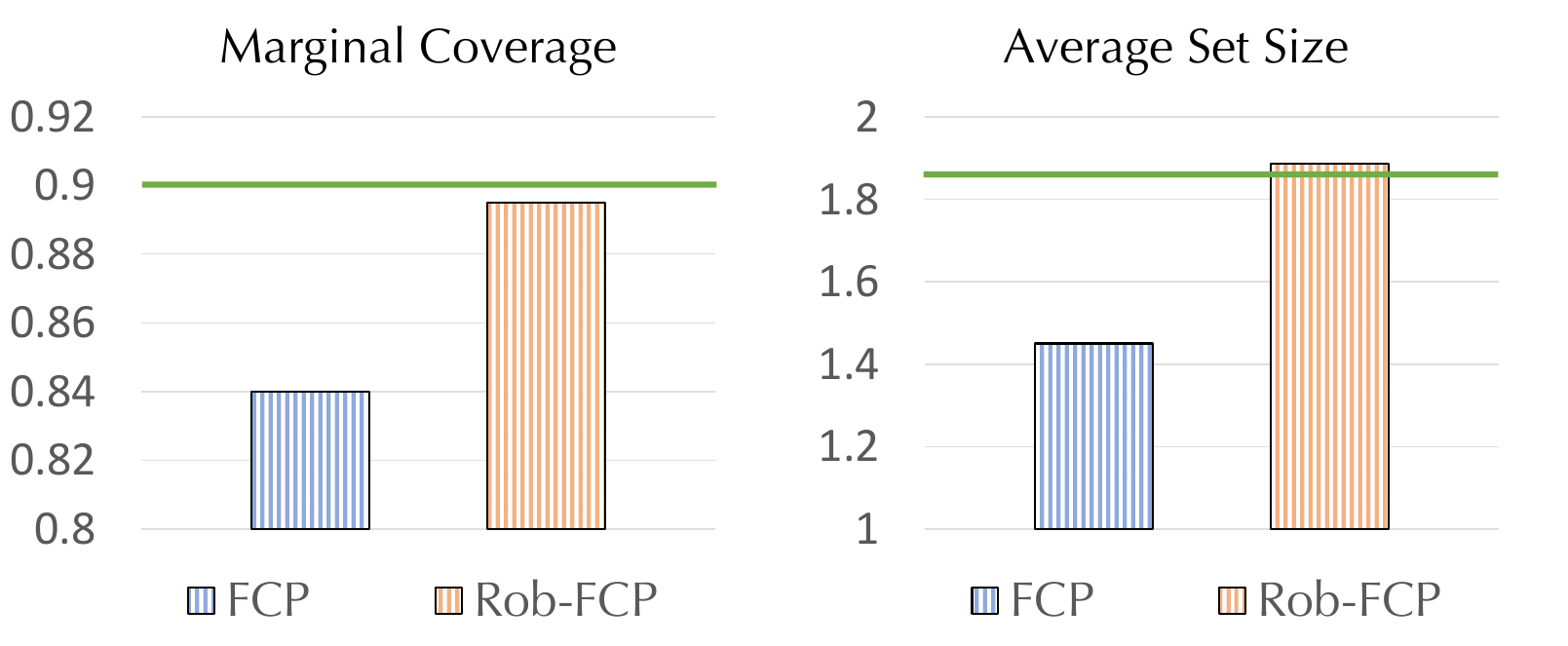}
    \caption{Marginal coverage / average set size under coverage attack with 40\% malicious clients with $\beta=0.5$ on CIFAR-10. The green horizontal line represents the benign marginal coverage and average set size without any malicious clients.}
    \label{fig:abl_cov_noniid}
\end{figure}

\begin{figure}[t]
    \centering
    \includegraphics[width=\linewidth]{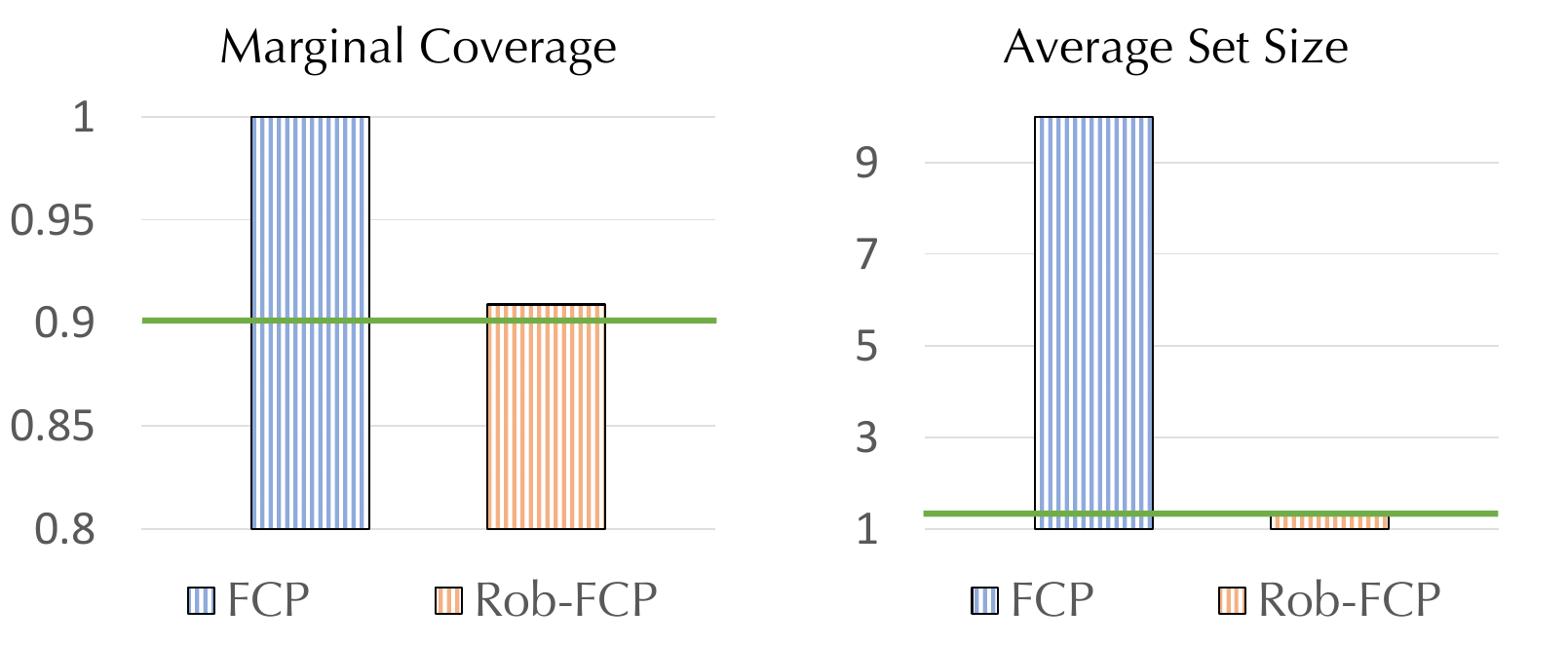}
    \caption{Marginal coverage / average set size under efficiency attack with 40\% malicious clients with $\beta=0.0$ on CIFAR-10. The green horizontal line represents the benign marginal coverage and average set size without any malicious clients.}
    \label{fig:abl_eff_iid}
\end{figure}

\begin{figure}[t]
    \centering
    \includegraphics[width=\linewidth]{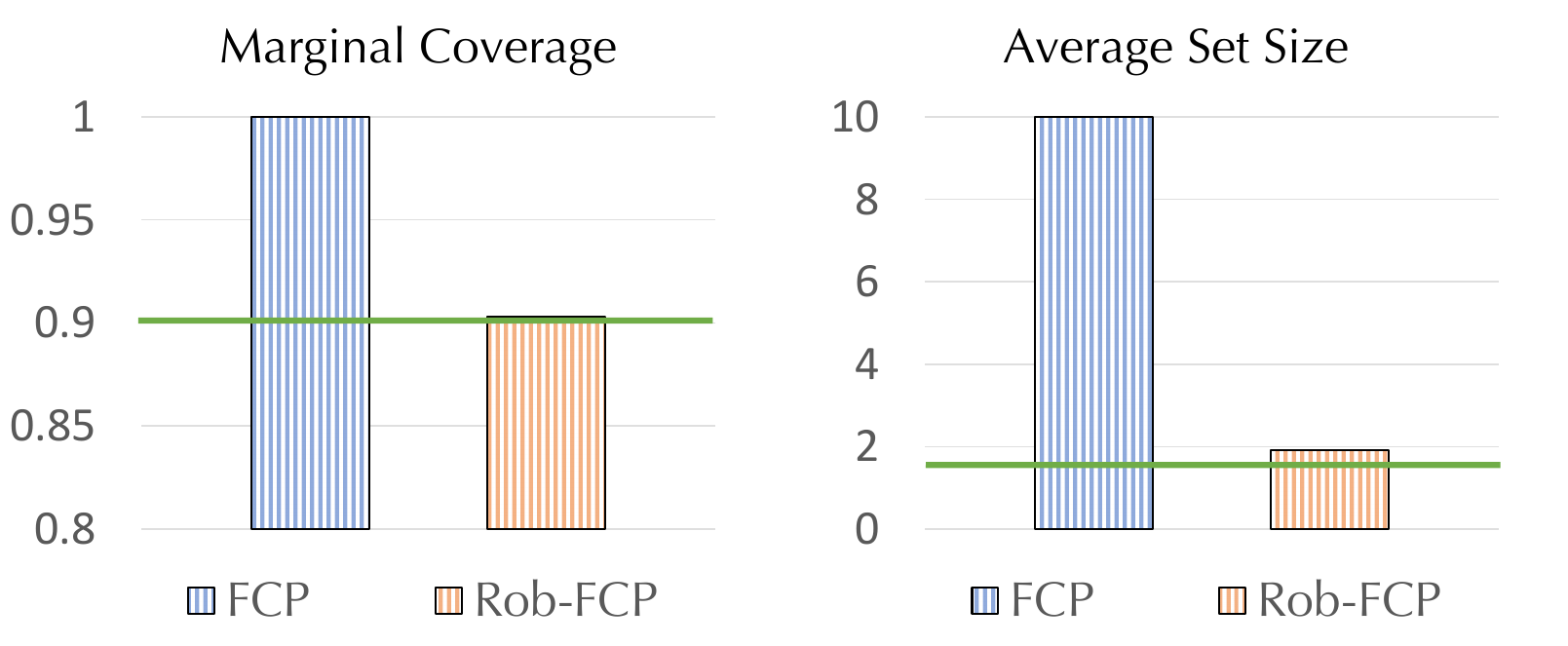}
    \caption{Marginal coverage / average set size under efficiency attack with 40\% malicious clients with $\beta=0.5$ on CIFAR-10. The green horizontal line represents the benign marginal coverage and average set size without any malicious clients.}
    \label{fig:abl_eff_noniid}
\end{figure}

\begin{figure}[t]
    \centering
    \includegraphics[width=\linewidth]{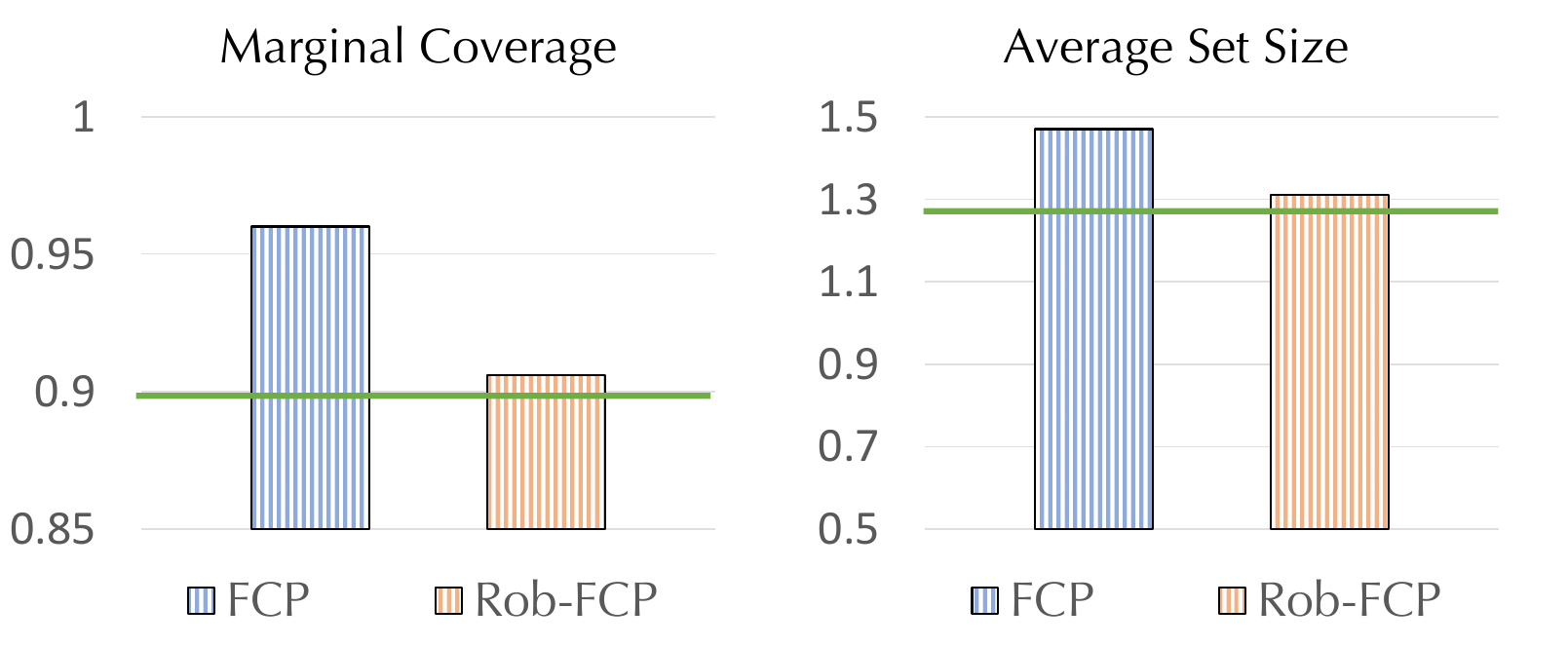}
    \caption{Marginal coverage / average set size under Gaussian attack with 40\% malicious clients with $\beta=0.0$ on CIFAR-10. The green horizontal line represents the benign marginal coverage and average set size without any malicious clients.}
    \label{fig:abl_gauss_iid}
\end{figure}

\begin{figure}[t]
    \centering
    \includegraphics[width=\linewidth]{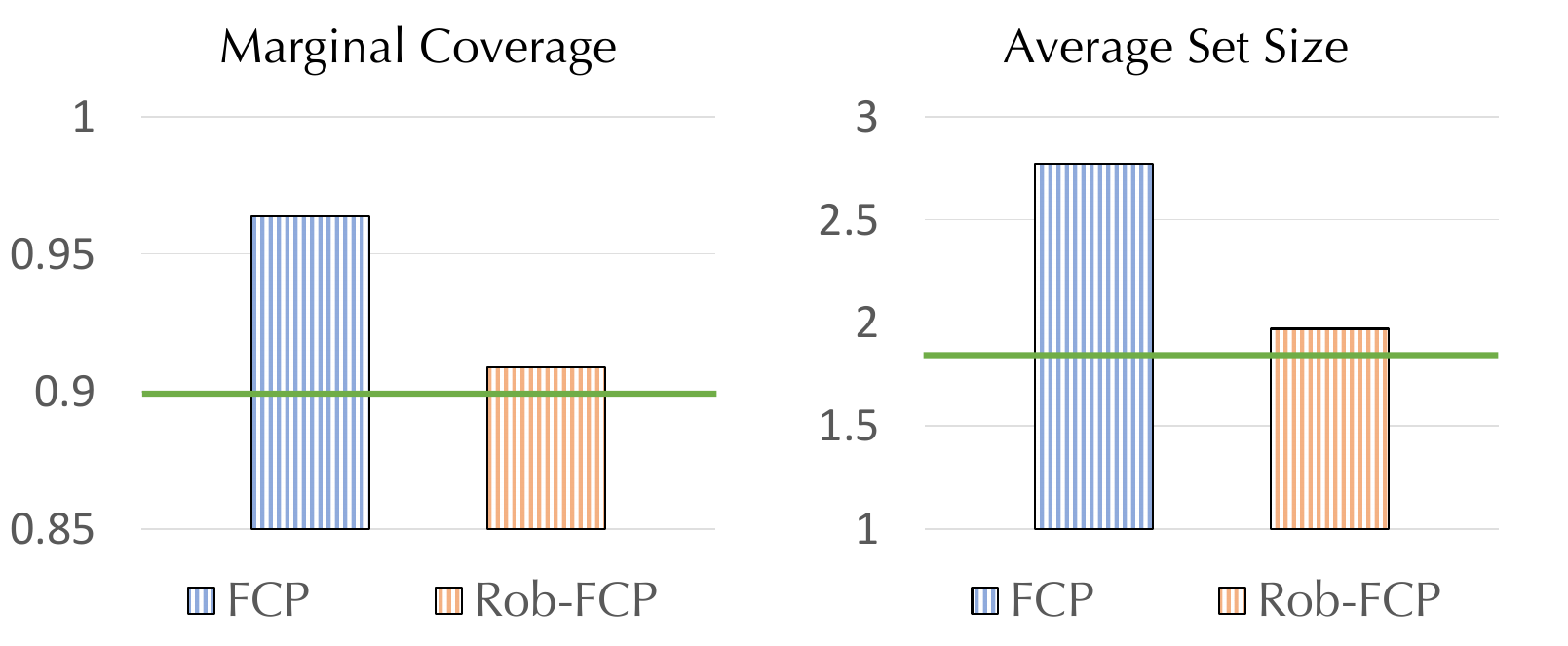}
    \caption{Marginal coverage / average set size under Gaussian attack with 40\% malicious clients with $\beta=0.5$ on CIFAR-10. The green horizontal line represents the benign marginal coverage and average set size without any malicious clients.}
    \label{fig:abl_gauss_noniid}
\end{figure}



\begin{figure}[t]
    \centering
    \includegraphics[width=\linewidth]{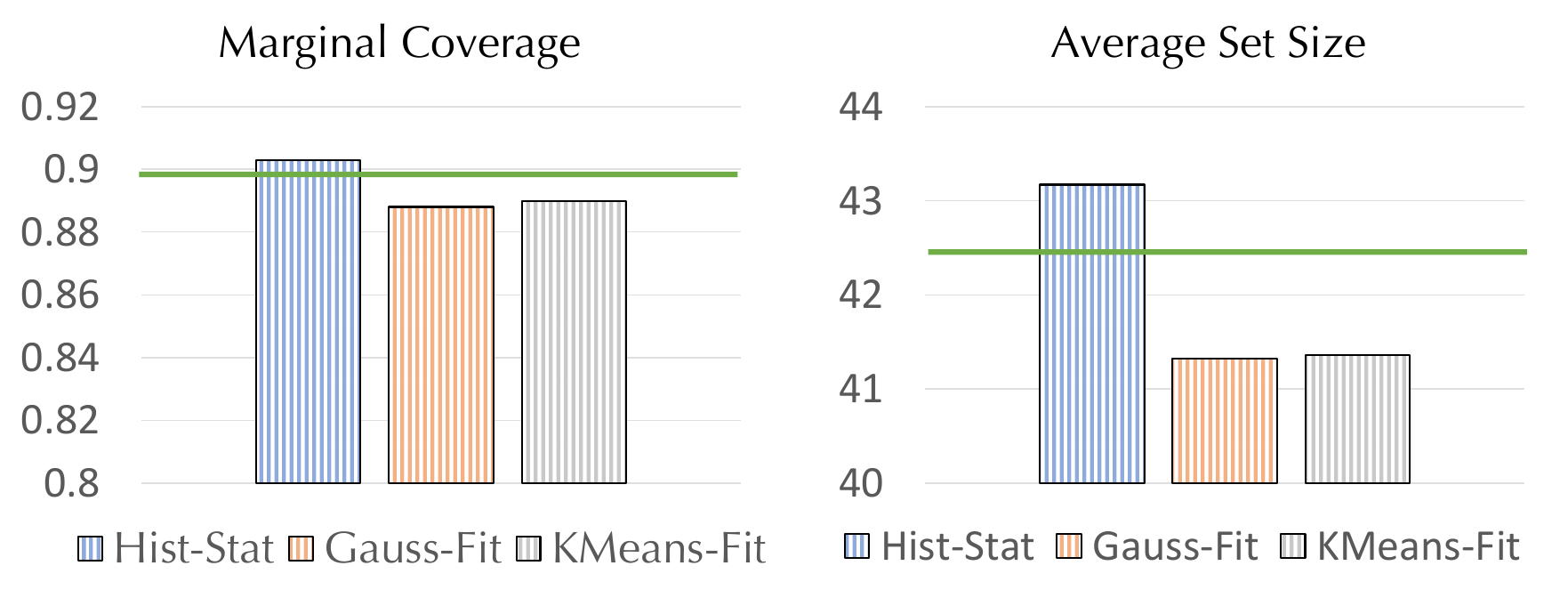}
    \caption{Marginal coverage / average set size under coverage attack with 40\% malicious clients with $\beta=0.5$ on Tiny-ImageNet. The green horizontal line represents the benign marginal coverage and average set size without any malicious clients.}
    \label{fig:abl_method_noniid}
\end{figure}

\begin{figure}[t]
    \centering
    \includegraphics[width=\linewidth]{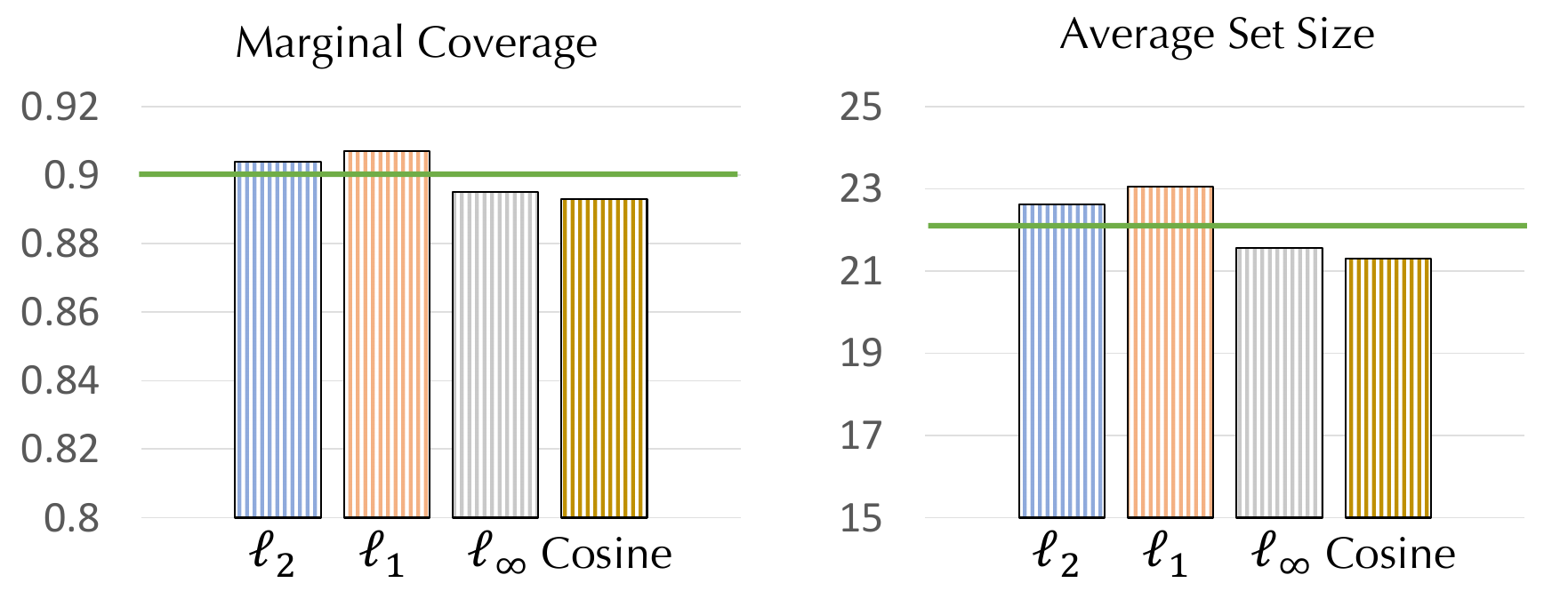}
    \caption{Marginal coverage / average set size under coverage attack with 40\% malicious clients with $\beta=0.0$ on Tiny-ImageNet. The green horizontal line represents the benign marginal coverage and average set size without any malicious clients.}
    \label{fig:abl_dis}
\end{figure}

\end{document}